\documentclass[11pt]{article}
\usepackage[preprint]{acl}
\usepackage{times}
\usepackage{latexsym}
\usepackage[T1]{fontenc}
\usepackage[utf8]{inputenc}
\usepackage{array}
\usepackage{multirow}
\usepackage{bm}
\usepackage{adjustbox}
\usepackage{amsmath,amssymb,amsthm}
\usepackage{booktabs}
\usepackage{graphicx}
\usepackage{xcolor}
\usepackage{algorithm}
\usepackage{algpseudocode}
\usepackage[most]{tcolorbox}
\usepackage{microtype}

\newtheorem{theorem}{Theorem}
\newtheorem{proposition}{Proposition}
\newtheorem{corollary}{Corollary}
\theoremstyle{definition}
\newtheorem{assumption}{Assumption}

\newcommand{\Lbit}{L_{\mathrm{bit}}}
\newcommand{\Cmin}{C_{\min}}
\newcommand{\Sw}{S_{\mathrm{w}}}
\newcommand{\supp}[1]{Appendix~#1}

\newtcolorbox{answerbox}{colback=gray!6,colframe=gray!55,boxrule=0.5pt,arc=1pt,left=4pt,right=4pt,top=3pt,bottom=3pt,fontupper=\small}

\newcommand{\XSfiftyCIgeminiTwoFive}{[0.79, 4.04]}
\newcommand{\XSfiftyCIgemmaFour}{[0.27, 0.37]}

\newcommand{\XSfiftyCIqwenThreevl}{[0.39, 0.73]}

\newcommand{\XSfiftyCIsonnetFourSix}{[0.18, 0.47]}
\newcommand{\XSfiftygemThreeEight}{0.24}
\newcommand{\XSfiftygeminiTwoFive}{1.28}
\newcommand{\XSfiftygemmaFour}{0.31}
\newcommand{\XSfiftygptFiveFour}{0.05}
\newcommand{\XSfiftygptFiveSix}{0.21}
\newcommand{\XSfiftyqwenThreevl}{0.54}
\newcommand{\XSfiftysonnetFiveFive}{0.28}
\newcommand{\XSfiftysonnetFourSix}{0.35}

\newcommand{\XagCovSheetsgptFiveSix}{0.75}

\newcommand{\XagExactCountgemThreeEight}{68}

\newcommand{\XagFgemThreeEightAll}{0.64}

\newcommand{\XagMinusDCIgemThreeEight}{[$-$0.40, $-$0.20]}

\newcommand{\XagMinusDgemThreeEight}{$-$0.30}

\newcommand{\XagTokIngemThreeEight}{376.3k}
\newcommand{\XagTokIngptFiveSix}{29.3k}
\newcommand{\XagTokInsonnetFiveFive}{11.3k}

\newcommand{\XagTokOutgptFiveSix}{2.9k}
\newcommand{\XagTokOutsonnetFiveFive}{1.4k}
\newcommand{\XagZoomsgemThreeEight}{17.8}
\newcommand{\XagZoomsgptFiveSix}{11.4}
\newcommand{\XagZoomssonnetFiveFive}{2.6}

\newcommand{\XareaFirstgptFiveSix}{1.00}

\newcommand{\XareaLastgptFiveSix}{0.91}

\newcommand{\XasCons}{16}
\newcommand{\XasMGptFiveSix}{+0.39}

\newcommand{\XasMin}{$-$0.03}
\newcommand{\XasModels}{8}
\newcommand{\XasN}{40}
\newcommand{\XasNewN}{10}
\newcommand{\XasNewSup}{2}

\newcommand{\XasSup}{24}

\newcommand{\XasThmMax}{+0.37}
\newcommand{\XasThmMin}{+0.06}
\newcommand{\XasThmSup}{4}

\newcommand{\XasWholeRecGptFiveSix}{0.49}
\newcommand{\XasWholeRecGptSixOne}{0.91}
\newcommand{\XasWholeRecOpusFiveFive}{0.76}
\newcommand{\XcapBestBits}{422}
\newcommand{\XcapBestFull}{438}

\newcommand{\XcapBestRecall}{0.98}
\newcommand{\XcapNlastgemThreeEight}{0.62}

\newcommand{\XcapNlastgemmaFour}{0.34}
\newcommand{\XcapNlastgptFiveFour}{0.85}
\newcommand{\XcapNlastgptFiveSix}{0.99}
\newcommand{\XcapNlastqwenThreevl}{0.22}
\newcommand{\XcapNlastsonnetFiveFive}{1.00}
\newcommand{\XcapNlastsonnetFourSix}{0.57}

\newcommand{\XcapNtwoqwenThreevl}{1.00}

\newcommand{\XcapSlopegptFiveSix}{$-$0.010}

\newcommand{\XcapSlopesonnetFiveFive}{$-$0.003}

\newcommand{\XclutSlopeCIgptFiveSix}{[$-$0.062, +0.035]}

\newcommand{\XclutSlopegptFiveSix}{$-$0.015}

\newcommand{\XctrFsccountgd}{0.90}

\newcommand{\XctrFscgecoTwo}{0.86}

\newcommand{\XfBKDAGgptFiveSix}{$-$0.01}

\newcommand{\XfBKDAOnegemThreeEight}{$-$0.34}

\newcommand{\XfBKDAOnegptFiveSix}{$-$0.22}

\newcommand{\XfBKDAOnesonnetFiveFive}{$-$0.16}
\newcommand{\XfBKDCIAGgptFiveSix}{[$-$0.09, +0.08]}

\newcommand{\XfBKDCIAOnegemThreeEight}{[$-$0.45, $-$0.24]}

\newcommand{\XfBKDCIAOnegptFiveSix}{[$-$0.34, $-$0.09]}

\newcommand{\XfBKDCIAOnesonnetFiveFive}{[$-$0.25, $-$0.07]}

\newcommand{\XfBKDCIDgptFiveSix}{[$-$0.10, +0.08]}

\newcommand{\XfBKDDgptFiveSix}{$-$0.01}

\newcommand{\XfBKSmedGptFiveSix}{1.5}

\newcommand{\XfBKSoneCIGptFiveSix}{[$-$0.40, $-$0.06]}
\newcommand{\XfBKSoneCISonnetFiveFive}{[$-$0.22, $-$0.04]}

\newcommand{\XfBKSoneGptFiveSix}{$-$0.23}

\newcommand{\XfBKSoneNGptFiveSix}{39}
\newcommand{\XfBKSoneNSonnetFiveFive}{50}
\newcommand{\XfBKSoneSonnetFiveFive}{$-$0.13}

\newcommand{\XfBenchAGgptFiveSix}{0.74}

\newcommand{\XfBenchAGsonnetFiveFive}{0.75}
\newcommand{\XfBenchAclaude}{0.36}
\newcommand{\XfBenchAgemThreeEight}{0.66}

\newcommand{\XfBenchAgemini}{0.14}
\newcommand{\XfBenchAgemmaFour}{0.22}
\newcommand{\XfBenchAgpt}{0.46}
\newcommand{\XfBenchAgptFiveSix}{0.60}
\newcommand{\XfBenchAqwenThreevl}{0.29}
\newcommand{\XfBenchAsonnetFiveFive}{0.71}
\newcommand{\XfBenchAtm}{0.58}
\newcommand{\XfBenchDclaude}{0.25}
\newcommand{\XfBenchDgemThreeEight}{0.74}
\newcommand{\XfBenchDgemini}{0.24}
\newcommand{\XfBenchDgemmaFour}{0.43}
\newcommand{\XfBenchDgpt}{0.36}
\newcommand{\XfBenchDgptFiveSix}{0.53}
\newcommand{\XfBenchDqwenThreevl}{0.26}
\newcommand{\XfBenchDsonnetFiveFive}{0.64}
\newcommand{\XfDmACIclaude}{[$-$0.15, $-$0.08]}
\newcommand{\XfDmACIgemThreeEight}{[+0.05, +0.11]}
\newcommand{\XfDmACIgemini}{[+0.07, +0.13]}
\newcommand{\XfDmACIgemmaFour}{[+0.17, +0.24]}
\newcommand{\XfDmACIgpt}{[$-$0.14, $-$0.06]}
\newcommand{\XfDmACIgptFiveSix}{[$-$0.10, $-$0.03]}

\newcommand{\XfDmACIsonnetFiveFive}{[$-$0.11, $-$0.04]}
\newcommand{\XfDmAclaude}{$-$0.11}
\newcommand{\XfDmAgemThreeEight}{+0.08}
\newcommand{\XfDmAgemini}{+0.10}
\newcommand{\XfDmAgemmaFour}{+0.21}
\newcommand{\XfDmAgpt}{$-$0.10}
\newcommand{\XfDmAgptFiveSix}{$-$0.07}

\newcommand{\XfDmAsonnetFiveFive}{$-$0.08}

\newcommand{\XfHEightgemThreeEight}{+0.42}

\newcommand{\XfHEightgptFiveSix}{+0.21}

\newcommand{\XfHEightsonnetFiveFive}{+0.24}

\newcommand{\XfSrFracDowngptFiveSix}{46}

\newcommand{\XfTOnebenchlargeDL}{+0.07}
\newcommand{\XfTOnebenchlargeDLCI}{[$-$0.01, +0.15]}

\newcommand{\XfTOnebenchsmallDL}{+0.14}
\newcommand{\XfTOnebenchsmallDLCI}{[+0.10, +0.19]}

\newcommand{\XfTOnesheetslargeDL}{+0.05}
\newcommand{\XfTOnesheetslargeDLCI}{[+0.002, +0.09]}

\newcommand{\XfTOnesheetslargeN}{43}
\newcommand{\XfTOnesheetslargeUD}{+0.05}

\newcommand{\XfTOnesheetssmallDL}{+0.17}
\newcommand{\XfTOnesheetssmallDLCI}{[+0.12, +0.21]}

\newcommand{\XfTOnesheetssmallN}{137}

\newcommand{\XfoCropFalse}{3}
\newcommand{\XfoFOneMax}{+0.32}
\newcommand{\XfoFOneMin}{+0.04}
\newcommand{\XfoFOneN}{8}
\newcommand{\XfoFOneSup}{4}
\newcommand{\XfoFewer}{7}
\newcommand{\XfoGemmaRate}{5.7}
\newcommand{\XfoGemmaRecall}{3}
\newcommand{\XfoGemmaRho}{49}
\newcommand{\XfoModels}{8}

\newcommand{\XfscODCIgptFiveSixsmallTU}{[+0.06, +0.22]}

\newcommand{\XfscODCIsonnetFiveFivesmallTU}{[+0.12, +0.27]}

\newcommand{\XfscODgptFiveSixsmallTU}{+0.14}

\newcommand{\XfscODsonnetFiveFivesmallTU}{+0.20}

\newcommand{\XfscParsesonnetFiveFive}{16}

\newcommand{\XlbFcountgd}{0.10}
\newcommand{\XlbFgecoTwo}{0.23}
\newcommand{\XlbFowlvTwo}{0.07}

\newcommand{\XlbOraclegecoTwo}{0.44}

\newcommand{\XlbRdcountgd}{0.11}
\newcommand{\XlbRdgecoTwo}{0.15}

\newcommand{\XlbTwofpccountgdimgOne}{0.19}

\newcommand{\XlbTwofpcgecoTwoimgOne}{0.34}

\newcommand{\XmosFgptFiveSixkEight}{0.53}
\newcommand{\XmosFgptFiveSixkFour}{0.74}
\newcommand{\XmosFgptFiveSixkOne}{0.76}
\newcommand{\XmosFgptFiveSixkTwo}{0.74}
\newcommand{\XmosFsonnetFiveFivekEight}{0.90}
\newcommand{\XmosFsonnetFiveFivekFour}{0.84}
\newcommand{\XmosFsonnetFiveFivekOne}{0.91}
\newcommand{\XmosFsonnetFiveFivekTwo}{0.93}
\newcommand{\XmosPgptFiveSixkEight}{0.72}
\newcommand{\XmosPgptFiveSixkFour}{0.86}
\newcommand{\XmosPgptFiveSixkOne}{0.80}
\newcommand{\XmosPgptFiveSixkTwo}{0.79}
\newcommand{\XmosPsonnetFiveFivekEight}{0.96}
\newcommand{\XmosPsonnetFiveFivekFour}{0.91}
\newcommand{\XmosPsonnetFiveFivekOne}{0.87}
\newcommand{\XmosPsonnetFiveFivekTwo}{0.92}
\newcommand{\XmosRegFsonnetFiveFivekEight}{0.22}

\newcommand{\XmosRgptFiveSixkEight}{0.47}
\newcommand{\XmosRgptFiveSixkFour}{0.68}
\newcommand{\XmosRgptFiveSixkOne}{0.77}
\newcommand{\XmosRgptFiveSixkTwo}{0.74}
\newcommand{\XmosRsonnetFiveFivekEight}{0.88}
\newcommand{\XmosRsonnetFiveFivekFour}{0.85}
\newcommand{\XmosRsonnetFiveFivekOne}{0.98}
\newcommand{\XmosRsonnetFiveFivekTwo}{0.94}

\newcommand{\XmosdRCIgptFiveSixkEight}{[$-$0.44, $-$0.16]}

\newcommand{\XmosdRCIsonnetFiveFivekEight}{[$-$0.19, $-$0.03]}
\newcommand{\XmosdRCIsonnetFiveFivekFour}{[$-$0.27, $-$0.04]}
\newcommand{\XmosdRCIsonnetFiveFivekTwo}{[$-$0.08, $-$0.003]}
\newcommand{\XmosdRgptFiveSixkEight}{$-$0.31}
\newcommand{\XmosdRgptFiveSixkFour}{$-$0.09}
\newcommand{\XmosdRgptFiveSixkTwo}{$-$0.03}
\newcommand{\XmosdRsonnetFiveFivekEight}{$-$0.10}
\newcommand{\XmosdRsonnetFiveFivekFour}{$-$0.13}
\newcommand{\XmosdRsonnetFiveFivekTwo}{$-$0.04}

\newcommand{\XnAGDRallgptFiveSix}{+0.09}
\newcommand{\XnAGDRallgptFiveSixCI}{[+0.04, +0.16]}

\newcommand{\XnAGDRallsonnetFiveFive}{+0.05}
\newcommand{\XnAGDRallsonnetFiveFiveCI}{[+0.01, +0.11]}

\newcommand{\XnDRDgemThreeEight}{+0.03}
\newcommand{\XnDRDgemThreeEightCI}{[$-$0.01, +0.07]}

\newcommand{\XnDRDgptFiveSix}{+0.15}
\newcommand{\XnDRDgptFiveSixCI}{[+0.07, +0.26]}

\newcommand{\XnDRDsonnetFiveFive}{$-$0.01}
\newcommand{\XnDRDsonnetFiveFiveCI}{[$-$0.02, +0.003]}

\newcommand{\XnFAGgemThreeEightSheets}{0.53}

\newcommand{\XnFAGgptFiveSixSheets}{0.97}

\newcommand{\XnFAGsonnetFiveFiveSheets}{0.95}

\newcommand{\XnFAOneRgemThreeEightSheets}{0.30}

\newcommand{\XnFAOneRgptFiveSixSheets}{0.89}

\newcommand{\XnFAOneRsonnetFiveFiveSheets}{0.73}

\newcommand{\XnFDRgemThreeEightSheets}{0.91}

\newcommand{\XnFDRgptFiveSixSheets}{0.91}

\newcommand{\XnFDRsonnetFiveFiveSheets}{0.90}

\newcommand{\XnFNATRgptFiveSixSheets}{0.33}

\newcommand{\XnHOnegptFiveSix}{+0.06}

\newcommand{\XnHThreegptFiveSixAGOVZ}{+0.01}
\newcommand{\XnHThreegptFiveSixAGOVZCI}{[$-$0.003, +0.03]}

\newcommand{\XnHThreegptFiveSixOVCAOneR}{$-$0.04}
\newcommand{\XnHThreegptFiveSixOVCAOneRCI}{[$-$0.15, +0.05]}

\newcommand{\XnHThreegptFiveSixOVZOVC}{+0.11}
\newcommand{\XnHThreegptFiveSixOVZOVCCI}{[+0.02, +0.26]}

\newcommand{\XnHThreesonnetFiveFiveAGOVZ}{+0.09}
\newcommand{\XnHThreesonnetFiveFiveAGOVZCI}{[+0.02, +0.19]}

\newcommand{\XnHThreesonnetFiveFiveOVCAOneR}{+0.07}
\newcommand{\XnHThreesonnetFiveFiveOVCAOneRCI}{[$-$0.11, +0.25]}

\newcommand{\XnHThreesonnetFiveFiveOVZOVC}{+0.06}
\newcommand{\XnHThreesonnetFiveFiveOVZOVCCI}{[$-$0.04, +0.19]}

\newcommand{\XnXNineHOnegptFiveSixHolm}{0.013}

\newcommand{\XovSmedSheets}{0.93}
\newcommand{\XplainFgemThreeEighttAll}{0.65}

\newcommand{\XplainFgptFiveSixrAll}{0.79}

\newcommand{\XqAgF}{0.21}

\newcommand{\XqAgMinusD}{$-$0.41}
\newcommand{\XqAgMinusDCI}{[$-$0.52, $-$0.30]}
\newcommand{\XqCountErrAg}{0.21}
\newcommand{\XqCountErrD}{0.25}
\newcommand{\XqD}{0.62}

\newcommand{\XrdDClDgemThreeEightAOnegemThreeEight}{[$-$0.004, +0.33]}

\newcommand{\XrdDClDgptFiveSixAOnegptFiveSix}{[$-$0.06, +0.09]}

\newcommand{\XrdDClDsonnetFiveFiveAOnesonnetFiveFive}{[$-$0.08, +0.02]}

\newcommand{\XrdDDgemThreeEightAOnegemThreeEight}{+0.23}

\newcommand{\XrdDDgptFiveSixAOnegptFiveSix}{0.00}

\newcommand{\XrdDDsonnetFiveFiveAOnesonnetFiveFive}{$-$0.05}

\newcommand{\XrdFAGgptFiveSix}{0.94}
\newcommand{\XrdFAGsonnetFiveFive}{0.89}
\newcommand{\XrdFAOnegemThreeEight}{0.68}
\newcommand{\XrdFAOnegptFiveSix}{0.80}
\newcommand{\XrdFAOnesonnetFiveFive}{0.84}
\newcommand{\XrdFDgemThreeEight}{0.91}
\newcommand{\XrdFDgptFiveSix}{0.80}
\newcommand{\XrdFDsonnetFiveFive}{0.79}
\newcommand{\XrdFHgemThreeEight}{0.74}
\newcommand{\XrdFHgptFiveSix}{0.79}
\newcommand{\XrdFHsonnetFiveFive}{0.80}
\newcommand{\XrdFTM}{0.38}
\newcommand{\XrdFTMoracle}{0.54}
\newcommand{\XrdSgemThreeEight}{0.92}
\newcommand{\XrdSgptFiveSix}{1.25}
\newcommand{\XrdSsonnetFiveFive}{1.43}
\newcommand{\XrobMax}{0.15}
\newcommand{\XrobMaxModel}{Qwen3-VL}
\newcommand{\XrobZero}{3}

\newcommand{\XrsnFgemThreeEighttAll}{0.45}

\newcommand{\XrsnFgptFiveSixrAll}{0.92}

\newcommand{\XrsnMinusDCIgemThreeEighttSheets}{[$-$0.76, $-$0.37]}

\newcommand{\XrsnMinusDCIgptFiveSixrSheets}{[$-$0.11, +0.27]}

\newcommand{\XrsnMinusDCIsonnetFiveFivetSheets}{[$-$0.35, $-$0.02]}

\newcommand{\XrsnMinusDgemThreeEighttSheets}{$-$0.58}

\newcommand{\XrsnMinusDgptFiveSixrSheets}{+0.10}

\newcommand{\XrsnMinusDsonnetFiveFivetSheets}{$-$0.18}

\newcommand{\XtOneDAIC}{5.4}
\newcommand{\XtOneHThree}{+0.58}
\newcommand{\XtOneHThreeCI}{[+0.46, +0.69]}

\newcommand{\XtTwoRCgptFiveFour}{0.62}

\newcommand{\XtTwoRCsonnetFourSix}{0.39}

\newcommand{\XtTwoRSgptFiveFour}{0.62}

\newcommand{\XtTwoRSsonnetFourSix}{0.52}

\newcommand{\XtcCIIngptFiveSix}{[+0.02, +0.13]}
\newcommand{\XtcCIInsonnetFiveFive}{[+0.06, +0.20]}
\newcommand{\XtcCIOutgptFiveSix}{[$-$0.20, $-$0.08]}
\newcommand{\XtcCIOutsonnetFiveFive}{[$-$0.16, $-$0.07]}
\newcommand{\XtcCons}{25}

\newcommand{\XtcIngptFiveSix}{+0.08}
\newcommand{\XtcInsonnetFiveFive}{+0.13}
\newcommand{\XtcInteriorPct}{84}
\newcommand{\XtcN}{89}
\newcommand{\XtcOutN}{21}
\newcommand{\XtcOutNeg}{12}
\newcommand{\XtcOutNegPoint}{19}
\newcommand{\XtcOutgptFiveSix}{$-$0.14}
\newcommand{\XtcOutsonnetFiveFive}{$-$0.11}
\newcommand{\XtcSup}{60}

\newcommand{\XtcThmMin}{$-$0.01}
\newcommand{\XtcThmN}{40}
\newcommand{\XtcThmSup}{33}
\newcommand{\XtcVendCIGemThreeEight}{[$-$0.09, +0.12]}
\newcommand{\XtcVendCIGptFiveFour}{[$-$0.28, $-$0.07]}
\newcommand{\XtcVendCIGptFiveSix}{[$-$0.34, $-$0.02]}
\newcommand{\XtcVendCIGptFiveSixR}{[$-$0.57, $-$0.23]}
\newcommand{\XtcVendGemThreeEight}{+0.01}
\newcommand{\XtcVendGptFiveFour}{$-$0.16}
\newcommand{\XtcVendGptFiveSix}{$-$0.18}
\newcommand{\XtcVendGptFiveSixR}{$-$0.39}

\newcommand{\XtcVendNGptFiveFour}{20}

\newcommand{\XtcViol}{4}
\newcommand{\XtcViolCI}{[$-$0.76, $-$0.44]}
\newcommand{\XtcViolD}{$-$0.60}
\newcommand{\XtcViolDpos}{0.60 (95\% interval 0.44 to 0.76)}

\newcommand{\XtileTokIngemThreeEight}{96.9k}
\newcommand{\XtileTokIngptFiveSix}{16.8k}
\newcommand{\XtileTokInsonnetFiveFive}{20.6k}

\newcommand{\XtileTokOutgptFiveSix}{0.5k}
\newcommand{\XtileTokOutsonnetFiveFive}{2.3k}
\newcommand{\XupDCIgemThreeEight}{[$-$0.01, +0.08]}

\newcommand{\XupDCIgemmaFour}{[$-$0.10, +0.03]}

\newcommand{\XupDgemThreeEight}{+0.03}

\newcommand{\XupDgemmaFour}{$-$0.03}

\newcommand{\XupEightgptFiveSix}{0.65 to 0.83}

\newcommand{\XupEightsonnetFiveFive}{0.65 to 0.62}

\newcommand{\FDsonnetFiveFiveAll}{0.92}

\newcommand{\FDgemThreeEightAll}{0.94}

\newcommand{\FATwogeminiTwoFiveAll}{0.07}

\newcommand{\FDgptFiveFourAll}{0.67}

\newcommand{\FHgptFiveFourAll}{0.83}

\newcommand{\FDsonnetFourSixAll}{0.61}

\newcommand{\FHsonnetFourSixAll}{0.86}

\newcommand{\FDgeminiTwoFiveAll}{0.41}

\newcommand{\FHgeminiTwoFiveAll}{0.85}

\newcommand{\FDgemmaFourAll}{0.75}

\newcommand{\FHgemmaFourAll}{0.86}

\newcommand{\FDqwenThreevlAll}{0.62}

\newcommand{\FHqwenThreevlAll}{0.91}

\newcommand{\FATwogptFiveFourAll}{0.62}

\newcommand{\FATwosonnetFourSixAll}{0.36}

\newcommand{\FCgeminiTwoFiveAll}{0.18}

\newcommand{\FATwogemmaFourAll}{0.31}

\newcommand{\FCgemmaFourAll}{0.68}

\newcommand{\FATwoqwenThreevlAll}{0.36}

\newcommand{\FBgptFiveFourAll}{0.51}

\newcommand{\FBgeminiTwoFiveAll}{0.18}

\newcommand{\FTMdevAll}{0.87}

\newcommand{\FTMlooAll}{0.89}

\newcommand{\CDMinusATwogptFiveFour}{+0.05}
\newcommand{\CDMinusATwogptFiveFourCI}{[$-$0.07, +0.17]}

\newcommand{\CHMinusTMdevgptFiveFour}{$-$0.04}

\newcommand{\CDMinusATwosonnetFourSix}{+0.25}
\newcommand{\CDMinusATwosonnetFourSixCI}{[+0.11, +0.38]}

\newcommand{\CDMinusATwogeminiTwoFive}{+0.34}
\newcommand{\CDMinusATwogeminiTwoFiveCI}{[+0.27, +0.42]}

\newcommand{\CDMinusATwogemmaFour}{+0.44}
\newcommand{\CDMinusATwogemmaFourCI}{[+0.34, +0.54]}

\newcommand{\CDMinusATwoqwenThreevl}{+0.26}
\newcommand{\CDMinusATwoqwenThreevlCI}{[+0.11, +0.39]}

\newcommand{\CHMinusTMdevqwenThreevl}{+0.04}

\newcommand{\CDMinusATwosheetsgptFiveSix}{+0.13}
\newcommand{\CDMinusATwosheetsgptFiveSixCI}{[$-$0.09, +0.37]}

\newcommand{\CDMinusATwosheetssonnetFiveFive}{+0.24}
\newcommand{\CDMinusATwosheetssonnetFiveFiveCI}{[+0.06, +0.43]}

\newcommand{\CDMinusATwosheetsgemThreeEight}{+0.36}
\newcommand{\CDMinusATwosheetsgemThreeEightCI}{[+0.18, +0.55]}

\newcommand{\CgenclaudeATwosheets}{+0.53}
\newcommand{\CgenclaudeATwosheetsCI}{[+0.37, +0.68]}
\newcommand{\CgengptFiveSixAOne}{+0.17}
\newcommand{\CgengptFiveSixAOneCI}{[+0.08, +0.27]}
\newcommand{\CgengptFiveFiveAOne}{$-$0.07}
\newcommand{\CgengptFiveFiveAOneCI}{[$-$0.16, +0.03]}

\newcommand{\CgengeminiATwosheets}{+0.52}
\newcommand{\CgengeminiATwosheetsCI}{[+0.33, +0.70]}

\newcommand{\CcapnativeMinusautosheets}{$-$0.28}
\newcommand{\CcapnativeMinusautosheetsCI}{[$-$0.49, $-$0.05]}

\newcommand{\CBMinusDgptFiveFour}{$-$0.17}
\newcommand{\CBMinusDgptFiveFourCI}{[$-$0.30, $-$0.03]}
\newcommand{\CBMinusDgeminiTwoFive}{$-$0.24}
\newcommand{\CBMinusDgeminiTwoFiveCI}{[$-$0.32, $-$0.14]}
\newcommand{\CDMinusCgeminiTwoFive}{+0.23}
\newcommand{\CDMinusCgeminiTwoFiveCI}{[+0.17, +0.30]}
\newcommand{\CDMinusCgemmaFour}{+0.08}
\newcommand{\CDMinusCgemmaFourCI}{[+0.02, +0.14]}

\newcommand{\RhogptFiveFour}{0.77}
\newcommand{\RhogptFiveFourCI}{[0.58, 0.88]}
\newcommand{\RhosonnetFourSix}{0.70}
\newcommand{\RhosonnetFourSixCI}{[0.43, 0.88]}
\newcommand{\RhogeminiTwoFive}{0.58}
\newcommand{\RhogeminiTwoFiveCI}{[0.32, 0.78]}
\newcommand{\RhogemmaFour}{0.76}
\newcommand{\RhogemmaFourCI}{[0.53, 0.90]}
\newcommand{\RhoqwenThreevl}{0.56}
\newcommand{\RhoqwenThreevlCI}{[0.29, 0.77]}
\newcommand{\RhosonnetFiveFive}{0.62}

\newcommand{\RhoopusFiveFive}{0.53}

\newcommand{\RhogptFiveFive}{0.56}

\newcommand{\RhogptFiveSix}{0.50}

\newcommand{\RhogemThreeEight}{0.33}
\newcommand{\RhogemThreeEightCI}{[$-$0.05, 0.66]}

\newcommand{\GainLowgptFiveFour}{+0.41}
\newcommand{\GainHighgptFiveFour}{$-$0.10}
\newcommand{\GainHighgptFiveFourCI}{[$-$0.24, +0.02]}

\newcommand{\HybBound}{0.98}

\newcommand{\CbudEightDef}{0.54}
\newcommand{\CbudEightLo}{0.44}
\newcommand{\CbudEightUh}{0.71}
\newcommand{\CcapCalls}{336}
\newcommand{\CcapMaxOut}{3{,}428}
\newcommand{\CcapNormal}{335}
\newcommand{\CfigCmin}{37{,}000}
\newcommand{\CfigH}{7{,}698}
\newcommand{\CfigM}{74}
\newcommand{\CfigMsent}{13}
\newcommand{\CfigPctToken}{40}
\newcommand{\CfigPropB}{35{,}000}
\newcommand{\CfigSwhole}{0.40}
\newcommand{\CfigTargets}{23}
\newcommand{\CfigTokens}{49{,}000}
\newcommand{\CfigViews}{20}
\newcommand{\CfigW}{11{,}236}
\newcommand{\CfscNoTruncGptFiveSix}{0.12}
\newcommand{\CfscNoTruncNGptFiveSix}{49}
\newcommand{\CfscNoTruncNSonnetFiveFive}{26}
\newcommand{\CfscNoTruncSonnetFiveFive}{0.17}
\newcommand{\CgemFourAns}{94}
\newcommand{\CgemFourFpMix}{0.35}
\newcommand{\CgemFourFpNorm}{0.35}
\newcommand{\CgemFourFpPx}{0.02}
\newcommand{\CgemFourLsMix}{0.53}
\newcommand{\CgemFourLsNorm}{0.38}
\newcommand{\CgemFourLsPx}{0.27}
\newcommand{\CgemFourOver}{14}
\newcommand{\CgenGptSixOneMax}{324}
\newcommand{\CgenSninetyGemFour}{0.29}
\newcommand{\CgenSninetyGemThreeEight}{0.76}
\newcommand{\CmapBandMax}{+0.14}
\newcommand{\CmapBandMin}{+0.08}
\newcommand{\CmapBelowMax}{0.46}
\newcommand{\CmapBelowMin}{0.10}

\newcommand{\CmapLowSMax}{+0.46}
\newcommand{\CmapLowSMin}{+0.26}
\newcommand{\CmapMidSMax}{+0.17}
\newcommand{\CmapMidSMin}{+0.07}

\newcommand{\CmapZoomCells}{six}
\newcommand{\CmapZoomMax}{0.22}
\newcommand{\CmapZoomMin}{0.01}
\newcommand{\CmapZoomSig}{four}
\newcommand{\CpackMax}{35}
\newcommand{\CpackMean}{8.6}
\newcommand{\CpackMedian}{4}
\newcommand{\CscaleCalls}{187{,}000}
\newcommand{\CscaleCtrl}{2{,}769}
\newcommand{\CscaleImages}{3{,}398}
\newcommand{\CscaleInst}{199{,}000}
\newcommand{\CscaleModels}{12}
\newcommand{\CsmallH}{1{,}301}
\newcommand{\CsmallM}{142}
\newcommand{\CsmallRecTile}{0.75}
\newcommand{\CsmallRecWhole}{1.00}
\newcommand{\CsmallSt}{1.40}
\newcommand{\CsmallSw}{2.86}
\newcommand{\CsmallW}{2{,}489}
\newcommand{\CtCells}{10}
\newcommand{\CtCminMax}{3.5}
\newcommand{\CtCminMin}{1.6}

\newcommand{\CtFpGemFourRecW}{0.50}

\newcommand{\CtFpGemFourdRecD}{$-$0.01}
\newcommand{\CtFpGemThreeEightRecR}{0.43}
\newcommand{\CtFpGemThreeEightRecW}{0.25}
\newcommand{\CtFpGemThreeEightRel}{73}

\newcommand{\CtFpGemThreeEightdRecD}{$-$0.13}

\newcommand{\CtFpQwenThreevlRecR}{0.12}
\newcommand{\CtFpQwenThreevlRecW}{0.05}
\newcommand{\CtFpQwenThreevlRel}{126}

\newcommand{\CtFpQwenThreevldRecD}{$-$0.11}
\newcommand{\CtLsGemFourRecR}{0.98}
\newcommand{\CtLsGemFourRecW}{0.85}
\newcommand{\CtLsGemFourRel}{15}

\newcommand{\CtLsGemThreeEightRecR}{0.90}
\newcommand{\CtLsGemThreeEightRecW}{0.62}
\newcommand{\CtLsGemThreeEightRel}{45}

\newcommand{\CtLsQwenThreevlRecR}{0.56}
\newcommand{\CtLsQwenThreevlRecW}{0.37}
\newcommand{\CtLsQwenThreevlRel}{51}

\newcommand{\CtOneViewGpt}{49}
\newcommand{\CtSigFOne}{six}
\newcommand{\CtSigPrec}{six}
\newcommand{\CtSigRec}{five}
\newcommand{\CtTokDMax}{1.06}
\newcommand{\CtTokDMin}{0.11}
\newcommand{\CtTokMax}{8.6}
\newcommand{\CtTokMin}{3.1}
\newcommand{\CthumbSideClaude}{46}
\newcommand{\CthumbSideGemini}{22}
\newcommand{\CthumbSideGpt}{60}
\newcommand{\CthumbSideQwen}{33}
\newcommand{\UcsMax}{0.28}
\newcommand{\UfTOnesheetslargeDL}{0.05}
\newcommand{\UfTOnesheetssmallDL}{0.17}
\newcommand{\UfscODgptFiveSixsmallTU}{0.14}
\newcommand{\UfscODsonnetFiveFivesmallTU}{0.20}
\newcommand{\UnHThreegptFiveSixOVZOVC}{0.11}
\newcommand{\UtOneHThree}{0.58}

\newcommand{\CstepRepMin}{1.04}
\newcommand{\CstepRepMax}{1.09}
\newcommand{\CstepAppMin}{2.04}
\newcommand{\CstepAppMax}{2.09}

\newcommand{\CstepGridApp}{5.00}

\newcommand{\CagLocGem}{64}
\newcommand{\CagLocGemTwo}{82}
\newcommand{\CagLocQwen}{21}
\newcommand{\CagLocQwenTwo}{46}
\newcommand{\CtileLocGem}{95}
\newcommand{\CtileLocQwen}{70}
\newcommand{\CagCntGem}{85}
\newcommand{\CtileCntGem}{76}
\newcommand{\CagCntQwen}{32}
\newcommand{\CtileCntQwen}{32}

\newcommand{\CagDistGem}{0.55}
\newcommand{\CagDistQwen}{2.08}

\providecommand{\lsv}[3]{\csname ls-#1-#2-#3\endcsname}
\expandafter\def\csname ls-gpt61-8-R\endcsname{1.00}
\expandafter\def\csname ls-gpt61-8-P\endcsname{0.98}
\expandafter\def\csname ls-gpt61-8-joint\endcsname{11}
\expandafter\def\csname ls-gpt61-8-empty\endcsname{0}
\expandafter\def\csname ls-gpt61-8-k\endcsname{12}
\expandafter\def\csname ls-gpt61-8-planned\endcsname{12}
\expandafter\def\csname ls-gpt61-8-under\endcsname{0}
\expandafter\def\csname ls-gpt61-8-over\endcsname{1}
\expandafter\def\csname ls-gpt61-16-R\endcsname{1.00}
\expandafter\def\csname ls-gpt61-16-P\endcsname{0.97}
\expandafter\def\csname ls-gpt61-16-joint\endcsname{11}
\expandafter\def\csname ls-gpt61-16-empty\endcsname{0}
\expandafter\def\csname ls-gpt61-16-k\endcsname{12}
\expandafter\def\csname ls-gpt61-16-planned\endcsname{12}
\expandafter\def\csname ls-gpt61-16-under\endcsname{0}
\expandafter\def\csname ls-gpt61-16-over\endcsname{1}
\expandafter\def\csname ls-gpt61-32-R\endcsname{1.00}
\expandafter\def\csname ls-gpt61-32-P\endcsname{0.97}
\expandafter\def\csname ls-gpt61-32-joint\endcsname{11}
\expandafter\def\csname ls-gpt61-32-empty\endcsname{0}
\expandafter\def\csname ls-gpt61-32-k\endcsname{12}
\expandafter\def\csname ls-gpt61-32-planned\endcsname{12}
\expandafter\def\csname ls-gpt61-32-under\endcsname{0}
\expandafter\def\csname ls-gpt61-32-over\endcsname{1}
\expandafter\def\csname ls-gpt61-64-R\endcsname{1.00}
\expandafter\def\csname ls-gpt61-64-P\endcsname{0.96}
\expandafter\def\csname ls-gpt61-64-joint\endcsname{11}
\expandafter\def\csname ls-gpt61-64-empty\endcsname{0}
\expandafter\def\csname ls-gpt61-64-k\endcsname{12}
\expandafter\def\csname ls-gpt61-64-planned\endcsname{12}
\expandafter\def\csname ls-gpt61-64-under\endcsname{0}
\expandafter\def\csname ls-gpt61-64-over\endcsname{5}
\expandafter\def\csname ls-gpt61-128-R\endcsname{1.00}
\expandafter\def\csname ls-gpt61-128-P\endcsname{0.94}
\expandafter\def\csname ls-gpt61-128-joint\endcsname{9}
\expandafter\def\csname ls-gpt61-128-empty\endcsname{0}
\expandafter\def\csname ls-gpt61-128-k\endcsname{12}
\expandafter\def\csname ls-gpt61-128-planned\endcsname{12}
\expandafter\def\csname ls-gpt61-128-under\endcsname{0}
\expandafter\def\csname ls-gpt61-128-over\endcsname{4}
\expandafter\def\csname ls-gpt61-256-R\endcsname{1.00}
\expandafter\def\csname ls-gpt61-256-P\endcsname{0.97}
\expandafter\def\csname ls-gpt61-256-joint\endcsname{11}
\expandafter\def\csname ls-gpt61-256-empty\endcsname{0}
\expandafter\def\csname ls-gpt61-256-k\endcsname{12}
\expandafter\def\csname ls-gpt61-256-planned\endcsname{12}
\expandafter\def\csname ls-gpt61-256-under\endcsname{1}
\expandafter\def\csname ls-gpt61-256-over\endcsname{2}
\expandafter\def\csname ls-gpt61-324-R\endcsname{1.00}
\expandafter\def\csname ls-gpt61-324-P\endcsname{0.97}
\expandafter\def\csname ls-gpt61-324-joint\endcsname{11}
\expandafter\def\csname ls-gpt61-324-empty\endcsname{0}
\expandafter\def\csname ls-gpt61-324-k\endcsname{12}
\expandafter\def\csname ls-gpt61-324-planned\endcsname{12}
\expandafter\def\csname ls-gpt61-324-under\endcsname{0}
\expandafter\def\csname ls-gpt61-324-over\endcsname{1}
\expandafter\def\csname ls-opus55-8-R\endcsname{0.97}
\expandafter\def\csname ls-opus55-8-P\endcsname{0.90}
\expandafter\def\csname ls-opus55-8-joint\endcsname{7}
\expandafter\def\csname ls-opus55-8-empty\endcsname{0}
\expandafter\def\csname ls-opus55-8-k\endcsname{12}
\expandafter\def\csname ls-opus55-8-planned\endcsname{12}
\expandafter\def\csname ls-opus55-8-under\endcsname{0}
\expandafter\def\csname ls-opus55-8-over\endcsname{4}
\expandafter\def\csname ls-opus55-16-R\endcsname{1.00}
\expandafter\def\csname ls-opus55-16-P\endcsname{0.90}
\expandafter\def\csname ls-opus55-16-joint\endcsname{6}
\expandafter\def\csname ls-opus55-16-empty\endcsname{0}
\expandafter\def\csname ls-opus55-16-k\endcsname{12}
\expandafter\def\csname ls-opus55-16-planned\endcsname{12}
\expandafter\def\csname ls-opus55-16-under\endcsname{0}
\expandafter\def\csname ls-opus55-16-over\endcsname{11}
\expandafter\def\csname ls-opus55-32-R\endcsname{1.00}
\expandafter\def\csname ls-opus55-32-P\endcsname{0.95}
\expandafter\def\csname ls-opus55-32-joint\endcsname{10}
\expandafter\def\csname ls-opus55-32-empty\endcsname{0}
\expandafter\def\csname ls-opus55-32-k\endcsname{12}
\expandafter\def\csname ls-opus55-32-planned\endcsname{12}
\expandafter\def\csname ls-opus55-32-under\endcsname{0}
\expandafter\def\csname ls-opus55-32-over\endcsname{5}
\expandafter\def\csname ls-opus55-64-R\endcsname{1.00}
\expandafter\def\csname ls-opus55-64-P\endcsname{0.96}
\expandafter\def\csname ls-opus55-64-joint\endcsname{10}
\expandafter\def\csname ls-opus55-64-empty\endcsname{0}
\expandafter\def\csname ls-opus55-64-k\endcsname{12}
\expandafter\def\csname ls-opus55-64-planned\endcsname{12}
\expandafter\def\csname ls-opus55-64-under\endcsname{0}
\expandafter\def\csname ls-opus55-64-over\endcsname{3}
\expandafter\def\csname ls-opus55-128-R\endcsname{1.00}
\expandafter\def\csname ls-opus55-128-P\endcsname{0.96}
\expandafter\def\csname ls-opus55-128-joint\endcsname{11}
\expandafter\def\csname ls-opus55-128-empty\endcsname{0}
\expandafter\def\csname ls-opus55-128-k\endcsname{12}
\expandafter\def\csname ls-opus55-128-planned\endcsname{12}
\expandafter\def\csname ls-opus55-128-under\endcsname{0}
\expandafter\def\csname ls-opus55-128-over\endcsname{4}
\expandafter\def\csname ls-opus55-196-R\endcsname{0.83}
\expandafter\def\csname ls-opus55-196-P\endcsname{0.97}
\expandafter\def\csname ls-opus55-196-joint\endcsname{9}
\expandafter\def\csname ls-opus55-196-empty\endcsname{2}
\expandafter\def\csname ls-opus55-196-k\endcsname{12}
\expandafter\def\csname ls-opus55-196-planned\endcsname{12}
\expandafter\def\csname ls-opus55-196-under\endcsname{2}
\expandafter\def\csname ls-opus55-196-over\endcsname{2}
\expandafter\def\csname ls-gem4-8-R\endcsname{1.00}
\expandafter\def\csname ls-gem4-8-P\endcsname{0.98}
\expandafter\def\csname ls-gem4-8-joint\endcsname{11}
\expandafter\def\csname ls-gem4-8-empty\endcsname{0}
\expandafter\def\csname ls-gem4-8-k\endcsname{12}
\expandafter\def\csname ls-gem4-8-planned\endcsname{12}
\expandafter\def\csname ls-gem4-8-under\endcsname{0}
\expandafter\def\csname ls-gem4-8-over\endcsname{1}
\expandafter\def\csname ls-gem4-16-R\endcsname{1.00}
\expandafter\def\csname ls-gem4-16-P\endcsname{0.97}
\expandafter\def\csname ls-gem4-16-joint\endcsname{11}
\expandafter\def\csname ls-gem4-16-empty\endcsname{0}
\expandafter\def\csname ls-gem4-16-k\endcsname{12}
\expandafter\def\csname ls-gem4-16-planned\endcsname{12}
\expandafter\def\csname ls-gem4-16-under\endcsname{0}
\expandafter\def\csname ls-gem4-16-over\endcsname{1}
\expandafter\def\csname ls-gem4-32-R\endcsname{1.00}
\expandafter\def\csname ls-gem4-32-P\endcsname{0.97}
\expandafter\def\csname ls-gem4-32-joint\endcsname{11}
\expandafter\def\csname ls-gem4-32-empty\endcsname{0}
\expandafter\def\csname ls-gem4-32-k\endcsname{12}
\expandafter\def\csname ls-gem4-32-planned\endcsname{12}
\expandafter\def\csname ls-gem4-32-under\endcsname{0}
\expandafter\def\csname ls-gem4-32-over\endcsname{1}
\expandafter\def\csname ls-gem4-42-R\endcsname{0.99}
\expandafter\def\csname ls-gem4-42-P\endcsname{0.97}
\expandafter\def\csname ls-gem4-42-joint\endcsname{11}
\expandafter\def\csname ls-gem4-42-empty\endcsname{0}
\expandafter\def\csname ls-gem4-42-k\endcsname{12}
\expandafter\def\csname ls-gem4-42-planned\endcsname{12}
\expandafter\def\csname ls-gem4-42-under\endcsname{5}
\expandafter\def\csname ls-gem4-42-over\endcsname{1}
\expandafter\def\csname ls-gpt56-8-R\endcsname{0.96}
\expandafter\def\csname ls-gpt56-8-P\endcsname{0.95}
\expandafter\def\csname ls-gpt56-8-joint\endcsname{6}
\expandafter\def\csname ls-gpt56-8-empty\endcsname{0}
\expandafter\def\csname ls-gpt56-8-k\endcsname{12}
\expandafter\def\csname ls-gpt56-8-planned\endcsname{12}
\expandafter\def\csname ls-gpt56-8-under\endcsname{2}
\expandafter\def\csname ls-gpt56-8-over\endcsname{4}
\expandafter\def\csname ls-gpt56-16-R\endcsname{0.99}
\expandafter\def\csname ls-gpt56-16-P\endcsname{0.96}
\expandafter\def\csname ls-gpt56-16-joint\endcsname{11}
\expandafter\def\csname ls-gpt56-16-empty\endcsname{0}
\expandafter\def\csname ls-gpt56-16-k\endcsname{12}
\expandafter\def\csname ls-gpt56-16-planned\endcsname{12}
\expandafter\def\csname ls-gpt56-16-under\endcsname{2}
\expandafter\def\csname ls-gpt56-16-over\endcsname{4}
\expandafter\def\csname ls-gpt56-32-R\endcsname{1.00}
\expandafter\def\csname ls-gpt56-32-P\endcsname{0.95}
\expandafter\def\csname ls-gpt56-32-joint\endcsname{9}
\expandafter\def\csname ls-gpt56-32-empty\endcsname{0}
\expandafter\def\csname ls-gpt56-32-k\endcsname{12}
\expandafter\def\csname ls-gpt56-32-planned\endcsname{12}
\expandafter\def\csname ls-gpt56-32-under\endcsname{1}
\expandafter\def\csname ls-gpt56-32-over\endcsname{3}
\expandafter\def\csname ls-gpt56-64-R\endcsname{0.97}
\expandafter\def\csname ls-gpt56-64-P\endcsname{0.97}
\expandafter\def\csname ls-gpt56-64-joint\endcsname{10}
\expandafter\def\csname ls-gpt56-64-empty\endcsname{0}
\expandafter\def\csname ls-gpt56-64-k\endcsname{12}
\expandafter\def\csname ls-gpt56-64-planned\endcsname{12}
\expandafter\def\csname ls-gpt56-64-under\endcsname{1}
\expandafter\def\csname ls-gpt56-64-over\endcsname{3}
\expandafter\def\csname ls-gpt56-128-R\endcsname{0.91}
\expandafter\def\csname ls-gpt56-128-P\endcsname{0.95}
\expandafter\def\csname ls-gpt56-128-joint\endcsname{8}
\expandafter\def\csname ls-gpt56-128-empty\endcsname{0}
\expandafter\def\csname ls-gpt56-128-k\endcsname{12}
\expandafter\def\csname ls-gpt56-128-planned\endcsname{12}
\expandafter\def\csname ls-gpt56-128-under\endcsname{9}
\expandafter\def\csname ls-gpt56-128-over\endcsname{3}
\expandafter\def\csname ls-gpt56-256-R\endcsname{0.34}
\expandafter\def\csname ls-gpt56-256-P\endcsname{0.76}
\expandafter\def\csname ls-gpt56-256-joint\endcsname{0}
\expandafter\def\csname ls-gpt56-256-empty\endcsname{5}
\expandafter\def\csname ls-gpt56-256-k\endcsname{12}
\expandafter\def\csname ls-gpt56-256-planned\endcsname{12}
\expandafter\def\csname ls-gpt56-256-under\endcsname{10}
\expandafter\def\csname ls-gpt56-256-over\endcsname{1}
\expandafter\def\csname ls-gpt56-324-R\endcsname{0.29}
\expandafter\def\csname ls-gpt56-324-P\endcsname{0.65}
\expandafter\def\csname ls-gpt56-324-joint\endcsname{0}
\expandafter\def\csname ls-gpt56-324-empty\endcsname{6}
\expandafter\def\csname ls-gpt56-324-k\endcsname{12}
\expandafter\def\csname ls-gpt56-324-planned\endcsname{12}
\expandafter\def\csname ls-gpt56-324-under\endcsname{9}
\expandafter\def\csname ls-gpt56-324-over\endcsname{3}
\expandafter\def\csname ls-gem38-8-R\endcsname{0.96}
\expandafter\def\csname ls-gem38-8-P\endcsname{0.98}
\expandafter\def\csname ls-gem38-8-joint\endcsname{7}
\expandafter\def\csname ls-gem38-8-empty\endcsname{0}
\expandafter\def\csname ls-gem38-8-k\endcsname{12}
\expandafter\def\csname ls-gem38-8-planned\endcsname{12}
\expandafter\def\csname ls-gem38-8-under\endcsname{4}
\expandafter\def\csname ls-gem38-8-over\endcsname{1}
\expandafter\def\csname ls-gem38-16-R\endcsname{0.93}
\expandafter\def\csname ls-gem38-16-P\endcsname{0.97}
\expandafter\def\csname ls-gem38-16-joint\endcsname{8}
\expandafter\def\csname ls-gem38-16-empty\endcsname{0}
\expandafter\def\csname ls-gem38-16-k\endcsname{12}
\expandafter\def\csname ls-gem38-16-planned\endcsname{12}
\expandafter\def\csname ls-gem38-16-under\endcsname{3}
\expandafter\def\csname ls-gem38-16-over\endcsname{0}
\expandafter\def\csname ls-gem38-32-R\endcsname{0.77}
\expandafter\def\csname ls-gem38-32-P\endcsname{0.99}
\expandafter\def\csname ls-gem38-32-joint\endcsname{7}
\expandafter\def\csname ls-gem38-32-empty\endcsname{1}
\expandafter\def\csname ls-gem38-32-k\endcsname{12}
\expandafter\def\csname ls-gem38-32-planned\endcsname{12}
\expandafter\def\csname ls-gem38-32-under\endcsname{5}
\expandafter\def\csname ls-gem38-32-over\endcsname{0}
\expandafter\def\csname ls-gem38-42-R\endcsname{0.67}
\expandafter\def\csname ls-gem38-42-P\endcsname{0.97}
\expandafter\def\csname ls-gem38-42-joint\endcsname{1}
\expandafter\def\csname ls-gem38-42-empty\endcsname{0}
\expandafter\def\csname ls-gem38-42-k\endcsname{12}
\expandafter\def\csname ls-gem38-42-planned\endcsname{12}
\expandafter\def\csname ls-gem38-42-under\endcsname{11}
\expandafter\def\csname ls-gem38-42-over\endcsname{0}
\expandafter\def\csname ls-qwen3vl-8-R\endcsname{0.89}
\expandafter\def\csname ls-qwen3vl-8-P\endcsname{0.87}
\expandafter\def\csname ls-qwen3vl-8-joint\endcsname{4}
\expandafter\def\csname ls-qwen3vl-8-empty\endcsname{0}
\expandafter\def\csname ls-qwen3vl-8-k\endcsname{12}
\expandafter\def\csname ls-qwen3vl-8-planned\endcsname{12}
\expandafter\def\csname ls-qwen3vl-8-under\endcsname{3}
\expandafter\def\csname ls-qwen3vl-8-over\endcsname{3}
\expandafter\def\csname ls-qwen3vl-16-R\endcsname{0.88}
\expandafter\def\csname ls-qwen3vl-16-P\endcsname{0.97}
\expandafter\def\csname ls-qwen3vl-16-joint\endcsname{4}
\expandafter\def\csname ls-qwen3vl-16-empty\endcsname{0}
\expandafter\def\csname ls-qwen3vl-16-k\endcsname{12}
\expandafter\def\csname ls-qwen3vl-16-planned\endcsname{12}
\expandafter\def\csname ls-qwen3vl-16-under\endcsname{10}
\expandafter\def\csname ls-qwen3vl-16-over\endcsname{0}
\expandafter\def\csname ls-qwen3vl-32-R\endcsname{0.77}
\expandafter\def\csname ls-qwen3vl-32-P\endcsname{0.90}
\expandafter\def\csname ls-qwen3vl-32-joint\endcsname{3}
\expandafter\def\csname ls-qwen3vl-32-empty\endcsname{0}
\expandafter\def\csname ls-qwen3vl-32-k\endcsname{12}
\expandafter\def\csname ls-qwen3vl-32-planned\endcsname{12}
\expandafter\def\csname ls-qwen3vl-32-under\endcsname{9}
\expandafter\def\csname ls-qwen3vl-32-over\endcsname{1}
\expandafter\def\csname ls-qwen3vl-64-R\endcsname{0.39}
\expandafter\def\csname ls-qwen3vl-64-P\endcsname{0.82}
\expandafter\def\csname ls-qwen3vl-64-joint\endcsname{0}
\expandafter\def\csname ls-qwen3vl-64-empty\endcsname{3}
\expandafter\def\csname ls-qwen3vl-64-k\endcsname{12}
\expandafter\def\csname ls-qwen3vl-64-planned\endcsname{12}
\expandafter\def\csname ls-qwen3vl-64-under\endcsname{11}
\expandafter\def\csname ls-qwen3vl-64-over\endcsname{1}

\newcommand{\CgainN}{10}
\newcommand{\CgainRhoSw}{$-$0.64}
\newcommand{\CgainPSw}{0.04}

\newcommand{\CcostPairs}{470}

\newcommand{\CreachN}{34}
\newcommand{\CreachGemThreeEight}{19}
\newcommand{\CreachGemFour}{32}

\usepackage[capitalize,nameinlink]{cleveref}
\theoremstyle{plain}
\newtheorem{lemma}{Lemma}
\newcommand{\restatename}{}
\newtheorem*{restatedthm}{\restatename}
\newenvironment{restate}[3]{\renewcommand{\restatename}{#1~\ref*{#2}}\begin{restatedthm}[#3]}{\end{restatedthm}}
\theoremstyle{definition}
\newtheorem*{restateddef}{\restatename}
\newenvironment{restatedef}[3]{\renewcommand{\restatename}{#1~\ref*{#2}}\begin{restateddef}[#3]}{\end{restateddef}}
\theoremstyle{plain}
\newcommand{\main}[1]{\ref*{#1}}
\newcommand{\cw}{c_{\mathrm{w}}}
\newcommand{\Jw}{J_{c_{\mathrm{w}}}}


\crefname{assumption}{Assumption}{Assumptions}
\crefname{lemma}{Lemma}{Lemmas}
\hypersetup{pdftitle={Why VLMs Miss Small Objects, and When Zooming In Is Safe},pdfauthor={Junzhe Shi, Yuan Gan, Shida Jiang}}

\title{Why VLMs Miss Small Objects,\\and When Zooming In Is Safe}

\author{Junzhe Shi\thanks{Corresponding author.} \\
  Systems Engineering \\
  University of California, Berkeley \\
  Berkeley, CA, USA \\
  \texttt{junzhe@berkeley.edu} \\\And
  Yuan Gan \\
  Department of AI \\
  Quotr AI \\
  Emeryville, CA, USA \\
  \texttt{yuan@quotr.ai} \\\And
  Shida Jiang \\
  Systems Engineering \\
  University of California, Berkeley \\
  Berkeley, CA, USA \\
  \texttt{shida\_jiang@berkeley.edu}}

\begin{document}
\maketitle
\begin{abstract}
Screenshots, documents, engineering drawings, shelves and maps hold many small objects of kinds that change from task to task, and vision-language models (VLMs) promise to find them from an instruction and an example. Yet they see each image through a fixed token budget and miss small objects, so every user, and every agent that calls a model, must choose whether to send an image whole, raise its resolution, split it, or zoom. We show that VLMs miss small objects for two reasons a caller can estimate before any call: an object gets too few visual tokens ($S$), and a call must search and report too much ($L$). Raising $S$ by a factor $a$ multiplies the tokens needed to see an image by $a^2$, while splitting at the same scale lowers $L$ at a cost at most linear in the number of views. Under two monotonicity assumptions, splitting cannot lower expected recall if no view zooms out and views overlap by one object, and neither part can be dropped; every such split pays a minimum token cost, and a simple rule pays it up to an explicit factor. Better models need less $S$, but an image can be arbitrarily large while a call's budget is fixed, so splitting does not become obsolete, for zoom agents either. In about \CscaleCalls{} requests to \CscaleModels{} VLMs, up to GPT-6.1, Claude Opus 5.5 and Gemini 4, on controlled images \XasSup{} of \XasN{} predicted orderings hold significantly and none is violated; splits that never zoom out raised recall on average over models wherever $S$ was below 1.5, and splits that zoom out lowered it in \XtcOutNegPoint{} of \XtcOutN{} model and data settings. The rule raised recall significantly in \CtSigRec{} of ten settings, by up to \UcsMax{} (from \CtLsGemThreeEightRecW{} to \CtLsGemThreeEightRecR), and precision and F1 in \CtSigPrec, at a token cost known before any call.
\end{abstract}

\begin{figure*}[t]
\centering
\includegraphics[width=\textwidth]{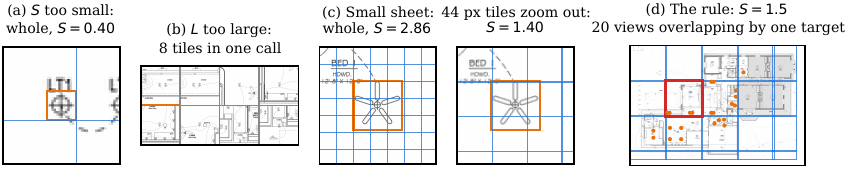}
\caption{The two quantities and the fix. $S$ is how many visual tokens an object spans per side; $L$ is how much one call must search and report. (a) An $\CfigW{} \times \CfigH$ px sheet sent whole within 2{,}500 tokens of 32 px: one of its \CfigTargets{} targets spans $S = \CfigSwhole$ tokens (blue: token grid; orange: target box). (b) Packing eight same-scale tiles into one call keeps $S$ and raises $L$; for GPT-5.6, recall on the 14 full sheets falls from \XmosRgptFiveSixkOne{} to \XmosRgptFiveSixkEight, although a packed call must report a median of only \CpackMedian{} targets. (c) A pixel-defined recipe (44 px symbols, 1{,}024 px tiles) zooms in on large sheets but out on a small one, from $S = \CsmallSw$ to \CsmallSt; GPT-5.4's recall falls from \CsmallRecWhole{} to \CsmallRecTile. (d) The rule of \cref{sec:safe}: \CfigViews{} views at $S = 1.5$ that overlap by one target, \CfigTokens{} tokens in total, against the minimum of \cref{thm:exact}(c) for any split that protects the sheet at this $S$ (blue: view borders; red: one view; orange dots: targets). (a), (c), (d): 32 px tokens, 2{,}500 per image; (a) and (c) show the pixels the model receives. Recalls in (b) and (c) are measured.}
\label{fig:intro}
\end{figure*}

\section{Introduction}
\label{sec:intro}

Many images people work with are large and dense: a high-resolution screenshot holds dozens of icons and controls, a document page hundreds of fields and marks, an engineering drawing hundreds of symbols, and there are shelf photographs, maps, tissue slides and satellite scenes. The objects are small and many, and their kinds change from task to task, so training a detector for each rarely pays. Vision-language models (VLMs) turn this into one instruction and one example, and they are becoming the perception layer of computer-use, document and design agents~\cite{su2025thinkingwithimages,zheng2026deepeyes}. But a VLM sees an image through a fixed token budget and misses small objects. Every user, and every agent that calls a VLM, must decide whether to send the image whole, raise its resolution, split it, or zoom, and the choices differ several-fold in accuracy and cost.

\paragraph{Our answer, in brief.}
VLMs miss small objects for two reasons that a caller can estimate before any call: an object gets too few visual tokens, and a call must search and report too much. Splitting the image cannot lower expected recall when no view zooms out and views overlap by one object, and neither part of this condition can be dropped. The guarantee assumes the model is monotone, which held on every standard input we tested (\cref{sec:safe}) and failed only when OpenAI models were sent larger images than our standard input (images we resize ourselves so that the vendor need not, \cref{sec:protocol}). In short: \emph{never zoom out, and overlap views by one object}. Protecting every target at a given $S$ has a minimum token price that no split avoids, and a simple rule pays it up to an explicit factor, so a caller can decide before any call how to send an image and what it will cost, for future models too. On the newest models, the rule raises recall most where the whole image gives targets the fewest tokens (\cref{sec:practice}).

\paragraph{Two quantities and their price.}
Each call gives the model an image and an example of the target and asks for the centre of every instance; a point counts if it lies within half a target of an instance not already matched. $S$ says whether the model can see the target (the object to find): the target's side in the image the model receives, in units of one visual token ($S = 1.5$: one and a half tokens wide). It follows from public interface rules and the image size, with no label and no call. An $\CfigW{} \times \CfigH$ px drawing sent whole within 2{,}500 tokens of 32 px shrinks its \CfigM{} px symbols to about \CfigMsent{} px, \CfigPctToken\% of a token ($S = \CfigSwhole$, \cref{fig:intro}a). $L$ says how much one call must do: \emph{search} a view, with its targets, look-alikes and lines, and \emph{report} one accurate position per target. If $S$ is the size of the print, $L$ is how many pages one must proofread at once; clear print does not help when too many pages arrive together.
The two fixes are priced very differently. Seeing a whole image at $S$ takes at least its area over the target's area times $S^2$ tokens, and more with the margins that protect every target, so raising $S$ by a factor $a$ raises this minimum $a^2$-fold: the drawing above needs at least about \CfigCmin{} tokens at $S = 1.5$, against 2{,}500 sent whole. Splitting into $k$ views at the same scale cuts each view's area to about $1/k$ at a cost at most linear in $k$: nearly free for models billed by resolution (GPT, Claude, Qwen), and $k$ images for fixed-grid models (Gemini), which also raises their $S$. A high-resolution mode raises $S$ but keeps $L$ at the whole image; a zoom agent sets both region by region.

\paragraph{Three surprises that the two quantities explain.}
\emph{More pixels can hurt}: sent through GPT-5.6's high-resolution mode with reasoning, the 14 full drawing sheets give each target more tokens, yet F1 falls from \XnFAOneRgptFiveSixSheets{} at 1{,}600 px to \XnFNATRgptFiveSixSheets; on this input path (the route by which an image reaches the model) the model stops satisfying ``more tokens do not hurt''. \emph{Splitting can zoom out}: recipes that rescale an image to a pixel size before tiling ignore how the vendor rescales the whole image. A 44 px recipe (44 px symbols, 1{,}024 px tiles) raises $S$ on the large drawing from \CfigSwhole{} to about 1.4 but lowers it on a $\CsmallW{} \times \CsmallH$ px drawing from \CsmallSw{} to \CsmallSt, and on this sheet GPT-5.4's recall falls from \CsmallRecWhole{} to \CsmallRecTile{} (\cref{fig:intro}c); where symbols were already large, it lowered recall in \XtcOutNegPoint{} of \XtcOutN{} model and data settings, significantly in \XtcOutNeg. \emph{Seeing has a price}: the bound above holds for every split that sees the whole image at $S$, whoever splits. Such effects are found by trial, and answers tuned for one model expire with the next; theorems can give what trial cannot: conditions checked before a call, guarantees for every model with two testable properties, a cost no protected split can beat, and conditions whose parts are each necessary, which explains both failures. We test them in about \CscaleCalls{} requests to \CscaleModels{} models, including the newest generation, on \CscaleImages{} images with about \CscaleInst{} instances.

\begin{answerbox}
\textbf{Three answers.}
\textbf{(1) Why VLMs miss small objects:} too few tokens per object ($S$) and too much to search and report in one call ($L$). The interface caps $S$ and prices seeing, the same for every model (\cref{sec:why}).
\textbf{(2) What better models change:} they need less $S$ and handle more per call, so fewer images need splitting. But an image can be arbitrarily large while a call's budget is fixed, so seeing a whole image keeps a minimum cost, for zoom agents too, and splitting does not become obsolete (\cref{sec:why}).
\textbf{(3) When splitting is safe and what it costs:} no zoom-out and one object of overlap, with monotonicity, suffice, and neither part can be dropped; protection has a minimum cost and a rule pays it up to an explicit factor; the guarantee degrades gracefully. What changes is who splits: the caller, a zoom tool or the model (\cref{sec:safe,sec:practice}).
\end{answerbox}

\begin{table}[t]
\centering
\small
\setlength{\tabcolsep}{2pt}
\begin{adjustbox}{max width=\linewidth}\begin{tabular}{lccc}
\toprule
Model & Budget $N$ & Targets & Targets \\
 & per view & per side & per call \\
\midrule
GPT-6.1 (GPT-5.6) & 8{,}192 & about \CthumbSideGpt{} & 300+ (about 130) \\
Claude Opus 5.5 & 4{,}784 & about \CthumbSideClaude{} & about 130 \\
Gemini 4 (Gemini 3.8) & about 1{,}090 & about \CthumbSideGemini{} & 40+ (about 16) \\
Qwen3-VL & about 2{,}500 & about \CthumbSideQwen{} & about 16 \\
\bottomrule
\end{tabular}\end{adjustbox}
\caption{A rule of thumb at $S = 1.5$: values at pooled recall about 0.9 on 12 controlled images per load, for one image content (one target type among three look-alike types), with the reasoning settings of \cref{sec:protocol}; in parentheses, the previous model of the vendor; ``+'': a full view showed no drop. Budget $N$ in tokens; targets per side: about $\sqrt{N}/1.5$. Details in \supp{L}.}
\label{tab:thumb}
\end{table}

\paragraph{What a user can take away.}
The rule protects at $S^\ast = 1.5$, above every model's half-recall point (\cref{sec:why}); the $L$ a call can carry grows with each generation (\cref{tab:thumb}). A rule of thumb needs only the target's size $m$ in pixels, the image size and the per-view budget $N$:
\textbf{(i) If the target is small, enlarge} it to about 48 px in the sent image (GPT, Qwen; 32 px tokens) or 42 px (Claude; 28 px tokens); fixed-grid models skip this step.
\textbf{(ii) If the image is large, split} it into views without zooming out. Use $S = \max(1.5, \Sw)$, where $\Sw$ is the $S$ of the whole image; each view is then at most about $m\sqrt{N}/S$ pixels, or $\sqrt{N}/S$ targets, on a side, and overlaps its neighbours by one target.
\textbf{(iii) If targets are many or the content busy, split further} at the same scale, to about the counts in \cref{tab:thumb}.
In all our data, splits that never zoom out raised recall by \CmapBelowMin{} to \CmapBelowMax{} on average where $S < 1.5$, and splits that zoom out lowered it by \CmapZoomMin{} to \CmapZoomMax{} (\cref{fig:map}).

\paragraph{Contributions.}
\textbf{(1) When splitting is safe and how little it can cost, answered exactly and constructively.} ``No zoom-out, one object of overlap'' guarantees that splitting does not lower expected recall for every monotone model, and each part is necessary. Every protected split pays a minimum token cost, and a directly executable rule pays at most this minimum times an explicit factor that tends to one for ideal views; with approximate assumptions, recall degrades by at most the measured defects. The condition covers every local task with one position and one answer per object, including detection and segmentation.
\textbf{(2) Why VLMs miss small objects, in one interface-level account that also predicts future models.} The interface caps $S$ and prices seeing, quadratically for $S$ and at most linearly for $L$; models set the $S$ they need and what one call can handle; a frontier of the same $S^2$ form binds every search strategy, zoom agents included, whenever recognition is graded, so better models do not make splitting obsolete; \cref{sec:future} says where the next models should improve.
\textbf{(3) Tests at scale, including the newest models.} Across \CscaleModels{} models, controlled tests give the assumptions measurable content; across all data, splits that never zoom out gain where $S$ is low and splits that zoom out lose; in \CtCells{} model and data settings the rule raised recall significantly in \CtSigRec{} and precision and F1 in \CtSigPrec, most where the whole image gives targets the fewest tokens, at a cost computable in advance; the effects of $S$ and $L$ carry over to natural images and zoom agents.
We release the LegendSpot benchmark (40 architectural drawings, 1{,}028 instances of legend symbols), a planning tool for safe splitting, the code and all model outputs at \url{https://github.com/shijunzhe/vlm-small-objects}.

\section{Related work}
\label{sec:related}

\paragraph{High-resolution VLMs.} The accuracy of VLMs drops with object size and with distractors~\cite{zhang2026perceptual,zhang2025mllmsknow}, and they lose small targets in large or cluttered inputs~\cite{wu2025visualhaystacks,sharma2024losing,zhang2025hrscene,zhang2025mmerealworld}. Dynamic tiling and any-resolution encoders~\cite{chen2024internvl15,xu2024llavauhd,bai2025qwen25vl} change the model to use more resolution. We leave the model fixed and derive what the interface allows in one call, a ceiling on $S$, and what seeing a whole image must cost.

\paragraph{Splitting and zooming at inference time.} Slicing Aided Hyper Inference (SAHI) cuts fixed-size slices at native resolution with a fixed overlap ratio for a detector~\cite{akyon2022sahi}; its slices never zoom out, but their $S$ is set by the object's pixel size rather than by what the model needs. DC$^2$ splits recursively into quadrants and merges text descriptions~\cite{wang2025dc2}, V$^\ast$ searches quadrants guided by a heatmap~\cite{wu2024vstar}, and ZoomEye runs a tree search over sub-patches ranked by the model's confidence~\cite{shen2024zoomeye}; ViCrop picks one crop from the model's attention or gradients~\cite{zhang2025mllmsknow}. These methods let heuristics or the model decide where to look. We add a protection condition that a caller can check before any call, the minimum cost of protection, and a construction within an explicit factor of it; for example, a $2 \times 2$ split without overlap leaves 9 to 19\% of our targets cut by a boundary (\supp{G}). Agents that zoom while reasoning~\cite{openai2025thinking,zheng2026deepeyes,lai2026minio3,zhang2025chainoffocus,su2025thinkingwithimages} choose their own views; \cref{prop:frontier} bounds what such a strategy must spend when recognition is graded, and we test four such agents (\cref{sec:practice}). \citet{liu2026hide} argue that zooming helps mainly by removing background rather than by enlarging the target; $S$ and $L$ separate the two effects, and which one dominates depends on the model (\cref{sec:agents}).

\paragraph{Many similar objects.} Errors of VLMs grow with the number of objects and with how many features they share, which \citet{campbell2024binding} attribute to the binding problem. Autoregressive detectors tend to stop before they have listed every object, and delaying the stop yields noisy and duplicated predictions~\cite{chen2022pix2seq}. We measure the combined effect of such limits at the level of a call, and give the remedy that a caller controls. Our companion study~\cite{companion2026readout} shows that how a call is asked to report (its request, coordinate format and output structure) changes the returned set; we hold the request and format fixed and ask how the scale and content of a call decide what it can find, and when splitting it is safe.

\paragraph{Counting and visual prompts.} Class-agnostic counters localize instances from exemplars~\cite{amininaieni2024countgd,pelhan2026geco2}, and VLM counting has been improved by dividing the image~\cite{qharabagh2026lvlmcount}. Benchmarks include FSC-147~\cite{ranjan2021fsc147}, CountBench~\cite{paiss2023countbench}, PixMo-Count~\cite{deitke2025molmo} and CountQA~\cite{tamarapalli2025countqa}, where accuracy falls as counts rise. Overlays such as marks, grids, labelled dots and scan lines are reported to help grounding and counting~\cite{yang2023setofmark,lei2025scaffold,izadi2025viser}; by \cref{prop:dpi}, an overlay on the sent image adds no Shannon information about the targets.

\paragraph{Symbol spotting in drawings.} Symbol spotting and floor-plan parsing have dedicated datasets and methods~\cite{liu2017raster,kalervo2019cubicasa,fan2021floorplancad,luo2025archcad}, surveyed by \citet{morenogarcia2019trends}; FloorPlanCAD and ArchCAD-400K train dedicated models for panoptic symbol spotting on large annotated sets. Recent benchmarks show that VLMs struggle to recognize and count symbols on drawings~\cite{shteriyanov2025blueprintsymvl,kondratenko2026aecv,jung2026drawingvqa}, and normalized cross-correlation template matching~\cite{lewis1995fast,brunelli2009template} is a strong, training-free baseline when a clean exemplar is available. We study VLM calls with one example and no training.

\paragraph{Diagnosing VLM vision.} VLMs fail simple geometric tasks~\cite{rahmanzadehgervi2024blind}; attention analyses show how image information reaches text tokens~\cite{kaduri2025whats} and how localization is computed~\cite{schaumloeffel2026mechanisms}. These works study a model's inside; we study the scale, content and cost of a call. Usable information~\cite{xu2020usable} explains why enlarging can help without adding Shannon information.

\section{Experimental protocol}
\label{sec:protocol}

\paragraph{Task and answer.} Each call gives the model an image (or a view) and one reference image of the target and asks for a JSON (JavaScript Object Notation) object with a count and the list of centres of every instance, in the model's published convention (pixels, or 0 to 1000). When an image is split, points outside a view's core are discarded and the rest mapped back to the image. Prompts are in \supp{C}.

\paragraph{Data.} We use \CscaleImages{} images of four kinds: 50 architectural drawings, 40 of which form LegendSpot (1{,}028 instances; 6 sheets used only to write the prompt, and 34 test drawings: 14 full sheets and 20 small drawings) and 10 whose reference symbols come from a symbol library; 429 floor plans from FloorPlanCAD~\cite{fan2021floorplancad} (369 single-block items and 60 mosaics); 150 natural images from FSC-147~\cite{ranjan2021fsc147}; and \CscaleCtrl{} controlled images, 12 of them for the controlled test and the rest for recognition, load and newest-generation measurements. Details in \supp{C}; earlier versions of this paper use internal names for the data sets, listed in \supp{M}.

\paragraph{Models and settings.} We use \CscaleModels{} models from five families (GPT, Claude, Gemini, Gemma, Qwen), each read in its published coordinate convention. Claude Sonnet 5.5 and Opus 5.5 run with the API's default adaptive thinking. The Qwen3-VL agent uses the instruct model without reasoning. Our \emph{standard input} is input that we resize ourselves so that the vendor need not resize it again: whole-image calls with a long side of at most 1{,}600 px (in the controlled test, at most 2{,}048 px and within each vendor's budget), and views within the budgets of \cref{tab:thumb}; the OpenAI models keep all these inputs at their sent size. Reasoning is off where the interface allows it and otherwise at its lowest setting (GPT-6.1 at its lowest effort, Gemini 3.8 at the minimal thinking budget, Gemini 4 at the lowest thinking level). Two kinds of input lie outside the standard input: native-resolution drawings in a high-resolution mode, and 2{,}048 px images that GPT-5.4 resizes itself. Other settings are in \supp{C}.

\paragraph{Metrics and statistics.} \emph{Hit recall} is the fraction of targets with a returned point within $\tau$; \emph{matched recall} is $\mathrm{TP}/n$, with $\mathrm{TP}$ (true positives) the size of a maximum one-to-one matching within $\tau$; precision is $\mathrm{TP}$ over the returned points, and F1 is their harmonic mean. Hit recall is at least matched recall, with equality when targets are more than $2\tau$ apart. Theory tests use hit recall, configuration comparisons matched recall, precision and F1. With the image weight $a(X) = 1/n(X)$ of \cref{sec:safe}, the theorems apply to mean recall per image. We use paired bootstrap with 4{,}000 resamples over images (source drawings for mosaics), and Holm correction for the agent comparisons. Comparisons with the 44 px recipe are observational: it differs from the rule in scale, view size and overlap at once. Recall, precision and F1 are fractions; differences are absolute unless marked as relative, and relative gains are given with both ends.

\section{Why VLMs miss small objects}
\label{sec:why}

The interface caps $S$ and prices seeing, for every model; models decide the $S$ they need and how much one call can handle. Better models therefore need fewer splits, but they cannot remove the need to split.

\subsection{Two quantities, and what the interface fixes}

An image $X = [0,W) \times [0,H)$ holds $n$ instances of a target with a $w \times h$ px box; $m = \sqrt{wh}$ is the target scale and $\tau = m/2$ the matching tolerance. A \emph{call} sends the model one image derived from $X$ at effective scale $s$, the pixels received per native pixel after all resampling by the caller and the vendor. The model cuts the received image into $t \times t$ px tokens and accepts at most $N$; resolution-scaled interfaces (GPT, Claude, Qwen) resize to fit, and fixed-grid interfaces (Gemini, Gemma) map every image to $g \times g$ tokens. A target therefore spans
\begin{equation}
S = m s / t
\label{eq:S}
\end{equation}
tokens per side. Since $s$ depends on the image's size, the same pixel operation gives different $S$ on images of different sizes; this is why splitting can zoom out, and why our conditions are stated in $S$ rather than in pixels. The content $L$ of a call is its view with everything in it, that is, all that the call must search and report: $S$ governs whether the model can see, $L$ how much it must do at once; we use the targets in a view as a rough measure of $L$, and \cref{thm:capacity} gives one in bits. Proofs of all results are in \supp{A}.

\begin{proposition}[Interface ceiling and price of seeing]
\label{prop:geometry}
(a) A resolution-scaled call that shows the whole image gives $S \le m\sqrt{N/(WH)}$, a bound from the image's area, so doubling the budget raises the reachable $S$ by at most 41\%. (b) Covering the image with the kept cores of views shown at $S \ge S^\ast$ requires at least $(WH/m^2)\,{S^\ast}^2$ tokens, more when the views overlap; raising $S^\ast$ by a factor $a$ multiplies this cost by $a^2$. (c) For a resolution-scaled model, splitting into $k$ views at a fixed scale keeps the total at the tokens of the whole area plus overlaps, rounding and per-image prompts; a fixed-grid model pays $kN$ for $k$ views, and a view spanning $\ell$ native pixels on its longer side gets $S = mg/\ell$, which rises as views shrink. The area of each view thus falls to about $1/k$ at a token cost at most linear in $k$.
\end{proposition}

\begin{proposition}[Information]
\label{prop:dpi}
Overlays on the sent image, and its interpolation, add no Shannon information about the targets.
\end{proposition}

By \cref{prop:dpi}, grids, marks or interpolation applied to the sent image cannot restore what the resize removed, while a crop cut from the native image can (\cref{sec:proofs-interface}). For \cref{fig:intro}'s drawing, \cref{prop:geometry}(b) gives about \CfigPropB{} tokens at $S = 1.5$, and about \CfigCmin{} with the margins (\cref{thm:exact}c).

\subsection{\texorpdfstring{Cause one: $S$}{Cause one: S}}

Gemini 3.8 lets the caller set its token budget while the pixels stay fixed. On controlled images, recall of 8 px targets at its three budgets is \CbudEightLo, \CbudEightDef{} and \CbudEightUh; on floor-plan mosaics, moving from the low to the default budget adds \UfTOnesheetssmallDL{} F1 on items whose $S$ is below 1.5 and \UfTOnesheetslargeDL{} on the others. On controlled images, recall rises gradually with $S$, and the $S$ a model needs differs by model: the $S$ at which models recall half the targets ($S_{50}$) ranges from below 0.3 to 1.28 (\cref{fig:curves}a; per-model values in \supp{D}).

\begin{figure}[t]
\centering
\includegraphics[width=\columnwidth]{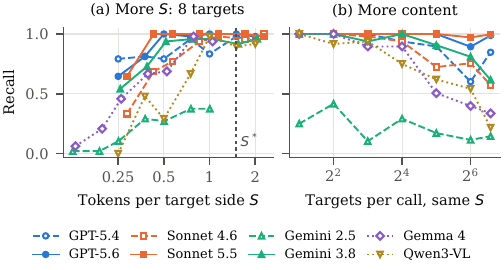}
\caption{Controlled images, eight models. (a) Recall rises gradually with $S$ (8 targets, 8 distractors); dashed line: $S^\ast = 1.5$. (b) At fixed $S$ (1.25 to 1.43 for resolution-scaled models), recall falls as a call must report more targets; lowest reasoning setting, Sonnet 5.5 with adaptive thinking. Dashed and hollow: older generation. $1024^2$ canvas, six images per point; mean matched recall per image; no intervals shown. The fixed-grid models in (b) have $S$ of 0.63 to 1.29.}
\label{fig:curves}
\end{figure}

\subsection{\texorpdfstring{Cause two: $L$}{Cause two: L}}

Adding content to a call while $S$ stays fixed lowers recall. On controlled images with the canvas, $S$ and tokens fixed, Qwen3-VL recalls \XcapNlastqwenThreevl{} of 96 targets, against \XcapNtwoqwenThreevl{} of two (\cref{fig:curves}b). On drawing sheets, packing eight tiles into one GPT-5.6 call at the same tokens per target lowers recall from \XmosRgptFiveSixkOne{} to \XmosRgptFiveSixkEight{} (difference \XmosdRgptFiveSixkEight, interval \XmosdRCIgptFiveSixkEight), although a packed call must report a median of only \CpackMedian{} targets (mean \CpackMean, at most \CpackMax). The failures occur in answers that end normally: in the controlled sweep \CcapNormal{} of \CcapCalls{} calls ended normally (the other is a Gemini 3.8 call whose thinking used the output cap), and the longest visible answer used \CcapMaxOut{} of the 8{,}192 tokens allowed. Where recall starts to fall depends on the model, content and reasoning setting: on the controlled images Sonnet 4.6 and GPT-5.4 drop from about 32 targets and GPT-5.6 and Sonnet 5.5 hold to 96, yet GPT-5.6 drops on packed drawing tiles. Lowering $L$ helps only where $L$ binds; splitting without enlarging does not help when $S$ is too small, and in the zoom-agent ablation (\cref{sec:agents}) lowering $L$ alone gave no significant gain, while enlarging gave GPT-5.6 \UnHThreegptFiveSixOVZOVC{} F1.

Part of the content limit holds for every model. Reporting $n$ positions in an area $A$ to within $\tau$ needs information:

\begin{theorem}[What a hit costs]
\label{thm:capacity}
For $n$ target centres independent and uniform on a region of area $A$ and any returned set $\hat Y$, $\mathbb{E}[\mathrm{Rec}] \cdot \Lbit \le I(Y; \hat Y) + o$, with $\mathrm{Rec}$ the matched recall and $\Lbit = n \log_2\!\big(A/(n\pi\tau^2)\big)$, where $o$ charges unmatched returned points and is $\log_2(n+1)$ without them.
\end{theorem}

Each hit therefore costs about $\log_2$ of the ratio between the area per target and the tolerance disk, in bits, so an answer that carries $C$ bits about positions can find at most about $(C + o)/\Lbit$ of the targets.

\subsection{What better models change, and why splitting stays}

\emph{The $S$ a model needs falls with each generation.} On the same token grid and the same controlled images, Gemini 3.8 recalls \UtOneHThree{} more than Gemini 2.5; and Gemini 4 receives the same 1{,}090 tokens per image as Gemini 3.8, yet the $S$ it needs for recall and precision of 0.9 falls from \CgenSninetyGemThreeEight{} to \CgenSninetyGemFour{} (both at the vendor's default reasoning; \supp{D}). On this grid the whole image already reaches that $S$ on \CreachGemThreeEight{} of the \CreachN{} LegendSpot drawings for Gemini 3.8 and on \CreachGemFour{} for Gemini 4, so fewer images need splitting.
\emph{What one call can handle grows.} With 96 targets on the same controlled images, recall rises from \XcapNlastgptFiveFour{} (GPT-5.4) to \XcapNlastgptFiveSix{} (GPT-5.6) and from \XcapNlastsonnetFourSix{} (Sonnet 4.6) to \XcapNlastsonnetFiveFive{} (Sonnet 5.5, with adaptive thinking); in our load sweep at $S = 1.5$, GPT-6.1 holds at \CgenGptSixOneMax{} targets (\cref{tab:thumb}, \supp{L}).
\emph{But the need to split stays.} An image can be arbitrarily large while a call's budget is fixed, and at the $S$ a model needs one call sees only a bounded area (\cref{prop:geometry}): each generation raises how much one call can handle, but no model removes the need to split.

The minimum cost also binds strategies that choose their own views. Let a call carry several images, as when an agent keeps earlier views, and count all their tokens.

\begin{proposition}[Recall-cost frontier]
\label{prop:frontier}
Suppose a policy sees the image only through the model's responses and makes a bounded number of calls, and that, given any history, adding one target changes the law of the model's response by at most $\sum_I f(S_I)$ in total variation, summed over the images $I$ of the call that show it at $S_I$, with $f$ non-decreasing (graded recognition). To find a target placed anywhere in a region $\Omega$ with average probability $\bar R$, the policy's expected tokens on images without that target are at least
\[
\frac{\bar R\,|\Omega| - \pi\tau^2\,\overline{\mathrm{FP}}_0}{\beta\,\kappa_f\, m^2}, \qquad \kappa_f = \sup_{S>0} f(S)/S^2,
\]
where $\overline{\mathrm{FP}}_0$ counts its false points there and $\beta \ge 1$ bounds $(1 + w/a_x)(1 + h/a_y)$ over its $a_x \times a_y$ views.
\end{proposition}

The bound covers policies that adapt, mix scales, randomize, or trade false points for hits, such as the zoom agents we test (\cref{sec:practice}). A high-resolution mode changes $S$ but not $L$, so it addresses at most the first cause.

\section{When splitting is safe, and what it costs}
\label{sec:safe}

Splitting cannot lower expected recall when no view zooms out and views overlap by one object, and neither part can be dropped. Every such split pays a minimum cost, and our rule pays it up to an explicit factor. All data show both predictions (\cref{fig:map}), and the guarantee fails exactly at the two places the theory names.

\paragraph{What the caller controls and what is tested.} A call $c$ shows a rectangular view $V_c$ of a source image at scale $s_c$, keeps the returned points inside a core $K_c \subseteq V_c$, and uses a fixed prompt. A target is \emph{interior} to $c$ if its box lies in $V_c$ and its centre at least $\tau$ inside $K_c$; $J_c$ is the set of interior centres. A \emph{model} gives a response law $P_X(\cdot \mid c)$ for every image $X$ and call $c$. For a distribution $\mathcal{D}$ of images (an application, a benchmark, or one image), a weight $a(X) \ge 0$ and a region $U$, $\nu_c(U) = \mathbb{E}[a(X) \sum_{x \in Y(X) \cap U \cap J_c} H_c(x)]$, where $Y(X)$ holds the target centres and $H_c(x) = 1$ if $c$ returns a kept point within $\tau$ of $x$; $a = 1$ counts detections, $a = 1/n(X)$ gives mean recall per image, and every result below holds for each $a$. A call $c'$ is a \emph{crop} of $c$ if it shows a sub-region of the same pixels at $s_{c'} \ge s_c$, and it \emph{refines} $c$ if it shows the same region at $s_{c'} \ge s_c$. The caller controls the geometry: no view zooms out, and views overlap by one object. The behaviour of the model is assumed, and tested:

\par\noindent\begin{minipage}{\linewidth}
\begin{assumption}[Monotonicity]
\label{ass:mono}
For all calls $c, c'$ and every region $U \subseteq J_c \cap J_{c'}$: (R) if $c'$ refines $c$, then $\nu_{c'}(U) \ge \nu_c(U)$: more tokens per target do not hurt; (M) if $c'$ is a crop of $c$ at the same scale, then $\nu_{c'}(U) \ge \nu_c(U)$: less content does not hurt.
\end{assumption}
\end{minipage}

$\mathcal{M}_{R,M}$ denotes the models satisfying Assumption~\ref{ass:mono}, whose two properties are averages that an experiment can estimate (zoom tools need one more property, \supp{A}).

\subsection{The exact condition for protection}

\begin{theorem}[Safe splitting]
\label{thm:safe}
Let the whole-image call show the image at $\Sw$, and let views $c_1, \dots, c_k$ be crops of it, so that $S_i \ge \Sw$. Under (R) and (M), if the cores shrunk by $\tau$ cover the image shrunk by $\tau$ and each view extends its core by half a target box on each side, the views together detect, in expectation, at least as many of the targets interior to the whole-image call as it does. This sufficient construction fixes the cores first and then adds the margin; in its terms, adjacent cores overlap by one target.
\end{theorem}

\Cref{thm:safe} establishes safety; \cref{thm:exact} gives the exact criterion and the cost lower bound; \cref{cor:safe} gives a construction within an explicit factor of it.

\begin{theorem}[Exact condition and minimum cost]
\label{thm:exact}
Let $\mathrm{d}$ be any finite set of calls. (a) If every centre interior to the whole-image call is interior to some call of $\mathrm{d}$ that is a crop or a refinement of it, then under (R) and (M), $\mathrm{d}$ detects in expectation at least as many of the targets interior to the whole-image call as it does, for every $\mathcal{D}$. (b) If target centres are more than $2\tau$ apart and some location interior to the whole-image call is interior to no such call, there is a $P \in \mathcal{M}_{R,M}$ under which the whole-image call detects every target interior to it and $\mathrm{d}$ never detects a target there. (c) If every interior centre is interior to such a call with $S_c \ge S \ge \Sw$, these calls use at least
\[
\Cmin = |J_{\mathrm{w}}|\,\lambda/(1 - \sigma)^2
\]
tokens, whatever the shapes, sizes, cores and overlaps of their views, where $|J_{\mathrm{w}}| = (W - e_x)(H - e_y)$ with margins $e_x = \max(w, 2\tau)$, $e_y = \max(h, 2\tau)$, $\lambda = S^2/m^2$ tokens per native pixel, and $\sigma = \sqrt{\lambda e_x e_y / N} < 1$.
\end{theorem}

\begin{proposition}[Each part is necessary]
\label{prop:tight}
For targets more than $2\tau$ apart: (a) if a target lies in the core of no view with $S_i \ge \Sw$, some $P \in \mathcal{M}_{R,M}$ makes the whole-image call find every target interior to it and the split never find that one; (b) neither (R) nor (M) can be dropped: for each, a model satisfying the other loses every target under some split while the whole-image call finds every target interior to it.
\end{proposition}

Protection is thus decided location by location: call a crop or refinement of the whole-image call \emph{qualified}; a split is protected for every model in $\mathcal{M}_{R,M}$ exactly when the interior set of the whole-image call satisfies \mbox{$J_{\mathrm{w}} \subseteq \bigcup\nolimits_{i \in I_{\mathrm{qual}}} J_i$}, the union over its qualified views (for targets more than $2\tau$ apart). A view that zooms out is not qualified, and never zooming out, with one object of overlap, meets the criterion. The same condition protects every local task with one position and one answer per object (\supp{A}).

\subsection{The minimum cost, and a rule within a factor of it}

\begin{algorithm}[t]
\caption{Safe splitting near the minimum cost}
\label{alg:rule}
\begin{algorithmic}[1]
\Require image $W \times H$, target box $w \times h$, interface ($t$, $N$ or $g$), protection point $S^\ast$
\State $S \gets \max(S^\ast, \Sw)$ \Comment{never zoom out}
\State cut the image at scale $S$ into the largest views within the budget, overlapping by $e_x$ and $e_y$, each its own core ($K_i = V_i$) and sent without further resizing
\State \emph{optional:} if a view holds many targets or busy content, also send sub-views of it at the same scale, overlapping by one target, and keep the view's points too (a subdivision fixed in advance for every view may instead replace the views)
\State send each view; keep its points, map them back and return them all, without merging duplicates
\end{algorithmic}
\end{algorithm}

\begin{corollary}[A safe split near the minimum]
\label{cor:safe}
Steps 1 and 2 of \cref{alg:rule}, with views $a_x \times a_y$ at pitches $p_x = a_x - e_x$ and $p_y = a_y - e_y$, detect in expectation at least as many of the targets interior to the whole-image call as it does under (R) and (M), and use at most $\Gamma\,\Cmin$ tokens, $\Gamma = \chi\,(1 + p_x/(W - e_x))(1 + p_y/(H - e_y))$, excluding per-image overhead. The rounding factor $\chi \ge 1$ equals 1 for ideal views of whole patches, for which the cost approaches $\Cmin$ as the image grows.
\end{corollary}

Hence $\Cmin \le C_{\mathrm{rule}} \le \Gamma\,\Cmin$ in image tokens, with $\Gamma$ set by the image, the target box, $S$ and the interface. Step 3 has two forms (\supp{I}). Sub-views \emph{appended} to a view cannot lower hits or matched true positives on any image, with no assumption, so the choice of which views to subdivide may depend on their content or answers; this form costs \CstepAppMin{} to \CstepAppMax{} times a view's tokens for resolution-scaled models (\CstepGridApp{} for fixed grids). Sub-views that \emph{replace} a view satisfy \cref{thm:exact}(a) when the subdivision is fixed in advance by the image's size, at \CstepRepMin{} to \CstepRepMax{} times its tokens. \Cref{tab:main} evaluates the rule without step 3.

\paragraph{Choosing the protection point.} $S^\ast = 1.5$ lies above the half-recall point of every model we tested, and \cref{cor:safe} holds for any $S^\ast$. The assumptions fix the sign of the gain, not its size:

\begin{proposition}[Attainable recall gains]
\label{prop:nonid}
Let target centres be more than $2\tau$ apart, with a continuous intensity. For any $0 \le r_{\mathrm{w}} \le r_{\mathrm{d}} \le 1$ some $P \in \mathcal{M}_{R,M}$ gives the whole-image call recall $r_{\mathrm{w}}$ and a covering split of \cref{thm:safe}, none of whose views is the whole image at the whole-image call's scale, recall $r_{\mathrm{d}}$ (weighted ratios of expected counts over the targets interior to the whole-image call).
\end{proposition}

Whether raising $S^\ast$ from 1.5 to 3 pays therefore depends on the model; at $S = \max(S^\ast, \Sw)$ it multiplies the area term of the cost by up to four (\cref{thm:exact}c); see \cref{fig:map}.

\subsection{Evidence across all data}

\begin{figure*}[t]
\centering
\includegraphics[width=\textwidth]{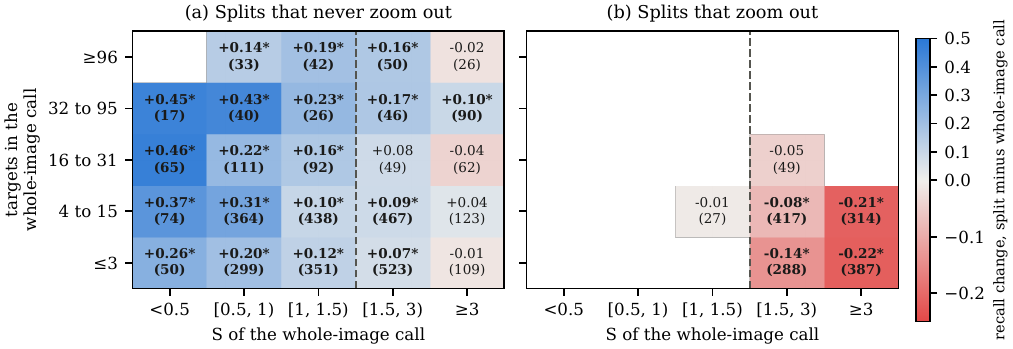}
\caption{Where splitting helps and where it hurts. Mean recall change, split minus whole, over every pair of a whole-image call and a split in our data, binned by the whole-image call's $S$ and number of targets; pairs are grouped by geometry only (does any view zoom out?). Drawings and floor plans use the tiles of the 44 px recipe and its fixed-grid counterpart, natural images same-scale $2 \times 2$ tiles. Recall: matched recall over all targets of the image, pooled over the eight models (Gemini 3.8 at three budgets) and all 40 LegendSpot drawings, the six prompt-writing sheets included. The recipe's tiles do not overlap, so panel (a) checks the no-zoom-out part of the condition. Dashed line: $S = 1.5$. Pairs per cell in parentheses; blank: fewer than 10 pairs; ${}^\ast$: 95\% bootstrap interval excludes 0 (images resampled; mosaics by source drawing). Details in \supp{H}.}
\label{fig:map}
\end{figure*}

Both predictions show in all our data (\cref{fig:map}). \emph{Splits that never zoom out} gain: every cell with $S < 1.5$ gains significantly, and no cell is significantly below zero. The gain grows as $S$ falls, \CmapLowSMin{} to \CmapLowSMax{} for $S < 0.5$, and, once $S$ is ample, as targets multiply: for $S$ between 1.5 and 3 it rises from \CmapMidSMin{} to \CmapMidSMax. This band comes mainly from Gemini 3.8, Gemma, Qwen and Claude (\CmapBandMin{} to \CmapBandMax); the GPT models are near 0. \emph{Splits that zoom out} lose in every cell, by \CmapZoomMin{} to \CmapZoomMax, significantly in \CmapZoomSig{} of \CmapZoomCells.

\paragraph{Controlled test.} On 12 controlled $2{,}048^2$ images, 8 models and 5 relations, every pair of calls meets the definitions exactly. \XasSup{} of the \XasN{} orderings (pairs of calls whose detections the assumptions order) hold significantly and none is violated. At fixed $S$, safe tiles, with less content per call, raise recall by \XasThmMin{} to \XasThmMax, significantly for \XasThmSup{} of \XasModels{} models (\supp{F}). Of the \XtcN{} pairs of calls that the assumptions order in our other experiments, \XtcSup{} hold significantly and \XtcViol{} are violated (below).

\paragraph{Robustness and F1.} When the assumptions hold only approximately, the split's recall is at least the whole-image call's minus the measured (R) and (M) defects, the uncovered share and any loss from merging duplicates; no term grows with the number of views, and some model attains the bound. A split keeps F1 at least that of the whole-image call if and only if its added false points are at most $\rho_{\mathrm{w}} = (n + \mathrm{FP}_{\mathrm{w}})/\mathrm{TP}_{\mathrm{w}} \ge 1/\mathrm{Rec}_{\mathrm{w}}$ times its added hits, where $\mathrm{TP}_{\mathrm{w}}$, $\mathrm{FP}_{\mathrm{w}}$ and $\mathrm{Rec}_{\mathrm{w}}$ are the whole-image call's true positives, unmatched points and recall (\supp{A}).

\subsection{Where protection ends}

\paragraph{Boundary one: the assumptions fail.} The guarantee covers only input paths on which (R) and (M) hold: given more pixels of the same content, a model may process the input differently. All \XtcViol{} violations we found are of this kind, OpenAI models sent larger images than our standard input (one example in \cref{sec:intro}). Before trusting a new path, such as a high-resolution mode or uploading a large image at native size, a caller should test (R) with a few pairs of calls on the same region at different token counts, or split the image by the rule so that every view stays within the budget.

\paragraph{Boundary two: the condition fails.} A location that no qualified view covers is unprotected, and a view that zooms out is not qualified (\cref{thm:exact}b): the guarantee is gone. Tiling that first rescales the image to a pixel target zooms out whenever the whole image already gives $S$ above the recipe's, which depends on the image's size and is visible only by computing $S$ (\cref{fig:intro}c, \cref{fig:map}b). Choosing $S = \max(S^\ast, \Sw)$ excludes this by construction.

\section{The rule in practice, and who splits}
\label{sec:practice}

\begin{table*}[t]
\centering
\small
\begin{adjustbox}{max width=\linewidth}\begin{tabular}{llcccccc}
\toprule
Data & Model & Recall & Precision & F1 & Tokens & Recall & Cost \\
 & & whole $\to$ rule & whole $\to$ rule & whole $\to$ rule & rule/whole & rule $-$ 44 px & $/C_{\min}$ \\
\midrule
LegendSpot & GPT-6.1 & 0.97 $\to$ 1.00 & 0.93 $\to$ 0.97 & 0.95 $\to$ 0.98 & 8.6 & +0.00 & 3.5 \\
LegendSpot & Claude Opus 5.5 & 0.83 $\to$ 0.92 & 0.82 $\to$ 0.92$^{*}$ & 0.81 $\to$ 0.90 & 5.2 & $-$0.05 & 2.6 \\
LegendSpot & Gemini 3.8 & 0.62 $\to$ 0.90$^{*}$ & 0.72 $\to$ 0.88$^{*}$ & 0.65 $\to$ 0.87$^{*}$ & 7.6 & $-$0.04 & 1.6 \\
LegendSpot & Gemini 4 & 0.85 $\to$ 0.98$^{*}$ & 0.89 $\to$ 0.93 & 0.85 $\to$ 0.94$^{*}$ & 7.5 & +0.00 & 1.6 \\
LegendSpot & Qwen3-VL & 0.37 $\to$ 0.56$^{*}$ & 0.75 $\to$ 0.93$^{*}$ & 0.37 $\to$ 0.56$^{*}$ & 6.1 & $-$0.07 & 2.0 \\
\midrule
Floor plans & GPT-6.1 & 0.63 $\to$ 0.61 & 0.80 $\to$ 0.84 & 0.63 $\to$ 0.64 & 4.4 & n/a & 3.3 \\
Floor plans & Claude Opus 5.5 & 0.51 $\to$ 0.56 & 0.81 $\to$ 0.79 & 0.57 $\to$ 0.59 & 3.1 & n/a & 2.0 \\
Floor plans & Gemini 3.8 & 0.25 $\to$ 0.43$^{*}$ & 0.58 $\to$ 0.78$^{*}$ & 0.32 $\to$ 0.50$^{*}$ & 6.8 & $-$0.13$^{*}$ & 1.6 \\
Floor plans & Gemini 4 & 0.50 $\to$ 0.57 & 0.70 $\to$ 0.83$^{*}$ & 0.56 $\to$ 0.63$^{*}$ & 6.7 & $-$0.01 & 1.6 \\
Floor plans & Qwen3-VL & 0.05 $\to$ 0.12$^{*}$ & 0.28 $\to$ 0.46$^{*}$ & 0.06 $\to$ 0.15$^{*}$ & 3.5 & $-$0.11$^{*}$ & 2.0 \\
\bottomrule
\end{tabular}
\end{adjustbox}
\caption{The rule ($S^\ast = 1.5$, without step 3 of \cref{alg:rule}) against the whole-image call on 34 LegendSpot drawings and 60 floor-plan mosaics. Recall, precision and F1: means per image under one-to-one matching; duplicates in overlaps are not merged and count against precision, which omits images where either call returned no point (Qwen3-VL: 17 and 47; others at most 5). Tokens: billed input tokens (prompts and reference image included), rule over whole-image call. Recall $-$ 44 px: difference to the 44 px recipe (n/a: not run). Cost $/\Cmin$: the rule's image tokens over $\Cmin$, median. ${}^\ast$: 95\% paired bootstrap interval excludes 0. Gemini 4's whole-image call asks for 0 to 1000 coordinates (\supp{B}).}
\label{tab:main}
\end{table*}

The rule raises recall most where the whole image gives targets the fewest tokens, at a cost known in advance, and zoom agents that succeed do so by splitting the image themselves. We compare it with the whole-image call and the 44 px recipe for GPT-6.1, Claude Opus 5.5, Gemini 3.8, Gemini 4 and Qwen3-VL, \CtCells{} model and data settings in all.

\paragraph{Recall and cost (\cref{tab:main}).} Recall rose significantly in \CtSigRec{} settings: Gemini 3.8 on LegendSpot from \CtLsGemThreeEightRecW{} to \CtLsGemThreeEightRecR{} (\CtLsGemThreeEightRel\% relative), Gemini 4 on LegendSpot from \CtLsGemFourRecW{} to \CtLsGemFourRecR{} (\CtLsGemFourRel\%), Qwen3-VL on LegendSpot from \CtLsQwenThreevlRecW{} to \CtLsQwenThreevlRecR{} (\CtLsQwenThreevlRel\%), Gemini 3.8 on floor-plan mosaics from \CtFpGemThreeEightRecW{} to \CtFpGemThreeEightRecR{} (\CtFpGemThreeEightRel\%), and Qwen3-VL on floor-plan mosaics from \CtFpQwenThreevlRecW{} to \CtFpQwenThreevlRecR{} (\CtFpQwenThreevlRel\%); relative gains use unrounded recalls. Precision rose significantly in \CtSigPrec{} of \CtCells{} settings, and so did F1 (\CtSigFOne{} settings). The price is \CtTokMin{} to \CtTokMax{} times the whole-image call's tokens, \CtTokDMin{} to \CtTokDMax{} times the 44 px recipe's, and \CtCminMin{} to \CtCminMax{} times $\Cmin$. Across the \CgainN{} settings, the gain falls as the whole-image call's median $S$ rises (Spearman $\rho = \CgainRhoSw$, $p = \CgainPSw$). The cost was fixed before the first call: in all \CcostPairs{} pairs of a drawing and a model, the rule sent exactly the planned views, and every view of a drawing was billed the same number of input tokens. The tokens come from raising $S$ to 1.5, and the gap to $\Cmin$ from images only a few views wide (\cref{cor:safe}); \CtOneViewGpt{} of GPT-6.1's 60 floor-plan mosaics need a single view.

\paragraph{Small gains point to step 3.} On floor plans the 44 px recipe recovers more targets than the rule for Gemini 3.8 and Qwen3-VL (\CtFpGemThreeEightdRecD{} and \CtFpQwenThreevldRecD). There the recipe cuts much smaller views, and the rule uses as little as \CtTokDMin{} of its tokens: for Qwen3-VL this mainly lowers $L$, and for the fixed-grid Gemini 3.8 a smaller view also raises $S$. On the same grid, Gemini 4, which needs much less $S$ (\cref{sec:why}), shows only \CtFpGemFourdRecD, not significant. The rule protects within an explicit factor of the lower bound; where $L$ binds, step 3 splits its views further at the same scale, at the cost given in \supp{I}.

\paragraph{Methods without a VLM, and natural images.} Template matching reaches \FTMdevAll{} F1 on the 34 drawings, and the best tiles (the 44 px recipe and its fixed-grid counterpart) score above it (Claude Sonnet 5.5 \FDsonnetFiveFiveAll, Gemini 3.8 \FDgemThreeEightAll). On the 10 further drawings whose reference symbols come from a symbol library, template matching falls to \XrdFTM{} and the VLMs are at least 0.30 higher. On FSC-147, same-scale $2 \times 2$ tiles raise small-object F1 for GPT-5.6 by \UfscODgptFiveSixsmallTU{} and for Claude Sonnet 5.5 by \UfscODsonnetFiveFivesmallTU{}: lowering $L$ at the same scale pays on natural images too.

\paragraph{Who splits: zoom agents.} With matched reasoning, GPT-5.6's zoom agent reaches \XnFAGgptFiveSixSheets{} F1 on full sheets against \XnFDRgptFiveSixSheets{} for its tiles (\XnHOnegptFiveSix, Holm $p = \XnXNineHOnegptFiveSixHolm$); Claude Sonnet 5.5's agent reaches \XnFAGsonnetFiveFiveSheets{} against \XnFDRsonnetFiveFiveSheets, not significant. In the ablation that changes one step at a time (\cref{sec:agents}), $S$ drives GPT-5.6's gain, and native pixels add \XnHThreesonnetFiveFiveAGOVZ{} for Claude. GPT-5.6 used \XagZoomsgptFiveSix{} zooms per drawing on average and on several sheets first laid a regular grid of crops over the sheet: the agent splits the image itself, and \cref{prop:frontier} bounds its cost whenever its model's recognition is graded. Over all drawings, the Gemini 3.8 and Qwen3-VL agents fail (\XagFgemThreeEightAll{} and \XqAgF): they count as accurately as their tiles, but only \CagLocGem\% and \CagLocQwen\% of their points lie within $\tau$ of an instance, against \CtileLocGem\% and \CtileLocQwen\% for their tiles, whose points the caller maps back (\supp{J}).

\section{Where the next models should improve}
\label{sec:future}

The interface fixes, for every model, the $S$ one call can reach and the minimum cost of seeing an image at a given $S$. The model sets the rest: the $S$ it needs, how much one call can handle, whether it is monotone, and how it reports positions. Progress pays on these model parameters, and \cref{tab:limits} sorts every limit by the evidence behind it and by what the newest generation changed.

\begin{table*}[t]
\centering
\small
\begin{tabular}{p{2.1cm}p{5.8cm}p{2.6cm}p{3.8cm}}
\toprule
Limit & Statement & Support & Newest generation \\
\midrule
Information & Detail lost in the vendor's resize cannot be recovered; overlays and interpolated zoom add none (\cref{prop:dpi}); a native crop can add it. & Proposition; agent ablation (\cref{sec:agents}) & Unchanged; whether it binds depends on the recognizer. \\
Recognition & Recall rises gradually with $S$; the $S$ a model needs is a model parameter. & Controlled sweeps; dose response (\cref{sec:S-evidence}) & Falls: on the same token grid Gemini 4 needs less than half the $S$ of Gemini 3.8. \\
Minimum cost & Seeing an image at $S^\ast$ costs at least $(WH/m^2){S^\ast}^2$ tokens; every protected split and every search policy with graded recognition pays a bound of this form (\cref{prop:geometry,thm:exact,prop:frontier}). & Propositions; agent logs & Formula unchanged; its value falls with the $S$ a model needs, and a zoom agent pays it itself. \\
Content per call & Recall falls with the content of one call: more targets, more area, more tiles per image, more blocks around a fixed block. & Controlled sweeps; packed tiles; mosaics (\cref{sec:L-evidence}) & Grows: GPT-6.1 holds at \CgenGptSixOneMax{} targets per view at $S = 1.5$; drawing content still binds. \\
Monotonicity & (R) and (M) carry every guarantee of \cref{sec:safe}. & Controlled test; all paired calls (\cref{sec:assume}) & Hold on standard input for every model in the controlled test (Boundary one in \cref{sec:safe}). \\
Interface & Coordinate frame, scale and format differ by vendor, mode and even answer; agents add a crop-to-overview conversion. & Probes; agent trajectories (\cref{sec:coords,sec:agents}) & Not improving; must be checked per model and mode. \\
Verification & With per-item error rate $\varepsilon$, all $n$ items are right with probability at least $1 - n\varepsilon$, and exactly $(1-\varepsilon)^n$ for independent errors, so exact large counts need $\varepsilon \ll 1/n$; per-item locations make results checkable. & Arithmetic & Lower $\varepsilon$; locations still needed for review. \\
Localization & A model that knows only a target's token cell places it within $\tau$ with probability at most $\pi S^2/4$ for $S \le 1$ (\cref{prop:loc}). & One-dot probe & No longer binds: newer models place a dot within about 1\,px. \\
\bottomrule
\end{tabular}
\caption{Limits on finding many small objects with VLMs, the evidence behind each, and what the newest generation changed.}
\label{tab:limits}
\end{table*}

\paragraph{Lower the $S$ a target needs, not only the budget.} By \cref{prop:geometry}(a), the $S$ of a whole-image call grows only with the square root of the budget. The coverage bound of \cref{prop:geometry}(b) scales with the square of the $S$ that recognition requires, and so does the frontier of \cref{prop:frontier} for recognition curves of a fixed shape, $f(S) = \min\{1, (S/S_0)^\gamma\}$ with $\gamma \ge 2$ (\cref{sec:proofs-search}); halving that $S$ quarters them, for tiles and agents alike. The protection minimum $\Cmin$ of \cref{thm:exact}(c) adds the overlap factor $1/(1-\sigma)^2$, which also shrinks as $S$ falls, and the cost of a concrete split follows from \cref{cor:safe}, with its overlaps, whole views and per-image overhead (\cref{sec:proofs-safe}). The newest generation shows that training moves this parameter (\cref{sec:why}); recognition at low token density is the improvement with the largest effect on cost.

\paragraph{Raise what one call can handle.} Within one call, recall falls as content grows. On floor-plan mosaics, GPT-5.6's F1 on a fixed first block changes by \XfBKDAOnegptFiveSix{} when eight blocks are added around it (\cref{sec:firstblock}); on controlled images, (M) crops at equal $S$ raise its recall by \XasMGptFiveSix{} (\cref{tab:assume}). Each generation handles more well-separated targets (\cref{sec:why}). Enumerating many objects in one long answer is the weak point, and region-wise decoding, structured outputs per region and dense output heads are natural candidates, which the nested pairs of \cref{tab:assume} can test. \Cref{thm:capacity} gives the unit: each hit costs about $\log_2$ of the ratio between the area per target and the tolerance disk, in bits, so progress is measured by the bits about positions that one answer can carry.

\paragraph{Be monotone by design, and test it.} A model that satisfies (R) and (M) inherits \cref{thm:safe} and \cref{cor:safe}, so callers can split it safely without tuning per model; Boundary one of \cref{sec:safe} names the input paths on which current models violate them. Developers can train for consistency across crops and scales of the same image, so that a target found in the whole image is also found in its crop and at a higher scale, and test it with nested pairs as in \cref{tab:assume}, on twelve controlled images per model.

\paragraph{Make the interface explicit, and prefer native pixels.} Returning coordinates in the frame of the image sent, reporting the effective scale and token grid, and not resizing silently are interface decisions rather than research problems. They decide whether a caller can compute $S$ before a call and apply the guarantees. By \cref{prop:dpi}, an overlay drawn on the sent image and interpolated zoom add no Shannon information, while a crop cut from the native image can add information about the targets (\cref{prop:zoom}). Zoom tools should return native pixels and convert coordinates themselves, which takes the conversion that two of four agents fail at (\cref{sec:agents}) out of the model's hands.

\paragraph{What to measure for a new model.} The parameters that matter are few and cheap to measure, and none needs labelled images of the target task: the budget and token size from billed tokens on blank images (\cref{sec:geometry}); the coordinate frame, scale and format from a one-dot probe and from images above the budget (\cref{sec:coords}); the $S$ recognition needs and the load one view carries from controlled images (\cref{sec:S-curves,sec:thumb}); (R) from a few pairs of calls on the same region at different token counts (\cref{sec:safe}); and, for agents, the conversion from crop to overview coordinates.

\section{Conclusion: back to computer vision}
\label{sec:conclusion}

VLMs miss small objects when an object gets too few visual tokens or a call must search and report too much. Splitting is safe exactly when every location inside the whole-image call lies inside some view that does not zoom out, which views that never zoom out and overlap by one object guarantee, for every model monotone on its input path; it has a minimum token cost, and a simple rule pays it up to an explicit factor. Across \CscaleModels{} models, including the newest, the rule raised recall most where the whole image gives targets the fewest tokens, at a cost known before the call.

Tiling, a classical tool of computer vision, does not expire with VLMs: what one call cannot afford to see at the needed $S$, or cannot handle, must be split. What changes is who splits: the caller, a zoom tool, or the architecture. Future models can need less $S$, handle more per call, stay monotone across crops and scales, and declare their coordinate frame and scale.

\bibliography{main}
\bibliographystyle{acl_natbib}

\appendix
\section{Proofs}
\label{sec:proofs}

This appendix states every result of the main text in full and proves it. Statements in the main text are short forms of the ones below; where they differ, the full form here is the one proved.

\subsection{Setting and definitions}
\label{sec:defs}

\paragraph{Images, targets and tokens.} An image $X = [0,W) \times [0,H)$ has $W \times H$ native pixels and contains $n$ instances of a target whose box is $w \times h$ pixels; $m = \sqrt{wh}$ is the target scale and $\tau = m/2$ the matching tolerance. A \emph{call} sends the model one image derived from $X$. Its \emph{effective scale} $s$ is the number of pixels the model receives per native pixel, after every resampling by the caller and by the vendor. The model divides the image it receives into tokens of $t \times t$ pixels, with $t$ fixed for each interface, and accepts at most $N$ tokens per image. Resolution-scaled interfaces (GPT, Claude, Qwen) resize an image when it exceeds $N$ tokens or a long-side cap; fixed-grid interfaces (Gemini, Gemma) pad an image to a square and resize it so that its side spans $g$ tokens, whatever its pixel size. A target therefore spans $S = ms/t$ tokens per side (\cref{eq:S}).

\paragraph{Calls, views and cores.} A call $c$ consists of a source image, a rectangular region $V_c$ of it (the \emph{view}), an effective scale $s_c$, a kept core $K_c \subseteq V_c$ (a rectangle, or a finite union of rectangles), a blank padding $P_c$, and a protocol $\pi_c$: the prompt and any marking of the core. Views and cores are half-open like $X$. The source is a fixed image derived from $X$ without reference to the targets: the native image, a lower-resolution copy such as an overview, or a composite such as several regions packed into one image; it is blank outside $X$. The input of $c$ is the region $V_c$ of the source resampled to scale $s_c$ and surrounded by $P_c$; we treat the order of cropping and resampling as immaterial, which ignores differences within one resampling kernel of the region's border. A \emph{model} is any map from calls, each together with the image it is derived from, to probability laws of returned point sets: a family $P_X(\cdot \mid c)$ of response laws, which may depend on the image $X$ and on the call. Points outside $K_c$ are discarded and the rest are mapped to native coordinates.

\paragraph{Detection and interior targets.} For a target centred at $x$, let $H_c(x) = 1$ if $c$ returns a kept point within $\tau$ of $x$ and $0$ otherwise, and let $p_c(x)$ be the probability that $H_c(x) = 1$. A set of calls \emph{detects} $x$ if one of its calls returns such a point, and its \emph{hit recall} is the fraction of targets detected. When target centres are more than $2\tau$ apart, each returned point lies within $\tau$ of at most one target, and hit recall equals recall under a maximum one-to-one matching. A target is \emph{interior} to $c$ if its box lies inside $V_c$ and its centre lies at least $\tau$ inside $K_c$. Because every target has the same box, this depends only on the centre, and $J_c$ denotes the set of centres interior to $c$. We write $B_0 = [-w/2, w/2] \times [-h/2, h/2]$, so the box of a target centred at $x$ is $x + B_0$, and $K \oplus B_0$ is the core enlarged by half a target box on each side; $X \ominus B_0$ is the set of centres whose box lies in $X$.

\paragraph{Averages over images.} Let $\mathcal{D}$ be any distribution of images: the images of an application, a benchmark, or a single image. Calls, and the regions on which we compare them, are specified by rules that may depend on the image's size but not on its targets. Thus $V_c$, $K_c$, $J_c$, and a \emph{region} $U$ (a measurable set of centres) are functions of the image's size; unions, intersections and differences of regions are regions. A split may use a number of calls that depends on the image's size. Let $Y(X)$ be the set of target centres of $X$ and $n(X) = |Y(X)|$. A \emph{weight} is a function $a(X) \ge 0$ of the image alone, with $\mathbb{E}[a(X)\,n(X)] < \infty$. For a region $U$,
\[
\nu_c(U) = \mathbb{E}\Big[a(X) \sum_{x \in Y(X) \cap U \cap J_c} H_c(x)\Big],
\]
the expectation taken over $\mathcal{D}$ and the model's randomness. With $a = 1$, $\nu_c(U)$ is the expected number of targets with centres in $U \cap J_c$ that $c$ detects; with $a = 1/n(X)$ (and $a = 0$ when $n(X) = 0$), $\nu_c(J_c)$ is the mean per-image fraction of targets that are interior to $c$ and detected, which is the mean recall per image when every target is interior. $\nu_c$ is additive over disjoint regions, and $\mu(U) = \mathbb{E}[a(X)\,|Y(X) \cap U|]$. All calls of one comparison share $Y(X)$ and $a$. The proofs below use only additivity, the assumptions for the same $a$, and constructions that hold target by target, so every result holds for every weight; we write $\nu$ without the weight.

\paragraph{Three relations between calls.} Call $c'$ is a \emph{crop} of $c$ if it shows a sub-region $V_{c'} \subseteq V_c$ of a source that agrees with the source of $c$ on $V_{c'}$, at scale $s_{c'} \ge s_c$, with the protocol $\pi_{\mathrm{crop}}$ with which we show crops; padding may differ. Any part of a view outside $X$ is blank and counts as padding, and $V_{c'}$ and $K_{c'}$ are then taken intersected with $X$. For packed tiles, agreement is up to the translation that places a tile in the composite. Call $c'$ \emph{refines} $c$ if both show the same region of the same source with the same core and protocol and $s_{c'} \ge s_c$; padding may differ. Call $c$ is a \emph{garbling} of $c'$ if both show the same region with the same core, padding, protocol and scale, and the source of $c$ is a fixed degradation of the source of $c'$ that does not depend on the targets, such as a lower-resolution copy.

\paragraph{Interface assumption.} The vendor does not resize an image again if its pixel size, after blank padding to a multiple of $t$, equals the size to which the vendor resized an earlier image (for fixed grids, the resize to $g$ tokens is the only one applied). A call can then be shown at any effective scale up to the budget by choosing its padding.

\begin{restatedef}{Assumption}{ass:mono}{Monotonicity, full form}
For a distribution $\mathcal{D}$ and a weight $a$, for all calls $c, c'$ and every region $U \subseteq J_c \cap J_{c'}$:
(R) if $c'$ refines $c$, then $\nu_{c'}(U) \ge \nu_c(U)$;
(M) if $c'$ is a crop of $c$ at the same effective scale, then $\nu_{c'}(U) \ge \nu_c(U)$.
\end{restatedef}

\begin{assumption}[Information]
\label{ass:info}
(I) If $c$ is a garbling of $c'$, then $\nu_{c'}(U) \ge \nu_c(U)$ for every region $U \subseteq J_c \cap J_{c'}$.
\end{assumption}

$\mathcal{M}_{R,M}$ is the class of models that satisfy (R) and (M) for the given $\mathcal{D}$ and $a$; a model that satisfies them target by target lies in $\mathcal{M}_{R,M}$ for every $\mathcal{D}$ and $a$, and the necessity results below exhibit such models. (I) is needed only for zoom tools (\cref{prop:zoom}). When $\mathcal{D}$ is a single image whose targets have distinct centres, taking $U$ to be a single centre, and conversely summing over the finitely many targets in $U$, shows that \cref{ass:mono} is equivalent to $p_{c'}(x) \ge p_c(x)$ for every target interior to both calls. For a general $\mathcal{D}$ it is weaker: it asks for the inequality only on average over the images and over the targets in each region, which is the quantity an experiment estimates, and it allows individual targets and images to go the other way. Every result below holds for whichever $\mathcal{D}$ the assumption is made for. (I) is the behavioural counterpart of \cref{prop:dpi}(b): when the input of $c$ is a function of the input of $c'$, the Bayes risk of $c'$ is at most that of $c$ for every fixed loss; for a VLM it is an assumption.

\subsection{The interface: Propositions~\main{prop:geometry} and~\main{prop:dpi}}
\label{sec:proofs-interface}

\begin{restate}{Proposition}{prop:geometry}{Interface ceiling and price of seeing, full form}
(a) A resolution-scaled call that shows the whole image gives $S \le m\sqrt{N/(WH)}$; a fixed-grid call gives $S = mg/\max(W,H)$ at every caller scale.
(b) Suppose every view's kept core has area at most $\kappa^2$ times the view's, $0 < \kappa \le 1$. Covering the image with kept cores of views shown at $S \ge S^\ast$ then requires at least $WH{S^\ast}^2/(\kappa^2 m^2)$ tokens, hence at least $WH{S^\ast}^2/(\kappa^2 m^2 N)$ calls (with $N = g^2$ for fixed grids).
(c) A fixed-grid view whose longer side spans $\ell$ native pixels gives $S = mg/\ell$, and $k$ views cost $kg^2$ tokens. For a resolution-scaled model, let views of native sizes $a_i \times b_i$, $i = 1, \dots, k$, each within the budget, be shown at a common scale $s$, with union $R$. Excluding a fixed overhead of $\omega$ tokens per image, they cost at most
\[
\Big(\frac{s}{t}\Big)^2\big(|R| + O\big) + \frac{s}{t}\sum_i (a_i + b_i) + k
\]
tokens, where $O = \sum_i a_ib_i - |R| \ge 0$ is the area counted more than once. For a grid of $k_x \times k_y$ views over an $A_x \times A_y$ region, overlapping by $e_x$ and $e_y$, $O \le (k_x - 1)e_xA_y + (k_y - 1)e_yA_x + (k - 1)e_xe_y$ and $\sum_i (a_i + b_i) \le k_yA_x + k_xA_y + k(e_x + e_y)$, so every term beyond the tokens of the region itself grows at most linearly in $k = k_xk_y$.
\end{restate}

\begin{proof}
A target of scale $m$ spans $ms$ sent pixels, that is, $S = ms/t$ tokens. Equivalently, a native region of area $A$ shown at $S$ tokens per target side covers $AS^2/m^2$ tokens, because one native pixel covers $(s/t)^2 = S^2/m^2$ tokens; rounding patch counts up only increases the count.
(a) A resolution-scaled whole-image call satisfies $\lceil Ws/t\rceil\lceil Hs/t\rceil \le N$, hence $WH(s/t)^2 \le N$, $s/t \le \sqrt{N/(WH)}$, and $S \le m\sqrt{N/(WH)}$ for any aspect ratio. A fixed grid resizes the padded square so that its side spans $g$ tokens, so $s = gt/\max(W,H)$ whatever the caller sent, and $S = mg/\max(W,H)$.
(b) A view whose kept core has area $A_{\mathrm{core}}$ has area at least $A_{\mathrm{core}}/\kappa^2$ and costs at least $A_{\mathrm{core}}{S^\ast}^2/(\kappa^2m^2)$ tokens at $S \ge S^\ast$; for a fixed grid the view, padded to a square of side $a \ge \sqrt{|V|}$, costs $g^2 = a^2 S_V^2/m^2 \ge |V| S_V^2/m^2$ tokens by (c). The cores must cover area $WH$ and overlaps only add cost, so the total is at least $WH{S^\ast}^2/(\kappa^2 m^2)$ tokens; at most $N$ fit in one call.
(c) A $g \times g$ grid over a view of side $\ell$ native pixels has $s/t = g/\ell$, so $S = mg/\ell$, and every view costs $g^2$. For a resolution-scaled model, a view of $a_i \times b_i$ native pixels at scale $s$ within the budget costs $\lceil a_is/t\rceil\lceil b_is/t\rceil \le (a_is/t + 1)(b_is/t + 1) = (s/t)^2a_ib_i + (s/t)(a_i + b_i) + 1$ tokens. Summing over the views and writing $\sum_i a_ib_i = |R| + O$ gives the bound. For the grid, the views in one row have widths summing to $A_x + (k_x - 1)e_x$ and the rows have heights summing to $A_y + (k_y - 1)e_y$, so $\sum_i a_ib_i = (A_x + (k_x-1)e_x)(A_y + (k_y-1)e_y)$, and $O$ is this minus $A_xA_y$, which equals $(k_x-1)e_xA_y + (k_y-1)e_yA_x$ plus the product $(k_x-1)(k_y-1)e_xe_y \le (k-1)e_xe_y$ of the corner overlaps. The sums of widths and heights over all views are $k_y(A_x + (k_x-1)e_x)$ and $k_x(A_y + (k_y-1)e_y)$, which gives the second bound.
\end{proof}

Splitting a view at a fixed scale therefore changes the tokens of a resolution-scaled model only through the overlaps, the rounding to whole patches and the per-image overhead, each at most linear in the number of views; a fixed-grid model pays $g^2$ per view but gains $S$ as views shrink. Raising $S$ by a factor $a$ multiplies every term of (b) by $a^2$.

\begin{restate}{Proposition}{prop:dpi}{Information, full form}
Let $Y$ be the target set and $R_s$ the interface's deterministic preprocessing at scale $s$. (a) For any transform $\gamma$ of the native image that is deterministic or randomized independently of $Y$ given $X$, $I(Y; R_s(\gamma(X))) \le I(Y; X)$. (b) For any transform $\eta$ of the model input that is deterministic or uses randomness independent of $(X, Y)$, for example upsampling or an overlay drawn on the sent image, $I(Y; \eta(R_s(X))) \le I(Y; R_s(X))$; moreover $\eta(R_s(X))$ is a garbling of $R_s(X)$, so for every loss its Bayes risk is at least that of $R_s(X)$. (c) If $R_s$ is linear and $\Lambda$ is a fixed additive overlay, $I(Y; R_s(X + \Lambda)) = I(Y; R_s(X))$.
\end{restate}

\begin{proof}
(a) Because $\gamma$ is deterministic or randomized independently of $Y$ given $X$, and $R_s$ is deterministic, $Y \to X \to \gamma(X) \to R_s(\gamma(X))$ is a Markov chain, and the data-processing inequality~\cite{cover2006elements} gives the claim. (b) Because $\eta$ acts on $R_s(X)$ alone, with any randomness independent of $(X,Y)$, $Y \to R_s(X) \to \eta(R_s(X))$ is a Markov chain. For the Bayes risk, any decision rule $d$ applied to $\eta(R_s(X))$ defines the randomized rule $z \mapsto d(\eta(z))$ applied to $R_s(X)$, with the same risk for every loss, so the best rule given $R_s(X)$ is at least as good; this is the elementary direction of Blackwell's comparison of experiments~\cite{blackwell1953equivalent}. (c) If $R_s$ is linear, $R_s(X+\Lambda) = R_s(X) + R_s(\Lambda)$. Because $\Lambda$ is fixed, $R_s(\Lambda)$ is a constant and $z \mapsto z + R_s(\Lambda)$ is a bijection, which preserves mutual information.
\end{proof}

Real interfaces clip, quantize and compress, so (c) is exact only up to those operations, and an overlay that occludes the image, $X \odot (1-M) + \Lambda$, is covered by (a) and (b) but not by (c). If the vendor does not resample or recompress the sent image, every overlay whose content depends on the image only through the sent image is of type (b): grids, numbered lines, labelled dots and interpolated zoom can change only how a model uses its input. An overlay computed from the native image, such as marks at candidates found on native pixels, and a crop cut from the native image and sent at a larger scale are bounded only by (a); the crop can carry detail that $R_s(X)$ discards. For a computationally bounded predictor, \emph{usable} information can still grow under deterministic processing~\cite{xu2020usable}: enlarging an input raises $S$ without adding Shannon information, and can therefore help when the binding constraint is the model's token geometry, but not when the pixels do not resolve the target.

\subsection{Safe splitting: Theorems~\main{thm:safe} and~\main{thm:exact}, Proposition~\main{prop:tight}, Corollary~\main{cor:safe}}
\label{sec:proofs-safe}

Throughout, $\cw$ is the whole-image call: it shows the whole native image, with the whole image as its core, at effective scale $s_{\mathrm{w}}$, so $\Sw = ms_{\mathrm{w}}/t$; $\Jw$ is the image shrunk by $e_x/2$ and $e_y/2$, where $e_x = \max(w, 2\tau)$ and $e_y = \max(h, 2\tau)$, of area $|\Jw| = (W - e_x)(H - e_y)$ when $W > e_x$ and $H > e_y$.

\begin{lemma}[Crop dominance]
\label{lem:crop}
Under (R) and (M), if $c'$ is a crop of $c$, then $\nu_{c'}(U) \ge \nu_c(U)$ for every region $U \subseteq J_c \cap J_{c'}$.
\end{lemma}

\begin{proof}
Let $\tilde c$ be the call that shows the same source region as $c'$ at the scale $s_c$ of $c$, with core $K_{c'}$ and protocol $\pi_{\mathrm{crop}}$, and with a padding chosen so that the vendor applies no further resizing, which the interface assumption allows: for a fixed grid, a blank canvas of the size of the padded image of $c$, and for an interface that enlarges small images, a canvas large enough to avoid it. Then $\tilde c$ is a crop of $c$ at the same effective scale, and $J_{\tilde c} = J_{c'}$ because $\tilde c$ and $c'$ have the same region and core, so $U \subseteq J_c \cap J_{\tilde c}$ and (M) gives $\nu_{\tilde c}(U) \ge \nu_c(U)$. The call $c'$ shows the same region of the same source with the same core and protocol at scale $s_{c'} \ge s_c$, with possibly different padding, so it refines $\tilde c$, and (R) gives $\nu_{c'}(U) \ge \nu_{\tilde c}(U)$. The rules that define $\tilde c$ do not depend on the targets, as \cref{ass:mono} requires. The intermediate call need not be one that anyone makes: the assumptions are statements about the model's response to any admissible call, and the argument only chains them.
\end{proof}

\begin{restate}{Theorem}{thm:safe}{Safe splitting, full form}
Let $c_1, \dots, c_k$ be crops of $\cw$, so that $S_i \ge \Sw$ for every $i$. Under (R) and (M), for every region $U \subseteq \Jw \cap \bigcup_i J_{c_i}$, the expected number of targets with centres in $U$ that $c_1, \dots, c_k$ detect is at least $\nu_{\cw}(U)$. If the cores shrunk by $\tau$ cover the image shrunk by $\tau$, and every view $V_i$ contains $(K_i \oplus B_0) \cap X$, then $\Jw \subseteq \bigcup_i J_{c_i}$, so the split detects in expectation at least as many of the targets interior to $\cw$ as $\cw$ does.
\end{restate}

\begin{proof}
Set $J_{c_i} = \emptyset$ for indices beyond the number of calls the rule uses on an image, so that the argument below holds image by image; expected counts are then exchanged with the countable sum by monotone convergence. Let $U \subseteq \Jw \cap \bigcup_i J_{c_i}$ and split it into the disjoint sets $U_i = (U \cap J_{c_i}) \setminus \bigcup_{j<i} J_{c_j}$. A target with centre in $U_i$ is detected by the split whenever $c_i$ detects it, and possibly also through another call, so the expected number of targets in $U$ that the split detects is at least $\sum_i \nu_{c_i}(U_i)$. Each $c_i$ is a crop of $\cw$ and $U_i \subseteq \Jw \cap J_{c_i}$, so \cref{lem:crop} gives $\nu_{c_i}(U_i) \ge \nu_{\cw}(U_i)$, and $\sum_i \nu_{\cw}(U_i) = \nu_{\cw}(U)$ because the $U_i$ partition $U$. For the second claim, let $x$ be interior to $\cw$, so its centre lies in the image shrunk by $\tau$ and its box inside the image. Because the cores shrunk by $\tau$ cover the image shrunk by $\tau$, $x$ lies at least $\tau$ inside some core $K_i$. The box of $x$ is $x + B_0$ with $x \in K_i$, and it lies inside the image, hence inside $(K_i \oplus B_0) \cap X$, which $V_i$ contains. So $x \in J_{c_i}$. Hence $\Jw \subseteq \bigcup_i J_{c_i}$, and the first claim with $U = \Jw$ gives the second.
\end{proof}

\begin{restate}{Proposition}{prop:tight}{Each part is necessary, full form}
Suppose target centres are more than $2\tau$ apart. (a) For any split and any target whose centre lies in the kept core of no view with $S_i \ge \Sw$, there is a model satisfying (R), (M) and (I) under which $\cw$ detects every target interior to it with probability one and the split never detects that target. (b) There are models satisfying (M) and (I) but not (R) under which a split with $S_i > \Sw$ for all $i$ detects nothing, and models satisfying (R) and (I) but not (M) under which any split into views that are not the whole image detects nothing, while in both cases $\cw$ detects every target interior to it.
\end{restate}

\begin{proof}
All models below return, when they return anything, exactly the centres of the targets whose boxes lie in the view, so a returned point lies within $\tau$ of a target only if it is that target's centre. (a) Let the model return these centres if $S_c \ge \Sw$ and nothing otherwise. A target interior to a call is then detected with probability one if $S_c \ge \Sw$ and zero otherwise. (M) and (I) compare calls with equal $S_c$ and the same interior target, so they hold with equality; (R) holds because raising $s_c$ can only move $S_c$ above $\Sw$. The whole-image call has $S_c = \Sw$ and detects every target interior to it. A target whose centre lies in no kept core of a view with $S_i \ge \Sw$ is never detected: views with $S_i < \Sw$ return nothing, and any other view that returns its centre discards it, because the centre is outside that view's core. (b) Let the model return the centres if $S_c \le \Sw$ and nothing otherwise. (M) and (I) hold with equality as before, (R) fails, $\cw$ detects every interior target, and a split with $S_i > \Sw$ detects nothing. Let instead the model return the centres only when its view is the whole image. (R) and (I) compare calls with the same view and hold with equality, (M) fails, and a split into views that are not the whole image detects nothing. Each model satisfies the stated assumptions target by target, so it satisfies them for every $\mathcal{D}$.
\end{proof}

The models in this proof satisfy the assumptions for every target, hence for every $\mathcal{D}$, so the condition is necessary whichever form of \cref{ass:mono} is assumed. A common recipe rescales every image so that the target spans a fixed number of pixels and then tiles. When the targets are already large in the whole-image call, this lowers $S$, \cref{thm:safe} gives no guarantee, and part (a) shows that none is possible from the assumptions.

With $\lambda = S^2/m^2$ the tokens per native pixel at $S$ and $\sigma = \sqrt{\lambda e_xe_y/N}$, the fraction of a view's side that one margin takes at the budget $N$:

\begin{restate}{Theorem}{thm:exact}{Exact condition and minimum cost, full form}
Let $\mathrm{d}$ be any finite set of calls. (a) If every centre in $\Jw$ is interior to some call in $\mathrm{d}$ that is a crop or a refinement of $\cw$, then under (R) and (M) $\mathrm{d}$ detects in expectation at least as many of the targets interior to $\cw$ as $\cw$ does, for every $\mathcal{D}$. (b) If target centres are more than $2\tau$ apart and some $x \in \Jw$ is interior to no such call, there is a model satisfying (R) and (M) for every target under which $\cw$ detects every target interior to it and $\mathrm{d}$ never detects a target centred at $x$. (c) Suppose every centre in $\Jw$ is interior to a call of $\mathrm{d}$ that is a crop or a refinement of $\cw$ with $S_c \ge S$, where $S \ge \Sw$ and $\sigma < 1$. Then these calls use at least
\[
\Cmin = \frac{|\Jw|\,\lambda}{(1 - \sigma)^2}
\]
tokens, whatever the shapes, sizes, cores and overlaps of their views; for a fixed grid the same holds with $N = g^2$.
\end{restate}

\begin{proof}
Here $\mathrm{d}$ is a rule fixed before any call, depending on the image's size but not on targets or responses, so whether a call is a crop or a refinement of $\cw$ is fixed by the rule.
(a) Calls of $\mathrm{d}$ that are neither crops nor refinements of $\cw$ can only add detections. Order the others as $c_1, \dots, c_k$ and set $U_i = (\Jw \cap J_{c_i}) \setminus \bigcup_{j<i} J_{c_j}$, which partition $\Jw$. For a crop, \cref{lem:crop} gives $\nu_{c_i}(U_i) \ge \nu_{\cw}(U_i)$; for a refinement, $J_{c_i} = \Jw$ and (R) gives the same. Summing as in the proof of \cref{thm:safe} gives the claim.

(b) Let $\mathcal{C}$ be the set of calls $c$ with $s_c \ge s_{\mathrm{w}}$ whose source agrees with the native image on $V_c$ up to a single translation and whose protocol is $\pi_{\mathrm{crop}}$, together with the calls that refine $\cw$ (source the native image, $V_c = K_c$ the whole image, the protocol of $\cw$, $s_c \ge s_{\mathrm{w}}$). Let the model return, for $c \in \mathcal{C}$, the centres of the targets with centres in $J_c$, and nothing for $c \notin \mathcal{C}$. (R): if $c \in \mathcal{C}$ and $c'$ refines $c$, then $c'$ has the same source, region, core and protocol at a scale at least $s_c$, so $c' \in \mathcal{C}$ (refinement is transitive) and $J_{c'} = J_c$, and equality holds; if $c \notin \mathcal{C}$, $\nu_c = 0$. (M): if $c \in \mathcal{C}$ and $c'$ is a crop of $c$ at the same scale, the source of $c'$ agrees with that of $c$, hence with the native image, on $V_{c'} \subseteq V_c$ (this covers crops of refinements of $\cw$), its protocol is $\pi_{\mathrm{crop}}$ and its scale is $s_c \ge s_{\mathrm{w}}$, so $c' \in \mathcal{C}$ and both calls detect every target in $J_c \cap J_{c'}$; if $c \notin \mathcal{C}$, $\nu_c = 0$. Both hold target by target, hence for every $\mathcal{D}$. The call $\cw$ is in $\mathcal{C}$ and detects every target interior to it. A call in $\mathcal{C}$ with protocol $\pi_{\mathrm{crop}}$ shows a sub-region of the image from a source that agrees with the native image there, at a scale at least $s_{\mathrm{w}}$, so it is a crop of $\cw$; the other calls in $\mathcal{C}$ are refinements of $\cw$ by definition. Hence every call of $\mathrm{d}$ that returns anything returns only centres in its interior set, which excludes $x$, and centres of other targets, which are more than $2\tau$ from $x$; so no returned point lies within $\tau$ of $x$.

(c) Let $c$ be one of these calls, let $V \cap X$, its view intersected with the image, have size $a_x \times a_y$, and let $\lambda_c = S_c^2/m^2$. A centre in $J_c$ has its box $x + B_0$ inside $V$ and lies at least $\tau$ inside $K_c \subseteq V$, so it lies at least $e_x/2$ from the left and right sides of $V$ and $e_y/2$ from the top and bottom: $J_c \cap \Jw$ lies in a rectangle of area $(a_x - e_x)_+(a_y - e_y)_+$. If this is zero there is nothing to show. Otherwise, with $\bar e = \sqrt{e_xe_y}$, $u = e_x/a_x$ and $v = e_y/a_y$, $(1-u)(1-v) = 1 - u - v + uv \le 1 - 2\sqrt{uv} + uv = (1-\sqrt{uv})^2$ and $\sqrt{uv} = \bar e/\sqrt{a_xa_y}$, so
\[
|J_c \cap \Jw| \le a_x a_y\Big(1 - \frac{\bar e}{\sqrt{a_x a_y}}\Big)^2,
\]
where $a_xa_y > \bar e^2$. The call costs $T_c \ge \lambda_c a_xa_y$ tokens and at most $N$, so $a_xa_y \le N/\lambda_c$ and $a_xa_y \le T_c/\lambda_c$; hence $|J_c \cap \Jw| \le (T_c/\lambda_c)(1 - \bar e\sqrt{\lambda_c/N})^2$, that is, $T_c \ge |J_c \cap \Jw|\,\lambda_c/(1 - \bar e\sqrt{\lambda_c/N})^2$. The function $z \mapsto z/(1 - \bar e\sqrt{z/N})^2$ increases for $0 < z < N/\bar e^2$, and $\lambda_c \ge \lambda$, so $T_c \ge |J_c \cap \Jw|\,\lambda/(1-\sigma)^2$. Summing over the calls, whose sets $J_c \cap \Jw$ cover $\Jw$, gives the bound. For a fixed grid, let $A$ be the longer side of the padded image in native pixels, so that $S_c = mg/A$, $a_x, a_y \le A$, and the call costs $g^2$; its set has area at most $(A - e_x)(A - e_y) \le (A - \bar e)^2$, because $e_x + e_y \ge 2\bar e$, and $A \le mg/S$ gives $g^2 \ge |J_c \cap \Jw|\,g^2/(mg/S - \bar e)^2 = |J_c \cap \Jw|\,\lambda/(1 - \bar eS/(mg))^2$, the same bound with $\sqrt N = g$.
\end{proof}

Parts (a) and (b) say that protection is decided target by target: a location is protected exactly when some crop or refinement of the whole-image call has it in its interior, and the covering hypothesis of \cref{thm:safe} is one way to ensure this everywhere. Part (c) uses no assumption about the model; it bounds what any split that protects every location with calls at $S_c \ge S$ must spend. $\Cmin$ is a lower bound, which \cref{cor:safe} approaches; every such call protects area at most $(N/\lambda)(1-\sigma)^2$, so any such split also needs at least $\Cmin/N$ calls.

\begin{restate}{Corollary}{cor:safe}{A safe split near the minimum, full form}
For a protection point $S^\ast$, let $S = \max(S^\ast, \Sw)$, with $\sigma < 1$, $W > e_x$ and $H > e_y$. Tile the image, from a corner, with views of native size $a_x \times a_y$ shown at $S$ (for a fixed grid, squares of side $mg/S$), whose kept cores are the whole views, at pitches $p_x = a_x - e_x > 0$ and $p_y = a_y - e_y > 0$, each view within the vendor's limits so that it is not resized, and let $T$ be the tokens of one view. Under (R) and (M) this split detects in expectation at least as many of the targets interior to $\cw$ as the whole-image call, and it attains $S^\ast$. Excluding per-image overhead it uses at most $\Gamma\,\Cmin$ tokens, with
\[
\Gamma = \chi\Big(1+\frac{p_x}{W-e_x}\Big)\Big(1+\frac{p_y}{H-e_y}\Big),
\]
where $\chi = T(1-\sigma)^2/(\lambda\, p_x p_y) \ge 1$, and $\chi = 1$ for the views $a_x = e_x/\sigma$, $a_y = e_y/\sigma$ when their sides are whole numbers of patches.
\end{restate}

\begin{proof}
Each view shows a region of size $a_x \times a_y$ at $S$ within the budget and the vendor's size limits, so no resizing occurs, its effective scale is $s = St/m \ge s_{\mathrm{w}}$, and, intersected with the image, it is a crop of $\cw$. With the core equal to the view, a centre is interior exactly when its box lies in the view and it lies at least $\tau$ inside the view, that is, when it lies in the view shrunk by $e_x/2$ horizontally and $e_y/2$ vertically; this is a half-open rectangle of size $p_x \times p_y$. With the first view at a corner, the first such rectangle starts at $(e_x/2, e_y/2)$, the corner of $\Jw$, and $\lceil (W-e_x)/p_x\rceil\lceil (H-e_y)/p_y\rceil$ of them, laid at pitches $p_x$ and $p_y$, tile a region containing $\Jw$; views that extend beyond the image show blank padding there. Every centre in $\Jw$ is therefore interior to a crop of $\cw$, and \cref{thm:exact}(a) gives the recall claim. For the cost, with $\lceil z \rceil \le z + 1$,
\begin{align*}
&\frac{\lceil (W-e_x)/p_x\rceil\lceil (H-e_y)/p_y\rceil\,T}{\Cmin}\\
&\quad\le \Big(1+\frac{p_x}{W-e_x}\Big)\Big(1+\frac{p_y}{H-e_y}\Big)\frac{T(1-\sigma)^2}{\lambda\,p_xp_y}.
\end{align*}
The per-call bound in the proof of \cref{thm:exact}(c), applied to one view whose interior has area $p_xp_y$, gives $T \ge \lambda\,p_xp_y/(1-\sigma)^2$, so $\chi \ge 1$. For $a_x = e_x/\sigma$ and $a_y = e_y/\sigma$, $T = \lambda\,a_xa_y = \lambda\,e_xe_y/\sigma^2 = N$ when the sides are whole numbers of patches, and $p_xp_y = e_xe_y(1/\sigma - 1)^2 = N(1-\sigma)^2/\lambda$, so $\chi = 1$. A fixed overhead of $\omega$ tokens per image adds $\omega$ per view. With $T$ tokens per view, the rule then uses at most $\Gamma\,\Cmin(1 + \omega/T)$ tokens, while the calls of \cref{thm:exact}(c) number at least $\Cmin/N$, since each uses at most $N$ tokens, and so use at least $\Cmin(1 + \omega/N)$; the ratio bound becomes $\Gamma\,(1 + \omega/T)/(1 + \omega/N)$, between $\Gamma$ and $\Gamma\,N/T$, and equal to $\Gamma$ when $T = N$, as for ideal views and fixed grids. For a fixed grid, a square view of side $a = mg/S$ has $S_V = S$ by \cref{prop:geometry}(c) and costs $g^2$; the same tiling argument applies, and for square targets $\lambda p^2 = (S^2/m^2)(mg/S - m)^2 = g^2(1 - S/g)^2 = g^2(1-\sigma)^2$, so $\chi = 1$.
\end{proof}

\paragraph{The rule as a recipe.} Choose $S = \max(S^\ast, \Sw)$, cut the image into views that overlap by one target, each as large as the budget allows at $S$ (about $\sqrt N \times \sqrt N$ tokens for square targets), and send each view without further resizing. The ideal views have aspect ratio $e_x/e_y$ and cost exactly $N$ tokens; when they exist within the vendor's limits, the ratio of the rule's cost to $\Cmin$ tends to $\chi$ as the image grows relative to the views. For square targets the ideal views are $\sqrt N \times \sqrt N$ patches whatever $S$, for example $50 \times 50$ at $N = 2{,}500$, where $\sigma = 0.03$ at $S = 1.5$. Keeping the whole view as the core removes the second margin of the covering construction of \cref{thm:safe}. Points near a view's edge, where a target may be cut, are kept, so the rule can add false points there, and targets in the overlaps of adjacent views can be reported twice; both are unmatched points priced by \cref{prop:f1}. The rule does not merge them; a merging step, where one is used, is the step whose loss $\zeta$ enters \cref{thm:robust}. Step 3 is treated next.

\paragraph{Step 3 in two forms.}\label{sec:step3-proof} Step 3 of \cref{alg:rule} splits a view $c$ of the rule into $q \times q$ sub-views $c_1, \dots, c_{q^2}$ at the same scale, each its own core, overlapping by $e_x$ and $e_y$, so that the sub-views' interiors, each the sub-view shrunk by $e_x/2$ and $e_y/2$, tile the interior of $c$. Let $r_q$ be their image tokens over those of $c$ (\cref{tab:step3}).
\emph{Replacing.} If the sub-views are sent instead of $c$, each is a sub-region of the same pixels at a scale at least $s_{\mathrm{w}}$, hence a crop of $\cw$, and every centre of $\Jw$ interior to $c$ is interior to some $c_j$. The split with $c$ replaced therefore still satisfies the hypothesis of \cref{thm:exact}(a), and under (R) and (M) it detects in expectation at least as much as $\cw$; the view costs $r_q$ times its tokens. This argument needs the replaced views to be fixed by the image's size, as all calls in \cref{sec:defs} are; a subdivision chosen from a view's content or answers changes which calls are compared and is covered by the appending form.
\emph{Appending.} If $c$ is sent and the sub-views are sent as an additional batch, the returned set $O$ of $c$ becomes $O' = O \uplus O_{\mathrm{new}}$, the points of $c$ together with those of the sub-views. On every image and every run, a target with a point of $O$ within $\tau$ has one in $O'$, so the hits cannot fall, and a maximum one-to-one matching of $O$ to the targets within $\tau$ remains a matching for $O'$, so the matched true positives cannot fall either. No assumption on the model is used, so the views to subdivide may be chosen from their content or answers. The unmatched points may grow; \cref{prop:f1} prices them. The view costs $1 + r_q$ times its tokens.

\paragraph{Which form of the assumption is needed.} In \cref{cor:safe}, $J_{c_i} \cap \Jw$ is the half-open view shrunk by $e_x/2$ and $e_y/2$, a rectangle of size $p_x \times p_y$, intersected with $\Jw$, and these rectangles tile a region containing $\Jw$; in \cref{thm:assign} below, $J_{c_i} \cap \Jw = K_i \cap (X \ominus B_0)$ and the cores partition the image. In both, the sets $U_i = J_{c_i} \cap \Jw$ partition $\Jw$, and the proofs apply (M) to $\cw$ and the intermediate call $\tilde c_i$ on $U_i = J_{\tilde c_i} \cap \Jw$, and (R) to $\tilde c_i$ and $c_i$ on $U_i = J_{c_i} \cap J_{\tilde c_i} \cap \Jw$, using $J_{\tilde c_i} = J_{c_i}$. No other region enters. These conclusions therefore hold under the much weaker assumption that, for the pairs of calls the proof compares, the expected number of targets interior to both calls and to the whole image that are detected does not decrease. The controlled test (\cref{sec:assume-direct}) estimates quantities of exactly this form.

\subsection{Sign, robustness and F1}
\label{sec:proofs-robust}

\begin{restate}{Proposition}{prop:nonid}{Attainable recall gains, full form}
Suppose target centres are more than $2\tau$ apart, the weighted expected number $\mathbb{E}[a\,n_{\mathrm{w}}]$ of targets interior to $\cw$ is positive and finite, and $F(z) = \mathbb{E}[a\,n_{\mathrm{w}}(z)]/\mathbb{E}[a\,n_{\mathrm{w}}]$ is continuous, where $n_{\mathrm{w}}(z)$ counts the targets interior to $\cw$ whose first coordinate is at most $z$. Let $c_1, \dots, c_k$ satisfy the hypotheses of \cref{thm:safe}, including its covering condition, none of them showing the whole image at scale $s_{\mathrm{w}}$. For any $0 \le r_{\mathrm{w}} \le r_{\mathrm{d}} \le 1$ there is a model satisfying (R), (M) and (I) for every target under which the weighted expected number of targets interior to $\cw$ that $\cw$ detects is $r_{\mathrm{w}}\mathbb{E}[a\,n_{\mathrm{w}}]$, and that $c_1, \dots, c_k$ detect is $r_{\mathrm{d}}\mathbb{E}[a\,n_{\mathrm{w}}]$.
\end{restate}

\begin{proof}
Let $u(x) = F(x_1)$ for a centre $x$. Let $\theta(c) = r_{\mathrm{w}}$ if $s_c \le s_{\mathrm{w}}$ and the region $V_c$ covers the image's whole extent in the source's own placement (from whatever source, a lower-resolution copy included, and up to the translation that places it in a composite), and $\theta(c) = r_{\mathrm{d}}$ otherwise. The model returns, for a call $c$, the centres of the targets with centres in $J_c$ and $u(x) < \theta(c)$, and nothing else. Detection is deterministic, and $p_c(x) = \mathbb{1}[u(x) < \theta(c)]$ for $x \in J_c$. (R): if $c'$ refines $c$, the views agree and $s_{c'} \ge s_c$, so $\theta(c') \ge \theta(c)$. (M): let $c'$ be a crop of $c$ at the same scale. If $\theta(c') = r_{\mathrm{w}}$, then $V_{c'}$ contains the whole image at $s_{c'} = s_c \le s_{\mathrm{w}}$, and $V_c \supseteq V_{c'}$ with a source that agrees on $V_{c'}$, so $\theta(c) = r_{\mathrm{w}}$; hence $\theta(c') \ge \theta(c)$ in every case. (I): the calls compared have the same view and scale, so $\theta$ is equal. All three hold target by target, hence for every $\mathcal{D}$. The whole-image call detects the targets in $\Jw$ with $u < r_{\mathrm{w}}$. No $c_i$ shows the whole image at scale $s_{\mathrm{w}}$, and a crop of $\cw$ has $s_{c_i} \ge s_{\mathrm{w}}$, so $\theta(c_i) = r_{\mathrm{d}}$; under the covering hypothesis every centre in $\Jw$ lies in some $J_{c_i}$, so the split detects exactly the targets in $\Jw$ with $u < r_{\mathrm{d}}$. Returned points match only their own targets because centres are more than $2\tau$ apart. By the probability integral transform applied to the weighted intensity measure, $\mathbb{E}[a\,\#\{x \in \Jw : u(x) < r\}] = \mathbb{E}[a\,n_{\mathrm{w}}(F^{-1}(r))] = r\,\mathbb{E}[a\,n_{\mathrm{w}}]$ for continuous $F$. This construction, like those of \cref{prop:tight}, \cref{thm:exact}(b) and \cref{thm:assign}, uses a model that is given the image with each call, as the definition of a model allows.
\end{proof}

Monotonicity therefore fixes the sign of the gain, and any gain between zero and the whole-image call's miss rate is attainable; its size is a property of the model, which the experiments measure.

\begin{theorem}[Splitting with approximate monotonicity]
\label{thm:robust}
Let $\cw$ and crops $c_1, \dots, c_k$ be as in \cref{thm:safe}, let $U_1, \dots, U_k$ be disjoint regions with $U_i \subseteq J_{c_i} \cap \Jw$, let $U_0 = \Jw \setminus \bigcup_i U_i$, and let $\tilde c_i$ show the region of $c_i$ at the scale of $\cw$, as in the proof of \cref{lem:crop}. Suppose that, for each $i$,
\begin{align*}
\nu_{\tilde c_i}(U_i) &\ge \nu_{\cw}(U_i) - \epsilon_{M,i}\,\mu(U_i),\\
\nu_{c_i}(U_i) &\ge \nu_{\tilde c_i}(U_i) - \epsilon_{R,i}\,\mu(U_i),
\end{align*}
and that a final merging step removes in expectation at most $\zeta\,\mu(\Jw)$ hits. Then, with $\mathcal{R}$ the expected number of detected targets interior to $\cw$ divided by $\mu(\Jw)$ (a ratio of expected counts),
\[
\mathcal{R}_{\mathrm{d}} \ge \mathcal{R}_{\mathrm{w}} - \psi - \bar\epsilon - \zeta,
\]
where $\psi = \mu(U_0)/\mu(\Jw)$ is the uncovered share and $\bar\epsilon = \sum_i (\epsilon_{M,i} + \epsilon_{R,i})\,\mu(U_i)/\mu(\Jw)$; with uniform defects $\epsilon_M, \epsilon_R$, $\mathcal{R}_{\mathrm{d}} \ge \mathcal{R}_{\mathrm{w}} - \psi - (1-\psi)(\epsilon_M + \epsilon_R) - \zeta$. No term grows with the number of views. If centres are more than $2\tau$ apart, the calls $\cw, \tilde c_i, c_i$ are distinct, and $0 \le \epsilon_{M,i}, \epsilon_{R,i}$ with $\epsilon_{M,i} + \epsilon_{R,i} \le 1$, the bound with $\zeta = 0$ is attained by some model.
\end{theorem}

\begin{proof}
A target with centre in $U_i$ is detected by the split whenever $c_i$ detects it, so before merging the split detects in expectation at least
\begin{align*}
&\sum_{i} \nu_{c_i}(U_i)\\
&\quad\ge \sum_i \big(\nu_{\cw}(U_i) - (\epsilon_{M,i} + \epsilon_{R,i})\mu(U_i)\big)\\
&\quad= \nu_{\cw}(\Jw) - \nu_{\cw}(U_0) - \bar\epsilon\,\mu(\Jw)
\end{align*}
targets interior to $\cw$, by the two defect inequalities and additivity over the disjoint $U_0, U_1, \dots, U_k$. Merging removes at most $\zeta\,\mu(\Jw)$; dividing by $\mu(\Jw)$ and using $\nu_{\cw}(U_0) \le \mu(U_0)$ gives the bound, and $\sum_{i \ge 1}\mu(U_i) = (1-\psi)\mu(\Jw)$ gives the uniform form. No independence is used. For attainment with $\zeta = 0$, let $\cw$ return the centre of every target in $\Jw$; let $\tilde c_i$ return each centre in $U_i$ independently with probability $1 - \epsilon_{M,i}$ and $c_i$ with probability $1 - \epsilon_{M,i} - \epsilon_{R,i}$; let no call other than $\cw$ return a centre in $U_0$, no $c_i$ return a centre outside $U_i$, and merging remove nothing. Because the calls are distinct, these laws define a model; because centres are more than $2\tau$ apart, a returned centre detects only its own target. Then $\mathcal{R}_{\mathrm{w}} = 1$, both defect inequalities hold with equality, and $\mathcal{R}_{\mathrm{d}} = 1 - \psi - \bar\epsilon$.
\end{proof}

The defects are quantities of the kind the controlled test estimates. With the defects set to the worst values within the 95\% intervals of the (M) and (R) columns of \cref{tab:assume} (two intervals measured on other pairs of calls), and $\psi = \zeta = 0$ as for \cref{cor:safe}, the guarantee weakens by at most \XrobMax{} (for \XrobMaxModel), and not at all for \XrobZero{} models. A large defect makes the guarantee nearly vacuous: for GPT-5.4 given images that the vendor resizes, the measured loss in (R) is \XtcViolDpos{} in recall (\cref{sec:violations}).

\cref{ass:mono} concerns detection and says nothing about points that match no target. Their effect on F1 has an exact price.

\begin{proposition}[F1 of a split]
\label{prop:f1}
On an image with $n$ targets, let a set of calls return the kept points $O$, let $\mathrm{TP}$ be the size of a maximum one-to-one matching of points to targets within $\tau$, and let $\mathrm{FP} = |O| - \mathrm{TP}$, so that $F_1 = 2\mathrm{TP}/(n + \mathrm{TP} + \mathrm{FP})$. For two sets of calls $\mathrm{d}$ and $\mathrm{w}$ with $\mathrm{TP}_{\mathrm{w}} > 0$,
\[
F_{1,\mathrm{d}} \ge F_{1,\mathrm{w}} \iff \Delta\mathrm{FP} \le \rho_{\mathrm{w}}\,\Delta\mathrm{TP},
\]
\[
\rho_{\mathrm{w}} = \frac{n + \mathrm{FP}_{\mathrm{w}}}{\mathrm{TP}_{\mathrm{w}}} \ge \frac{1}{\mathrm{Rec}_{\mathrm{w}}},
\]
where $\Delta\mathrm{FP} = \mathrm{FP}_{\mathrm{d}} - \mathrm{FP}_{\mathrm{w}}$, $\Delta\mathrm{TP} = \mathrm{TP}_{\mathrm{d}} - \mathrm{TP}_{\mathrm{w}}$ and $\mathrm{Rec}_{\mathrm{w}} = \mathrm{TP}_{\mathrm{w}}/n$. Since $\mathrm{TP}_{\mathrm{d}} \le n$, no set of calls that adds more than $(n + \mathrm{FP}_{\mathrm{w}})(1/\mathrm{Rec}_{\mathrm{w}} - 1)$ unmatched points can keep F1 at least that of $\mathrm{w}$. The same holds with $n$, $\mathrm{TP}$ and $\mathrm{FP}$ replaced by their expectations over $\mathcal{D}$ and the model. If target centres are more than $2\tau$ apart and every target is interior to $\cw$, then under (R) and (M) a split satisfying the hypotheses of \cref{thm:safe} or of \cref{thm:exact}(a) has $\mathbb{E}\,\mathrm{TP}_{\mathrm{d}} \ge \mathbb{E}\,\mathrm{TP}_{\mathrm{w}}$, so its F1 computed from expected counts is at least that of $\cw$ whenever it adds no more unmatched points than hits.
\end{proposition}

\begin{proof}
With $|O| = \mathrm{TP} + \mathrm{FP}$, $F_1 = 2\mathrm{TP}/(n + |O|)$. For $\mathrm{TP}_{\mathrm{w}} > 0$, $F_{1,\mathrm{d}} \ge F_{1,\mathrm{w}}$ holds if and only if $\mathrm{TP}_{\mathrm{d}}(n + \mathrm{TP}_{\mathrm{w}} + \mathrm{FP}_{\mathrm{w}}) \ge \mathrm{TP}_{\mathrm{w}}(n + \mathrm{TP}_{\mathrm{d}} + \mathrm{FP}_{\mathrm{d}})$, that is, $(\mathrm{TP}_{\mathrm{d}} - \mathrm{TP}_{\mathrm{w}})(n + \mathrm{FP}_{\mathrm{w}}) \ge \mathrm{TP}_{\mathrm{w}}(\mathrm{FP}_{\mathrm{d}} - \mathrm{FP}_{\mathrm{w}})$; dividing by $\mathrm{TP}_{\mathrm{w}}$ gives the condition, and $\rho_{\mathrm{w}} \ge n/\mathrm{TP}_{\mathrm{w}} = 1/\mathrm{Rec}_{\mathrm{w}}$. With $\Delta\mathrm{TP} \le n - \mathrm{TP}_{\mathrm{w}}$, the condition implies $\Delta\mathrm{FP} \le \rho_{\mathrm{w}}(n - \mathrm{TP}_{\mathrm{w}}) = (n + \mathrm{FP}_{\mathrm{w}})(1/\mathrm{Rec}_{\mathrm{w}} - 1)$. The argument uses only the counts, so it holds equally for their expectations. If centres are more than $2\tau$ apart, the disks of radius $\tau$ around them are disjoint, so a maximum matching within $\tau$ matches exactly one point to each detected target and $\mathrm{TP}$ equals the number of detected targets; \cref{thm:safe} then gives $\mathbb{E}\,\mathrm{TP}_{\mathrm{d}} \ge \mathbb{E}\,\mathrm{TP}_{\mathrm{w}}$, and if $\mathbb{E}\,\Delta\mathrm{FP} \le \mathbb{E}\,\Delta\mathrm{TP}$ the condition holds because $\rho_{\mathrm{w}} \ge 1$.
\end{proof}

The rate $\rho_{\mathrm{w}}$ is large whenever the whole-image call has low recall: a split can then add many unmatched points per additional hit and still raise F1, whereas when the whole image already yields most targets it can afford few. Duplicates in overlaps are unmatched points and are priced at the same rate.

\subsection{Information: \cref{thm:capacity}}
\label{sec:proofs-bits}

\begin{restate}{Theorem}{thm:capacity}{What a hit costs, full form}
Let a call, or any procedure, return a finite set $\hat Y$ of $R$ points for a region of area $A$ containing the centres $Y$ of $n$ targets, with $n$ and $A$ fixed and $Y$ a labelled tuple whose labels are exchangeable, and let $\mathrm{TP}$ be the size of a maximum one-to-one matching of points to centres within $\tau$. Let $p_k$ be the largest probability, over $k$ distinct labelled centres and $k$ disks of radius $\tau$, that each of these centres lies in its disk, and $(n)_k = n!/(n-k)!$. Then
\[
\mathbb{E}\big[-\log_2 p_{\mathrm{TP}} - \log_2 (n)_{\mathrm{TP}}\big] \le I(Y; \hat Y) + o,
\]
where $o = \mathbb{E}\log_2\binom{R}{\mathrm{TP}} + \log_2(n+1)$. If the centres are independent and uniform on the region, then with $\Lbit = n\log_2\big(A/(n\pi\tau^2)\big)$ and the matched recall $\mathrm{Rec} = \mathrm{TP}/n$, $\mathbb{E}[\mathrm{Rec}] \cdot \Lbit \le I(Y; \hat Y) + o$. Consequently, for independent uniform centres and $\Lbit > 0$: (a) a call whose answer carries at most $C$ bits about the target positions has expected recall at most $(C + o)/\Lbit$; (b) any policy whose report depends on the image only through the model's responses reaches expected recall $\mathbb{E}[\mathrm{Rec}]$ only if its transcript carries at least $\mathbb{E}[\mathrm{Rec}]\,\Lbit - o$ bits about the target positions, with $\Lbit$ computed for the whole image.
\end{restate}

\begin{proof}
Fix a deterministic rule that picks a maximum matching $M$ for each pair $(Y, \hat Y)$, with $\hat Y$ listed in a fixed order; $M$ records $\mathrm{TP}$, which labelled centres and which points are matched, and how. Then $I(Y; \hat Y) = I(Y; \hat Y, M) - I(Y; M \mid \hat Y) \ge I(Y; \hat Y, M) - H(M \mid \hat Y)$. Given $(\hat Y, M)$ with $\mathrm{TP} = k$, $Y$ lies in the set $\Sigma$ of configurations in which the $k$ matched labelled centres lie in the disks of radius $\tau$ around their matched points, and $P_Y(\Sigma) \le p_k$. Since the conditional law of $Y$ is supported on $\Sigma$, $D(P_{Y \mid \hat Y, M} \,\|\, P_Y) \ge -\log_2 P_Y(\Sigma) \ge -\log_2 p_k$, so $I(Y; \hat Y, M) \ge \mathbb{E}[-\log_2 p_{\mathrm{TP}}]$. Given $\hat Y$, $M$ is determined by $\mathrm{TP}$ (at most $n+1$ values), the matched labels and points, and their pairing, at most $(n)_k\binom{R}{k}$ choices for $\mathrm{TP} = k$; so $H(M \mid \hat Y) \le \log_2(n+1) + \mathbb{E}[\log_2 (n)_{\mathrm{TP}} + \log_2 \binom{R}{\mathrm{TP}}]$. Combining gives the first inequality. For independent uniform centres, a disk of radius $\tau$ meets the region in area at most $\pi\tau^2$, so $p_k \le (\pi\tau^2/A)^k$, and $(n)_k \le n^k$; hence the left side is at least $\mathbb{E}[\mathrm{TP}]\log_2(A/(n\pi\tau^2)) = \mathbb{E}[\mathrm{Rec}]\,\Lbit$. (a) follows directly. (b): the report is a function of the transcript $T$, which includes the policy's randomness, so $Y \to T \to \hat Y$ and $I(Y; \hat Y) \le I(Y; T)$, which equals the information the responses carry about $Y$ given the policy's randomness, because that randomness is independent of $Y$.
\end{proof}

\paragraph{Tightness near full recall.} Let the region be a rectangle whose sides are multiples of $\sqrt2\tau$, cut into $N_{\mathrm{cell}} = A/(2\tau^2)$ cells of side $\sqrt2\tau$, and let $n \le N_{\mathrm{cell}}$. The centre of a cell lies within $\tau$ of every point of the cell. Report the centre of every cell that contains a target centre. Each reported point lies within $\tau$ of a target in its own cell and distinct cells give distinct targets, so $\mathrm{TP} = R$ is the number of occupied cells and $o = \log_2(n+1)$. Two independent uniform centres share a cell with probability $1/N_{\mathrm{cell}}$, so $\mathbb{E}[\mathrm{Rec}] \ge 1 - (n-1)/(2N_{\mathrm{cell}})$. The report is a set of at most $n$ cells, so $I(Y; \hat Y) \le \log_2 \sum_{j \le n} \binom{N_{\mathrm{cell}}}{j} \le n\log_2(eN_{\mathrm{cell}}/n) = \Lbit + n\log_2(e\pi/2)$: about 2.1 bits per target more than $\Lbit$. The term $o$ charges reports that match nothing: an answer that lists every cell of a grid finds every target while saying nothing about where they are. Doubling both the targets and the area of a call doubles $\Lbit$; doubling the targets at a fixed area doubles it less one bit per target; packing $k$ views of equal density into one image adds their contents.

\subsection{Search strategies: \cref{prop:frontier}, adaptive zooming and zoom tools}
\label{sec:proofs-search}

\begin{assumption}[Graded recognition]
\label{ass:graded}
Each image's pixels depend on the image $X$ only within its view, and the model's response to a call depends only on the call's images and text and on the history of responses. There is a non-decreasing $f: [0,\infty) \to [0,1]$ such that, for any two images that differ only by one target with box $B$ and for any history of earlier responses, the laws of the model's response to a call under the two images differ in total variation by at most $\sum_I f(S_I)$, the sum over the images $I$ of the call whose views meet $B$.
\end{assumption}

$f(S)$ is at least the probability that a view at $S$ detects a target it shows minus the probability that it reports one where there is none. Because $f$ bounds the change in the whole response, a measured detection curve gives a lower bound on $f$.

\begin{restate}{Proposition}{prop:frontier}{Recall-cost frontier, full form}
Under \cref{ass:graded}, let $X_0$ be an image and $\Omega$ a measurable set of centres, $0 < |\Omega| < \infty$, such that for every $x \in \Omega$ the image $X_x$, $X_0$ with one target of box $B_x = x + B_0$ added at $x$, is valid and $x$ is more than $2\tau$ from every target of $X_0$. Consider any policy, possibly adaptive and randomized, whose calls and final report depend on the image only through the model's responses, with randomness independent of the image, and which makes at most $K$ calls almost surely, for some $K$. Let $\bar R$ be the probability, averaged over $x \in \Omega$, that its report on $X_x$ contains a point within $\tau$ of $x$, and $\overline{\mathrm{FP}}_0$ the expected number of reported points on $X_0$ farther than $\tau$ from every target. Then its expected token cost on $X_0$, counting all images of every call, is at least
\[
\frac{\bar R\, |\Omega| - \pi\tau^2\, \overline{\mathrm{FP}}_0}{\beta\, \kappa_f\, m^2}, \qquad \kappa_f = \sup_{S > 0} \frac{f(S)}{S^2},
\]
where $\beta \ge 1$ bounds $(1 + w/a_x)(1 + h/a_y)$ over the $a_x \times a_y$ views it uses.
\end{restate}

\begin{proof}
Fix $x \in \Omega$ and write $\Pr_0, \Pr_x$ for laws under $X_0$ and $X_x$. Let the transcript be the sequence of calls and responses, together with the policy's randomness. For $j = 0, \dots, K$, let $H_j$ be the law of the transcript when the first $j$ responses are drawn under $X_0$ and the rest under $X_x$, the policy choosing calls from the history as always; $H_K$ and $H_0$ are the laws under $X_0$ and $X_x$. $H_j$ and $H_{j-1}$ share the law of the first $j-1$ steps and use the same kernels after step $j$, so by the data-processing inequality for total variation, $\mathrm{TV}(H_j, H_{j-1})$ is at most the expected total variation of the $j$-th response law given the history, under $X_0$, which by \cref{ass:graded} is at most $\mathbb{E}_0 \sum_{I \in \text{call } j,\, V_I \cap B_x \ne \emptyset} f(S_I)$. By the triangle inequality, $\mathrm{TV}(\Pr_0, \Pr_x) \le \mathbb{E}_0 \sum_{j}\sum_{I \in \text{call } j} f(S_I)\,\mathbb{1}[V_I \cap B_x \ne \emptyset]$. The report is a function of the transcript and of independent randomness, so for the event $F_x$ that it contains a point within $\tau$ of $x$, $\Pr_x(F_x) \le \Pr_0(F_x) + \mathrm{TV}(\Pr_0, \Pr_x)$. Integrate over $x \in \Omega$. The left side gives $\bar R|\Omega|$. For the first term on the right, every reported point within $\tau$ of some $x \in \Omega$ is farther than $\tau$ from every target of $X_0$, because $x$ is more than $2\tau$ from them, and it covers at most $\pi\tau^2$ of $\Omega$; so $\int_\Omega \Pr_0(F_x)\,dx \le \pi\tau^2\,\overline{\mathrm{FP}}_0$. For the second, $\{x : V_I \cap (x + B_0) \ne \emptyset\} = V_I \oplus B_0$ has area $(a_x + w)(a_y + h) \le \beta |V_I|$, and an image costs at least $|V_I| S_I^2/m^2$ tokens, so $f(S_I)|V_I| \le \kappa_f m^2\,\mathrm{tokens}(I)$. Summing over images and calls bounds the integral by $\beta\kappa_f m^2$ times the expected token cost on $X_0$, and rearranging gives the claim.
\end{proof}

If $f(S) = \min\{1, (S/S_0)^\gamma\}$ with $\gamma \ge 2$, then $f(S)/S^2 = S^{\gamma-2}/S_0^\gamma \le 1/S_0^2$ for $S \le S_0$ and $1/S^2 \le 1/S_0^2$ above, so $\kappa_f = 1/S_0^2$: every search pays the coverage cost of \cref{prop:geometry}(b) at $S^\ast = S_0$ and $\kappa = 1$, up to the factors $\bar R$, $\beta$ and $|\Omega|/(WH)$, less the price of its false points. For $\gamma < 2$ the ratio is unbounded as $S \to 0$ and the bound is vacuous: spreading tokens thinly over the whole image could then pay. Which case holds is a property of the model, and lowering $S_0$ lowers the frontier quadratically. The frontier covers tiles, high-resolution calls, model-based propose-and-verify and zoom agents, with imperfect recall and false points charged at an explicit rate.

\paragraph{A sharp threshold.} Under a hard recognition threshold, an always-correct adaptive policy must pay the coverage cost exactly.

\begin{assumption}[Recognition threshold]
\label{ass:threshold}
There is an $S_{\mathrm{rec}} > 0$ such that, for any two images that differ only by one target with box $B$ and for any history of earlier views and responses, the model's response to a view $V$ has the same conditional distribution under both images unless $V$ shows targets at $S_V \ge S_{\mathrm{rec}}$ and $B$ lies in $V$.
\end{assumption}

\begin{proposition}[Adaptive zooming pays the coverage cost]
\label{prop:adaptive}
Under \cref{ass:threshold}, let $\mathcal{F}$ be a family of images, $X_0 \in \mathcal{F}$, and $D$ the set of centres $x$ such that $X_0$ with one target added at $x$ lies in $\mathcal{F}$. Consider any policy, possibly adaptive and randomized, whose views and final report depend on the image only through the model's responses, that takes views and then reports a target set, and suppose that on every image in $\mathcal{F}$ the report is correct (a one-to-one matching within $\tau$ covers every instance and leaves no report unmatched) with probability at least $1-\delta$, $\delta < \tfrac12$. Then on $X_0$ its expected token cost is at least $(1-2\delta)\,|D|\,S_{\mathrm{rec}}^2/m^2$.
\end{proposition}

\begin{proof}
Write $X_x$ for $X_0$ with one target, box $B_x$, added at centre $x \in D$. Let $Z$ be the union of all views the policy takes at $S_V \ge S_{\mathrm{rec}}$. Each such view costs at least $|V|S_V^2/m^2 \ge |V|\,S_{\mathrm{rec}}^2/m^2$ tokens and $\sum_V |V| \ge |Z|$, so the cost is at least $|Z|\,S_{\mathrm{rec}}^2/m^2$; it suffices to show $\mathbb{E}_0|Z| \ge (1-2\delta)|D|$. Fix $x \in D$ and let $E_x$ be the event that no view with $S_V \ge S_{\mathrm{rec}}$ contains $B_x$. By \cref{ass:threshold} and induction over views, the transcripts under $X_0$ and $X_x$ have the same law up to the first view at $S_V \ge S_{\mathrm{rec}}$ that contains $B_x$. Hence $\Pr_0(A \cap E_x) = \Pr_x(A \cap E_x)$ for every event $A$ of the report. No report is correct for both images, because a correct report has as many points as instances, which differ by one. Therefore
\begin{align*}
&\Pr_0(\mathrm{err}) + \Pr_x(\mathrm{err})\\
&\quad\ge \Pr_0(\hat Y \ne Y_0,\, E_x) + \Pr_x(\hat Y \ne Y_x,\, E_x)\\
&\quad= \Pr_0(\hat Y \ne Y_0,\, E_x) + \Pr_0(\hat Y \ne Y_x,\, E_x)\\
&\quad\ge \Pr_0(E_x),
\end{align*}
and correctness at least $1-\delta$ on both images gives $\Pr_0(E_x) \le 2\delta$. Since $x \in B_x$, the event $x \notin Z$ implies $E_x$, so $\Pr_0(x \notin Z) \le 2\delta$, and by Fubini $\mathbb{E}_0|D \setminus Z| \le 2\delta|D|$, hence $\mathbb{E}_0|Z| \ge (1-2\delta)|D|$.
\end{proof}

If each view with $S_V < S_{\mathrm{rec}}$ or not containing $B_x$ may change the response law by at most $\epsilon$ in total variation, and the policy takes at most $k$ views, the same argument on the stopped transcript gives $(1 - 2\delta - k\epsilon)|D|S_{\mathrm{rec}}^2/m^2$. Area excluded a priori, such as title blocks and margins, is not in $D$. Agents that crop and zoom therefore do not escape splitting; they perform it, and an overview shown at $S \ge S_{\mathrm{rec}}$ is itself a covering view whose tokens count. Our data show recognition rising gradually rather than at a threshold (\cref{sec:S-curves}), which is why the main text uses the graded frontier.

\paragraph{Zoom tools.} The three steps of the zoom-agent ablation (\cref{sec:agents}) are ordered by the assumptions at the level of single calls.

\begin{proposition}[Zoom tools]
\label{prop:zoom}
Let an overview call $c_{\mathrm{o}}$ show the whole image from a lower-resolution copy at scale $s_{\mathrm{o}}$, with its own core. Fix a region and three calls on it with a common core and the protocol $\pi_{\mathrm{crop}}$, $c_2$ and $c_3$ with the same padding: $c_1$, the region cut from the overview at scale $s_{\mathrm{o}}$; $c_2$, the same cut enlarged to a scale $s' > s_{\mathrm{o}}$; and $c_3$, native pixels of the region at scale $s'$. Let $U$ be a set of centres interior to all four calls. Under (R), (M) and (I), in expected number of targets detected in $U$, $c_1$ is no worse than $c_{\mathrm{o}}$ (M), $c_2$ is no worse than $c_1$ (R), and $c_3$ is no worse than $c_2$, which is a garbling of it (I).
\end{proposition}

\begin{proof}
$c_1$ shows a sub-region of the source of $c_{\mathrm{o}}$ at scale $s_{\mathrm{o}}$, so it is a crop of $c_{\mathrm{o}}$ at the same effective scale (M). $c_2$ shows the same region of the same source as $c_1$ with the same core at a larger scale, so it refines $c_1$ (R). $c_3$ and $c_2$ show the same region with the same core and scale; the source of $c_2$ is the lower-resolution copy, a fixed degradation of the native image, so $c_2$ is a garbling of $c_3$ (I).
\end{proof}

\cref{lem:crop} also orders two other designs we test. When several tiles are packed into one image, each single tile is a crop of the packed image if sent at no smaller scale (up to the separator lines drawn in the tile margins, outside the cores). When content is added around a fixed region, the region alone is a crop of the larger image if that image is sent at no larger scale.

\subsection{Every local task with one position and one answer per object}
\label{sec:proofs-local}

Nothing above depends on what a target is. An \emph{item} has an anchor point $x$; a support box, centred at $x$, inside the $w \times h$ box centred at $x$, and containing the pixels needed to answer the item; and a correct answer, such as a class, a mask, a transcription or a value. An item is interior to a call if the $w \times h$ box centred at its anchor lies inside $V_c$ and its anchor lies at least $\tau$ inside $K_c$.

\paragraph{Enumeration tasks.} A call returns located answers, and an item is detected if a kept answer lies within $\tau$ of its anchor and equals its correct answer. Detection of one or several classes, instance segmentation scored by an answer per anchor, text spotting, keypoint and landmark localization, cell detection and classification, and listing the interactive elements of a screenshot are enumeration tasks. With $p_c(x)$ and $\nu_c(U)$ defined by this event, \cref{ass:mono} is stated as before, and \cref{lem:crop}, \cref{thm:safe}, \cref{prop:tight}, \cref{thm:exact}, \cref{cor:safe}, \cref{thm:robust}, \cref{prop:nonid}, \cref{prop:f1} and \cref{prop:zoom} hold as stated, where a matching pairs only answers and items with equal answers. The proofs of \cref{lem:crop}, \cref{thm:safe}, \cref{thm:exact}(a) and (c), \cref{cor:safe}, \cref{thm:robust} and \cref{prop:zoom} use the detection event only through $\nu_c$ and the sets $J_c$, which depend only on anchors. The proofs of \cref{prop:tight}, \cref{thm:exact}(b) and \cref{prop:nonid}, and the attainment in \cref{thm:robust}, use models that return anchors with their correct answers; with anchors of items with equal answers more than $2\tau$ apart, a returned answer matches an item only if it is that item's anchor and answer.

\paragraph{Assignment tasks.} The items are given and each needs an answer: classify each listed region, read each field of a form, label each pixel or each cell of a grid (semantic segmentation). Calls and regions may depend on the image's size and on the given items, but not on their correct answers. A call answers the items whose anchors lie in its kept core; an item is interior to the call if the $w \times h$ box centred at its anchor lies in the view and its anchor in the core, with no margin $\tau$. $\nu_c(U)$ counts correct answers to the items anchored in $U \cap J_c$. A split answers each item by the call whose core contains its anchor, so no tolerance and no overlap of cores is needed.

\begin{theorem}[Safe splitting for assignment tasks]
\label{thm:assign}
Let $\cw$ show the whole native image $X$, with $X$ as its core, at effective scale $s_{\mathrm{w}}$ and answer every item, and let $c_1, \dots, c_k$ be crops of $\cw$ whose cores, taken half-open and intersected with $X$, partition it, each view $V_i$ containing $(K_i \oplus B_0) \cap X$. Under (R) and (M), over the items whose $w \times h$ boxes centred at their anchors lie in the image, the expected number of correct answers of the split is at least that of $\cw$.
\end{theorem}

\begin{proof}
Let $A$ be the set of anchors of items whose boxes lie in $X$; it depends on the given items, not on their answers. An item with anchor $x \in A \cap K_i$ has its box $x + B_0$ inside $X$, hence inside $(K_i \oplus B_0) \cap X \subseteq V_i$, and its anchor in $K_i$, so it is interior to $c_i$; it is interior to $\cw$, whose view and core are the image. The split answers it by $c_i$, so its expected number of correct answers over these items is $\sum_i \nu_{c_i}(A \cap K_i)$. The proof of \cref{lem:crop} uses only (R), (M) and the interior sets, so it applies to the event of a correct answer and gives $\nu_{c_i}(A \cap K_i) \ge \nu_{\cw}(A \cap K_i)$; summing over the partition gives $\nu_{\cw}(A)$.
\end{proof}

Every item receives exactly one answer, and a wrong answer on a background pixel or an empty field counts like a missed object, so \cref{thm:assign} bounds both kinds of error at once. \cref{thm:exact}(c) and \cref{cor:safe} carry over with $e_x = w$ and $e_y = h$, taking as cores the half-open $p_x \times p_y$ interiors of the views, with the outermost extended to the image's edge, so that the cores partition the image, and \cref{prop:tight}(a) holds with the model that answers correctly when $S_c \ge \Sw$ and incorrectly otherwise.

\subsection{Two reference lines}
\label{sec:proofs-ref}

\begin{assumption}[Token-quantized position]
\label{ass:quant}
The model knows only which token cell contains a target centre, and within the cell the centre is uniform.
\end{assumption}

\begin{proposition}[Quantized-observer ceiling]
\label{prop:loc}
Under \cref{ass:quant}, any reported point $\hat{x}$ satisfies $P(\lVert \hat{x} - x \rVert \le \tau) \le \varphi(S/2)$, where $\varphi(z)$ is the area of the part of the unit square within distance $z$ of its centre, with equality when $\hat{x}$ is the cell centre. In closed form the bound is $\pi S^2/4$ for $S \le 1$ and $1$ for $S \ge \sqrt{2}$.
\end{proposition}

\begin{proof}
Conditioned on its cell $Q$ (side $t$ sent pixels), the centre $x$ is uniform on $Q$. Measured in sent pixels the tolerance is $\tau s$, so $P(\lVert \hat x - x\rVert \le \tau) = \lvert \mathcal{B}(\hat x, \tau s) \cap Q\rvert / t^2$. As a function of $\hat x$ this is the convolution of the indicators of two centrally symmetric convex sets, which by Anderson's theorem~\cite{anderson1955integral} is maximized at the centre of $Q$. Rescaling by $t$ gives $\varphi(\tau s/t) = \varphi(S/2)$. The disk of radius $z$ centred in the unit square lies inside it for $z \le \tfrac12$ (area $\pi z^2$), covers it for $z \ge \tfrac{\sqrt2}{2}$, and in between loses four circular segments, giving $\varphi(z) = \pi z^2 - 4\big(z^2\arccos\tfrac{1}{2z} - \tfrac12\sqrt{z^2 - \tfrac14}\big)$.
\end{proof}

In a one-dot probe every probed model except Gemini 2.5 placed an isolated dot within about 2 px, far below the token size (\cref{sec:coords}), so this observer describes Gemini 2.5, and the newer models localize far more finely than it.

\begin{proposition}[Propose and verify]
\label{prop:hybrid}
Fix a maximum one-to-one matching within $\tau$ between candidates and the $n$ instances. Let $u$ candidates be matched and $v$ unmatched, and let $\alpha_1$ and $\alpha_0$ be the fractions of matched and of unmatched candidates that a verifier accepts. Then
\[
\frac{2\alpha_1 u}{\alpha_1 u + \alpha_0 v + n} \;\le\; F_1 \;\le\; \frac{2q}{1+q}, \qquad q = u/n,
\]
and both bounds equal $2q/(1+q)$ when $\alpha_1 = 1$ and $\alpha_0 = 0$.
\end{proposition}

\begin{proof}
The verifier keeps $\alpha_1 u + \alpha_0 v$ candidates. The $\alpha_1 u$ accepted matched candidates keep their partners in the fixed matching, so a maximum matching of the accepted set has at least $\alpha_1 u$ hits, and $F_1 \ge 2\alpha_1 u/(\alpha_1 u + \alpha_0 v + n)$. Any subset of candidates has at most $u$ hits, and a set with $\mathrm{TP}$ hits has at least $\mathrm{TP}$ elements, so $F_1 \le 2\,\mathrm{TP}/(\mathrm{TP} + n) \le 2u/(u + n)$.
\end{proof}

\section{The interface: token geometry, $S$ and coordinates}
\label{sec:interface}

\subsection{Token geometry}
\label{sec:geometry}

We sent blank white images of several sizes to every model and recorded the billed input tokens; the content of an image does not change its cost. \Cref{tab:calib} compares the measured geometry with vendor documentation~\cite{openai2026vision,anthropic2026vision,google2026mediares}, consulted in October 2026. Two behaviours were not in the documentation when we started: Gemini 2.5's \texttt{HIGH} setting is ignored when a request contains two images, so its high-setting calls place the legend crop in a band above the drawing, in one image, and GPT-5.5 and GPT-5.6 accept at least 8{,}192 patches at the default setting, where the documentation lists 2{,}500 for \texttt{high}; the current documentation states that GPT-5.6's default keeps the original size.

\begin{table*}[t]
\centering
\small
\begin{tabular}{p{3.0cm}p{5.5cm}p{6.2cm}}
\toprule
Model & Geometry (documentation) & Measured (input tokens of a blank image) \\
\midrule
GPT-5.4 & 32 px patches, $\times 1.2$; $\le$2{,}500 (auto), $\le$10{,}000 (original) & $2048^2$: 3{,}019; $3072^2$ original: 11{,}068 \\
GPT-5.5, GPT-5.6 & 32 px patches, $\times 1.2$; 2{,}500 (high), $\le$30{,}000 (original); for GPT-5.6, auto keeps the original size & $2048^2$: 4{,}934; $2896^2$: 8{,}281 patches, not resized; $6000\times4000$ original: 28{,}219 \\
GPT-6.1 & as GPT-5.6 & $3072^2$: 9{,}216 patches, not resized \\
Claude Sonnet 4.6 & 28 px; $\le$1{,}568 tokens, long side $\le$1{,}568 & $1536^2$ and larger: 1{,}542 \\
Claude Sonnet 5.5, Opus 5.5 & 28 px; $\le$4{,}784 tokens, long side $\le$2{,}576 & $1536^2$: 3{,}050; $1932^2$: 4{,}761; $2048^2$: 4{,}786 \\
Gemini 2.5 Flash & fixed grid, 258 tokens (default) & 258 at every size \\
Gemini 3.8 Flash & fixed grid, 1{,}120 (default and high), 2{,}240 (ultra-high) & 1{,}092 to 1{,}111; ultra-high 2{,}189 to 2{,}220 \\
Gemini 4 & not listed separately & about 1{,}091 at every size \\
Gemma 4 31B & fixed budget, 280 tokens (provider default) & 280 at every size \\
Qwen3-VL-235B & 32 px per token (16 px patches, $2\times2$ merge); cap set by the provider & about 2{,}500 at a long side of 1{,}600 px; one provider did not resize up to about 1{,}500 px \\
\bottomrule
\end{tabular}
\caption{Input geometry from documentation and from billed tokens. $\times 1.2$: billed tokens per 32 px patch for GPT; auto, high and original: OpenAI's detail settings. The per-view budgets used by the rule are 8{,}192 patches (GPT-6.1), 4{,}784 tokens (Claude Opus 5.5), about 1{,}090 (Gemini 3.8 and Gemini 4) and 2{,}500 (Qwen3-VL); \cref{sec:thumb} lists them with the targets per side they allow.}
\label{tab:calib}
\end{table*}

\subsection{Computing $S$}
\label{sec:computeS}

For every call we compute $S = ms/t$ from the measured geometry, before any call and without looking at the answer. The effective scale $s$ is the product of the scale at which we send the image and the vendor's own resize factor. For GPT models the vendor factor fits the image to the patch budget (none within 8{,}192 patches for GPT-5.6 and GPT-6.1, 2{,}500 for GPT-5.4); for Claude it is the documented resize to the long-side cap and the token budget of the model's tier; for Qwen3-VL it is $\min\{1, \sqrt{2{,}500 \cdot 1{,}024/(w'h')}\}$ for a sent image of $w' \times h'$ px. For fixed grids the token side is the side of the padded square divided by $\sqrt N$, so $S = mg/\max(W,H)$ for a whole image and $S = mg/\ell$ for a view of longer side $\ell$ (\cref{prop:geometry}), with $N = 258$ (Gemini 2.5), 280 (Gemma 4) and about 1{,}090 (Gemini 3.8 and Gemini 4). On the $4096 \times 2048$ canvas of the area sweep (\cref{sec:L-sweeps}) the fixed-grid token side computed as $\sqrt{WH/N}$ is smaller than the padded-square side by $\sqrt 2$; the values in \cref{tab:syn} use $\sqrt{WH/N}$. For a target with a non-square box, $m = \sqrt{wh}$ is the geometric mean; a stricter criterion that uses the shorter side gives the same conclusions.

\paragraph{The example of \cref{fig:intro}.} The drawing is $\CfigW \times \CfigH$ px and its target spans $m = \CfigM$ px. Sent whole within 2{,}500 patches of 32 px, the schematic interface of \cref{fig:intro} (a long side of about 1{,}930 px, the largest that fits), the target spans \CfigMsent{} px, $S = 0.40$, which is \CfigPctToken\% of a token side. The rule at $S = 1.5$ cuts \CfigViews{} views, each within the budget and about \CfigTokens{} tokens in all; the minimum cost of \cref{thm:exact}(c) is about \CfigCmin{} tokens, and the coverage bound of \cref{prop:geometry}(b) about \CfigPropB.

\paragraph{The small drawing.} On a drawing of $\CsmallW \times \CsmallH$ px with $m = \CsmallM$ px, the whole-image call gives GPT-5.4 $S = \CsmallSw$, and the 44 px recipe shrinks the drawing so that its tiles give $S = \CsmallSt$: the recipe zooms out. GPT-5.4 recalls \CsmallRecWhole{} of the targets in the whole image and \CsmallRecTile{} in the tiles, in both of two runs.

\subsection{Coordinates}
\label{sec:coords}

What reaches the user passes through three properties of a model's answers: the \emph{frame} (the image the coordinates refer to: as sent, or as resized by the vendor), the \emph{scale} (pixels or a 0 to 1000 normalization) and the \emph{format} (how a point is written). Each changed some of our scores by as much as a model upgrade, so every model is read in its published convention, and the reader is fixed before scoring.

\paragraph{Probe.} We placed one black dot at known positions in a $1600 \times 1000$ image and asked for its centre. GPT-5.4 erred by 0 to 1.4 px; Claude Sonnet 4.6 answered in the frame of its internally resized image ($1389 \times 868$) and erred by 0.3 to 2.3 px after rescaling; Gemini 2.5 Flash erred by 24 to 203 px, about one token. GPT-5.5, GPT-5.6, Claude Sonnet 5.5 and Gemini 3.8 answered within about 1 px, both in pixels and on a 0 to 1000 scale.

\paragraph{Conventions.} Gemma 4, Qwen3-VL, Gemini 2.5 and Gemini 3.8 answer on a 0 to 1000 scale even when pixels are requested, and we read them on that scale; Gemini 4 mixes the two (below). Claude Sonnet 4.6 answers in the frame of its resized image. When an input exceeds Claude Sonnet 5.5's 4{,}784-token cap, the model answers sometimes in the resized frame and sometimes in the original one; for such inputs we choose the frame per answer by a rule that uses only the image: the original frame if any point lies outside the resized frame, otherwise the frame under which more points fall on ink (a pixel darker than 128 within $m/2$), with ties going to the documented resized frame. Inputs within the cap, which include every call of the rule, need no such choice. Gemini 3.8 with thinking switches between pixel and 0 to 1000 coordinates from one image to the next.

\paragraph{Gemini 4.} Asked for pixel coordinates on whole images, Gemini 4 mixes the two conventions: \CgemFourOver{} of its \CgemFourAns{} whole-image answers in \cref{tab:main} contain coordinates above 1000, which can only be pixels, and the rest are mostly normalized. Read in any single convention, its whole-image recall is low (all pixels: \CgemFourLsPx{} on LegendSpot and \CgemFourFpPx{} on floor plans; all 0 to 1000: \CgemFourLsNorm{} and \CgemFourFpNorm; per answer, pixels when a coordinate exceeds 1000: \CgemFourLsMix{} and \CgemFourFpMix), far below what it reaches when asked for 0 to 1000 coordinates (below). We therefore also ran its whole-image calls with the same prompt, changing only the two sentences on coordinates to request 0 to 1000, the convention that all Gemini tile calls use. No answer then exceeds 1000, and its whole-image recall is \CtLsGemFourRecW{} on LegendSpot and \CtFpGemFourRecW{} on floor plans; \cref{tab:main} uses these calls. With the pixel prompt the rule's gain over the whole image would be inflated by the coordinate problem, which is neither $S$ nor $L$.

\paragraph{Format.} Qwen3-VL sometimes writes points as \texttt{"cx": [x, y]} lists, and Gemini 3.8 sometimes uses \texttt{x} and \texttt{y} keys instead of the requested \texttt{cx} and \texttt{cy}; our reader keeps both, and recovers points from malformed JSON with a regular expression (4 of 156 answers for Claude Sonnet 5.5 in the controlled sweeps, 20 for Qwen3-VL, 1 for Claude Sonnet 4.6, none for the others); an answer with no recoverable point counts as empty. Points on the 0 to 1000 scale refer to the image as sent, padding included, and every call uses the vendor's default sampling temperature. Reading only \texttt{cx} and \texttt{cy} keys changes results only for Gemini 3.8; \cref{sec:map} shows \cref{fig:map} under that reading.

\paragraph{Agents.} An agent must convert positions seen in a zoomed crop back to the overview, a conversion that no tiled or whole-image call requires; two of our four agents fail at it (\cref{sec:agents}). Unlike recognition and the content one call can handle, the coordinate interface did not improve across generations; zoom tools that convert coordinates themselves (\cref{sec:future}) would take this conversion out of the model's hands.

\section{Data and protocol}
\label{sec:protocol-full}

\subsection{Data}
\label{sec:data}

\paragraph{LegendSpot.} 40 partial screenshots of architectural drawings from different projects and sources, with no project-identifying information. Each image comes with one target symbol cropped from its legend, and every instance of that symbol is annotated with a centre point: 1{,}028 instances, 3 to 116 per image. Twenty images are large sheet regions (long side 2{,}947 to 11{,}484 px; \emph{full sheets}) and twenty are small drawings (483 to 3{,}520 px). Six sheets were used only to write the prompt and to choose the settings; the other 34 images (14 full sheets and 20 views, 660 instances) are the test set used for every result except the map of \cref{fig:map}, which pools all drawings (\cref{sec:map}).

\paragraph{Library-10.} Ten further architectural drawings, annotated by one of the authors before this study: 23 drawing and class items with 225 instances. Their reference symbols come from a shared symbol library rather than from each drawing's legend, so they differ from the drawn instances in scale and stroke. They are used for evaluation only.

\paragraph{Floor plans.} The test split of FloorPlanCAD~\cite{fan2021floorplancad} has $1{,}000 \times 1{,}000$ px rasters of 20\,m blocks. \emph{Floor-plan blocks}: 369 block and class items (21 countable object classes, 3 to 48 instances), 30 of which set thresholds and 339 form the test set, with the median-size instance as the reference. \emph{Floor-plan mosaics}: the first nine blocks of each of 60 source drawings tiled into a $3{,}000^2$ mosaic, whose top-left $2{,}000^2$ and first $1{,}000^2$ block are nested subsets, so the same symbols and reference appear at three areas; each mosaic counts one class (the class with most instances, at least 6, and at least one in the first block). \Cref{tab:main} uses the 60 full $3{,}000^2$ mosaics. FloorPlanCAD blocks are cut from larger drawings, so instances crossing a block edge are truncated and the remainder is unlabelled; we ignore predicted and annotated points within $m/2$ of any block edge.

\paragraph{Natural images.} 150 images of FSC-147~\cite{ranjan2021fsc147} (its 384 px versions), 50 in each of three strata of exemplar size, scored against the annotated dots within $m/2$ of the exemplar scale.

\paragraph{Controlled images.} \CscaleCtrl{} images built from LegendSpot legend templates. Each places $n$ copies of one template (the target) and $n$ copies drawn from three other templates (distractors) on a white canvas, on a jittered grid with spacing at least $1.6m$, after drawing thin grey horizontal and vertical segments (40 per megapixel unless varied). Templates are rescaled so that $m = \sqrt{wh}$ takes the stated value, and the reference image is the native-size target template; seeds are fixed per image. They serve four purposes. \emph{Sweeps} (156 images, six per level): resolution, $m \in \{8, 12, 16, 24, 32, 48, 64\}$ px with $n = 8$ on $1024^2$; target count, $m = 40$ and $n \in \{2, 4, 8, 16, 32, 64, 96\}$ on $1024^2$; area, $m = 40$ and $n = 16$ on canvases of $1024^2$, $1448^2$, $2048^2$, $2896^2$ and $4096 \times 2048$; clutter, $m = 16$ and $n = 8$ with 0 to 640 segments per megapixel; and upsampling, the resolution images with $m \le 16$ enlarged twofold by Lanczos interpolation. \emph{Controlled test}: twelve $2048^2$ images with 48 targets and 48 distractors each, at $m$ of 12, 20 or 32 px (\cref{sec:assume-direct}). \emph{Newest generation}: the recognition measurements of \cref{sec:newest}. \emph{Load at $S = 1.5$}: the sweep of \cref{sec:thumb}.

\paragraph{Scale.} In all we sent about \CscaleCalls{} requests to \CscaleModels{} models on \CscaleImages{} images with about \CscaleInst{} annotated instances.

\begin{table*}[t]
\centering
\footnotesize
\begin{tabular}{p{3.0cm}p{3.2cm}p{4.4cm}>{\raggedright\arraybackslash}p{3.1cm}}
\toprule
Experiment & Data & Models & Where \\
\midrule
Recognition, load, area, clutter, upsampling sweeps & controlled sweeps (156) & GPT-5.4, GPT-5.6, Claude Sonnet 4.6 and 5.5, Gemini 2.5 and 3.8, Gemma 4, Qwen3-VL & \cref{fig:curves}, \cref{sec:S-curves,sec:L-sweeps} \\
Budget at fixed pixels & controlled resolution sweep; floor-plan mosaics & Gemini 3.8 at three budgets & \cref{sec:why}, \cref{sec:budget} \\
Newest generation & controlled & GPT-5.6, GPT-6.1, Claude Opus 5.5, Gemini 3.8, Gemini 4 & \cref{sec:why}, \cref{sec:newest} \\
Controlled test & 12 controlled $2048^2$ images & GPT-5.4, GPT-5.6, GPT-6.1, Claude Sonnet 5.5, Claude Opus 5.5, Gemini 3.8, Gemma 4, Qwen3-VL & \cref{sec:safe}, \cref{sec:assume-direct} \\
Whole image, tiles, high resolution, propose and verify & LegendSpot (34) & the eight sweep models, and GPT-5.5 and Claude Opus 5.5 for the whole image & \cref{fig:map}, \cref{sec:generations,sec:assume-orderings,sec:errors,sec:more} \\
Packed tiles & LegendSpot full sheets (14) & GPT-5.6, Claude Sonnet 5.5 & \cref{sec:why}, \cref{sec:packing} \\
Tiles against whole image & floor-plan blocks (339), mosaics (60 at three areas) & the eight sweep models & \cref{fig:map}, \cref{sec:firstblock,sec:rule-more} \\
Rule against whole image and 44 px recipe & LegendSpot (34), mosaics (60) & GPT-6.1, Claude Opus 5.5, Gemini 3.8, Gemini 4, Qwen3-VL & \cref{tab:main}, \cref{sec:rule-full} \\
Same-scale tiles on natural images & FSC-147 (150) & GPT-5.6, Claude Sonnet 5.5, Gemini 3.8 & \cref{sec:practice}, \cref{sec:rule-more} \\
Zoom agents and their ablation & LegendSpot (34), Library-10, mosaics & GPT-5.6, Claude Sonnet 5.5, Gemini 3.8, Qwen3-VL & \cref{sec:practice}, \cref{sec:agents} \\
Template matching, trained counters & LegendSpot, Library-10, floor plans, FSC-147 & none (local) & \cref{sec:practice}, \cref{sec:rule-more} \\
Load at $S = 1.5$ & controlled & GPT-5.6, GPT-6.1, Claude Opus 5.5, Gemini 3.8, Gemini 4, Qwen3-VL & \cref{tab:thumb}, \cref{sec:thumb} \\
\bottomrule
\end{tabular}
\caption{Which images and models each experiment uses.}
\label{tab:experiments}
\end{table*}

\subsection{Models and settings}
\label{sec:models}

\begin{table*}[t]
\centering
\small
\begin{tabular}{p{3.3cm}p{3.0cm}p{8.4cm}}
\toprule
Model & Access & Reasoning setting in our calls \\
\midrule
GPT-5.4 & OpenAI API & none \\
GPT-5.5 & OpenAI API & none \\
GPT-5.6 Sol & OpenAI API & none; medium effort in the agent comparisons (agent, tiles, whole image); high effort in one further whole-image comparison \\
GPT-6.1 Sol & OpenAI API & lowest accepted effort (\texttt{low}); vendor default for the newest-generation measurement \\
Claude Sonnet 4.6 & Anthropic API & none \\
Claude Sonnet 5.5 & Anthropic API & the API's default adaptive thinking in every call; reasoning arm adds an explicit high effort \\
Claude Opus 5.5 & Anthropic API & the API's default adaptive thinking \\
Gemini 2.5 Flash & Google API & none \\
Gemini 3.8 Flash & Google API & minimal thinking budget; dynamic thinking in the agent and reasoning comparisons; vendor default for the newest-generation measurement \\
Gemini 4 & Google API & lowest thinking level (\texttt{low}); vendor default for the newest-generation measurement \\
Gemma 4 31B & hosted, pinned provider & none \\
Qwen3-VL-235B-A22B Instruct & hosted, pinned provider (a second provider for \cref{sec:newest}, a third for \cref{sec:thumb}) & none (the instruct model has no reasoning mode) \\
\bottomrule
\end{tabular}
\caption{The \CscaleModels{} models. We write GPT-5.6 and GPT-6.1 for the Sol models and Gemini 2.5 and Gemini 3.8 for the Flash models. Versions are those available in September and October 2026. The output cap was 8{,}192 tokens in the controlled sweeps and 16{,}384 for Gemini 4, the newest-generation measurement and the load sweep.}
\label{tab:models}
\end{table*}

\paragraph{Standard input.} Our standard input is input that we resize ourselves so that the vendor need not resize it again: whole-image calls with a long side of at most 1{,}600 px (drawings, floor plans; at most 2{,}048 px in the controlled test), and views within the per-view budgets of \cref{tab:calib}. We checked every experiment against this definition from the measured geometry: the OpenAI models keep all these inputs at their sent size, and the Gemini and Gemma grids apply their single resize. Two kinds of input lie outside it: native-resolution drawings sent through OpenAI's \texttt{detail=original} mode or Gemini's ultra-high setting, and controlled images enlarged to 2{,}048 px for GPT-5.4, whose 2{,}500-patch budget makes the vendor shrink them again. All four violations of the assumptions that we found lie there (\cref{sec:violations}); every pair of calls on standard input is consistent with them.

\subsection{Conditions}
\label{sec:conditions}

Every condition on drawings, floor plans and controlled images uses the same task: find every instance of the reference symbol and return its centre as JSON.
\emph{Whole image}: the image, long side at most 1{,}600 px, in one call.
\emph{High resolution}: the same call with OpenAI's \texttt{detail=original}, which accepts far more pixels, or Gemini's higher media-resolution settings.
\emph{Tiles}: the image cut into views with a margin; a blue frame marks each view's core, and only centres inside it are kept and mapped back. The \emph{44 px recipe} rescales the image so that the reference spans 44 px and cuts 1{,}024 px tiles, so it is defined in pixels; for fixed-grid models its counterpart chooses the field of view by \cref{prop:geometry}(c) so that the target spans 1.5 tokens (Gemma 4) or 2.5 tokens (Gemini).
\emph{The rule}: steps 1 and 2 of \cref{alg:rule} at $S^\ast = 1.5$ with the per-view budgets of \cref{tab:calib}, cores equal to views. $\Sw$ is the $S$ of the whole-image call of standard input (long side at most 1{,}600 px), and the target box $w \times h$ is the size of the reference image (the legend crop; for floor plans, the median-size instance). Each view is cut from the native image, resized with Lanczos interpolation to the scale $s = St/m$, and placed at the top left of a white canvas of whole tokens; for fixed-grid models a view is a square of side $mg/S$ native pixels, which we send without resizing. The reference image is resized by the same factor, a blue frame marks the view's border, and the prompt is that of tiles. All points are kept, so a target in an overlap can be reported twice, and such duplicates count against precision.
\emph{Packed tiles}: $k$ tiles of the 44 px recipe packed into one image, separated by grey lines in the tile margins.
\emph{Zoom agent}: the legend crop and an overview (long side 1{,}600 px) plus a \texttt{zoom} tool that returns any rectangle cut from the native image (long side at most 1{,}024 px, at most twice native), up to 20 calls; two variants of the tool return the rectangle cut from the overview at the overview's scale, or the same crop enlarged to the native crop's size (\cref{sec:agents}).
\emph{Propose and verify}: template matching proposes candidates, and the model accepts or rejects each one, shown as a crop at the reference's scale.
\emph{Template matching}: multi-scale, four-rotation normalized cross-correlation with non-maximum suppression, threshold chosen on the six prompt-writing sheets.
\emph{Trained counters}: CountGD~\cite{amininaieni2024countgd}, GeCo2~\cite{pelhan2026geco2} and the open-vocabulary detector OWLv2 in image-guided mode~\cite{minderer2023owlv2}.

\subsection{Prompts}
\label{sec:prompts}

\paragraph{Whole image.} \textit{You are a precise engineering drawing analysis assistant. You will be shown a reference image of ONE target symbol, followed by a full engineering drawing. Your task is to count every instance of that target symbol in the drawing. The drawing image is $W \times H$ pixels. All coordinates you output must be in this image's pixel space, where (0,0) is the top-left corner.} This is followed by a JSON schema with a \texttt{count} and a list \texttt{shapes} of centres \texttt{cx}, \texttt{cy}, and: \textit{The ``count'' must equal the length of ``shapes''. Count only symbols matching the reference image; do NOT count other symbols, text labels, or drawing linework.} The 0 to 1000 variant replaces the two sentences on coordinates by a request for coordinates normalized to 0 to 1000.

\paragraph{Tiles and views.} \textit{You will be shown a reference image of ONE target symbol, followed by one tile of a larger engineering drawing. Your task is to find every instance of that target symbol whose center lies inside the blue rectangle drawn on the tile. The area outside the blue rectangle is context only.} This is followed by the coordinate convention (pixels of the tile, or 0 to 1000) and the same schema. Views of the rule keep their whole view as the core.

\paragraph{Natural images.} \textit{You will be shown an exemplar crop of ONE object, followed by an image. Find every object in the image that is of the same kind as the exemplar (the exemplar itself is one of them).} The same schema follows, with \textit{Give the centre of each object.} A first, drawing-specific wording led GPT-5.6 to report single objects on natural images; \cref{tab:fsc} reports both.

\paragraph{Verification.} \textit{Image 1 is a reference crop of ONE target symbol taken from the drawing's legend. Image 2 is a small crop of the drawing, shown at the same scale as the reference, with a thin red square marking ONE candidate location. Decide whether the red square contains an instance of the reference symbol.} The answer is \texttt{\{"match": true or false, "confidence": 0-100\}}.

\subsection{Metrics, statistics and labels}
\label{sec:stats}

Predicted and annotated centres are matched one to one by a maximum matching within $\tau = m/2$, with $\mathrm{TP}$ matched pairs; matched recall is $\mathrm{TP}/n$, precision is $\mathrm{TP}$ over the returned points, and F1 is their harmonic mean. Tests of the theory use hit recall, under which a target counts as detected if any returned point lies within $\tau$. Hit recall is at least matched recall, with equality when targets are more than $2\tau$ apart, because each returned point then lies within $\tau$ of at most one target. Mean recall per image is the weighted measure of \cref{sec:defs} with $a(X) = 1/n(X)$, so the theorems apply to it. Runs are averaged per image, and every comparison is a paired bootstrap over images (4{,}000 resamples, percentile intervals), clustered by source drawing for the mosaics and for the orderings of \cref{sec:assume-orderings}; Holm's method~\cite{holm1979simple} corrects within the family of agent comparisons, and other intervals are unadjusted. A tolerance based on the template's ink box, or twice the primary tolerance, leaves the conclusions unchanged; \cref{tab:sens} lists the one comparison of two models that changes order.

\paragraph{Label audit.} After all runs we pooled the detections of template matching and six model configurations on the LegendSpot test set and flagged every location where at least three methods agreed on an instance absent from the labels (60 locations), and every label that no method found (none). The annotator judged each flagged location from a crop shown next to the legend symbol; all 60 were judged not to be instances, so the audit found no missing labels (one-sided 95\% upper bound: 5\% of such locations).

\section{Evidence on $S$}
\label{sec:S-evidence}

\subsection{Recognition curves on controlled images}
\label{sec:S-curves}

\begin{figure*}[t]
\centering
\includegraphics[width=\textwidth]{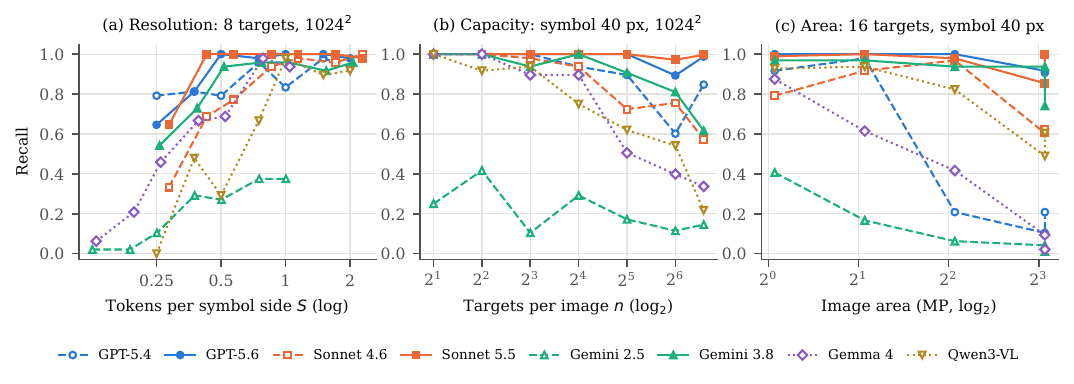}
\caption{Controlled images, eight models. (a) Recall against $S$ (8 targets and 8 distractors on $1024^2$). (b) Recall against the number of targets per call at fixed $S$ (1.25 to 1.43 for resolution-scaled models; symbol 40 px on $1024^2$). (c) Recall against image area with 16 targets of fixed pixel size. Dashed and hollow: older generation; solid and filled: newer generation. Six images per point; mean matched recall per image; no intervals shown. Panels (a) and (b) are \cref{fig:curves}.}
\label{fig:syn}
\end{figure*}

\begin{table*}[t]
\centering
\footnotesize
\setlength{\tabcolsep}{4pt}
\begin{adjustbox}{max width=\linewidth}\begin{tabular}{lcccccc}
\toprule
 & \multicolumn{2}{c}{Recognition} & \multicolumn{3}{c}{Target count (fixed $S$)} & Area \\
\cmidrule(lr){2-3}\cmidrule(lr){4-6}\cmidrule(lr){7-7}
Model & $S_{50}$ & R($S{<}0.5$) & $S$ & R, $n{=}2 \to 96$ & slope & R, 1 $\to$ 8.4 megapixels ($S$) \\
\midrule
Gemini 2.5 Flash & 1.28 & 0.11 & 0.63 & 0.25 $\to$ 0.15 & $-$0.032 & 0.41 $\to$ 0.01 (0.22) \\
Gemma 4 & 0.31 & 0.35 & 0.65 & 1.00 $\to$ 0.34 & $-$0.133 & 0.88 $\to$ 0.02 (0.23) \\
Qwen3-VL & 0.54 & 0.24 & 1.25 & 1.00 $\to$ 0.22 & $-$0.125 & 0.93 $\to$ 0.60 (0.69) \\
Claude Sonnet 4.6 & 0.35 & 0.51 & 1.43 & 1.00 $\to$ 0.57 & $-$0.074 & 0.79 $\to$ 0.62 (0.55) \\
GPT-5.4 & $<$0.25 & 0.80 & 1.25 & 1.00 $\to$ 0.85 & $-$0.052 & 0.92 $\to$ 0.21 (0.68) \\
\midrule
Gemini 3.8 Flash & $<$0.25 & 0.64 & 1.29 & 1.00 $\to$ 0.62 & $-$0.056 & 0.97 $\to$ 0.74 (0.46) \\
Claude Sonnet 5.5 & 0.28 & 0.82 & 1.43 & 1.00 $\to$ 1.00 & $-$0.003 & 0.99 $\to$ 1.00 (0.90) \\
GPT-5.6 & $<$0.25 & 0.73 & 1.25 & 1.00 $\to$ 0.99 & $-$0.010 & 1.00 $\to$ 0.91 (1.25) \\
 
\bottomrule
\end{tabular}\end{adjustbox}
\caption{Controlled images. $S_{50}$: the $S$ at which recall reaches one half (logistic fit of per-image recall on $\log S$; $<$0.25: below 0.25, extrapolated below the smallest $S$ tested). R($S{<}0.5$): mean recall over images with $S < 0.5$. Target count: $1024^2$ canvas, symbol 40 px, $n$ targets and $n$ distractors; slope of recall per doubling of $n$. Area: 16 targets at 40 px on canvases from $1024^2$ to $4096 \times 2048$, with $S$ on the largest canvas in brackets. Upper block: older generation; lower block: newer generation.}
\label{tab:syn}
\end{table*}

Recall rises with $S$ for every model, gradually rather than at a sharp threshold, and the level at which it rises differs by model (\cref{fig:syn}a, \cref{tab:syn}). Gemini 2.5 needs about one token per target side ($S_{50} = \XSfiftygeminiTwoFive$, interval \XSfiftyCIgeminiTwoFive). Gemma 4, Claude Sonnet 4.6 and Qwen3-VL reach one half at \XSfiftygemmaFour{} \XSfiftyCIgemmaFour, \XSfiftysonnetFourSix{} \XSfiftyCIsonnetFourSix{} and \XSfiftyqwenThreevl{} \XSfiftyCIqwenThreevl. GPT-5.4, GPT-5.6, Gemini 3.8 and Claude Sonnet 5.5 reach it near or below the smallest $S$ we tested for them (0.13 to 0.29), so their values (\XSfiftygptFiveFour, \XSfiftygptFiveSix, \XSfiftygemThreeEight, \XSfiftysonnetFiveFive) rest on extrapolation. Clutter, 0 to 640 grey segments per megapixel, changes recall little: the slope per doubling of the segment density has an interval containing zero for every model (for example GPT-5.6 \XclutSlopegptFiveSix{} \XclutSlopeCIgptFiveSix). The difference between models is not geometric. At Gemini 3.8's low budget its token grid is the same as Gemini 2.5's default grid, yet its recall on the same images is higher by \XtOneHThree{} \XtOneHThreeCI.

\subsection{The same pixels at three budgets}
\label{sec:budget}

Gemini 3.8 lets the caller set the token budget per image while the pixels stay fixed: about 257, 1{,}090 or 2{,}210 tokens. On the resolution sweep, recall of 8 px targets rises from \CbudEightLo{} to \CbudEightDef{} to \CbudEightUh{} across the three budgets. A logistic curve in $\log S$ captures most of this; adding a separate offset per budget fits better (difference in Akaike information criterion, $\Delta\mathrm{AIC} = \XtOneDAIC$), so the budget also acts beyond $S$. On floor-plan mosaics, raising the budget from low to default changes F1 by \XfTOnesheetssmallDL{} \XfTOnesheetssmallDLCI{} on the \XfTOnesheetssmallN{} items whose low-budget $S$ is below 1.5 and by \XfTOnesheetslargeDL{} \XfTOnesheetslargeDLCI{} on the other \XfTOnesheetslargeN; on floor-plan blocks the two gains are \XfTOnebenchsmallDL{} \XfTOnebenchsmallDLCI{} and \XfTOnebenchlargeDL{} \XfTOnebenchlargeDLCI. Raising it further, from default to ultra-high, changes F1 by at most \XfTOnesheetslargeUD{} in any of these groups, none significantly: the gain is concentrated where $S$ is low.

\subsection{Dose response on drawings and floor plans}
\label{sec:dose}

\begin{figure}[t]
\centering
\includegraphics[width=0.92\columnwidth]{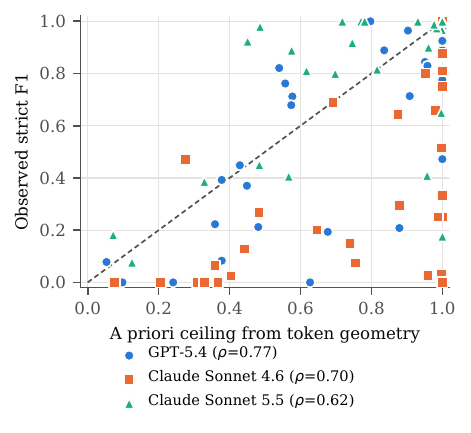}
\caption{Whole-image F1 per LegendSpot test drawing against a ceiling computed from the token geometry before any call (34 drawings; dashed line: observed equals ceiling), for three models. F1 under one-to-one matching (axis: strict F1); Spearman $\rho$ for every model in the text.}
\label{fig:dose}
\end{figure}

For each LegendSpot test drawing we computed a ceiling from the geometry alone, before any call, by simulating the token-quantized observer of \cref{prop:loc}: instances that share a token cell merge into one report at the cell centre. Whole-image F1 follows it (\cref{fig:dose}; Spearman $\rho = \RhogptFiveFour$ \RhogptFiveFourCI{} for GPT-5.4, \RhosonnetFourSix{} \RhosonnetFourSixCI{} for Claude Sonnet 4.6, \RhogemmaFour{} \RhogemmaFourCI{} for Gemma 4, \RhogeminiTwoFive{} \RhogeminiTwoFiveCI{} for Gemini 2.5, \RhoqwenThreevl{} \RhoqwenThreevlCI{} for Qwen3-VL, \RhogptFiveFive{} for GPT-5.5, \RhogptFiveSix{} for GPT-5.6, \RhosonnetFiveFive{} for Claude Sonnet 5.5, \RhoopusFiveFive{} for Claude Opus 5.5, and \RhogemThreeEight{} \RhogemThreeEightCI{} for Gemini 3.8, the only interval including zero). Tiles help where the ceiling is low (\GainLowgptFiveFour{} F1 for GPT-5.4) and not where it is high (\GainHighgptFiveFour{} \GainHighgptFiveFourCI). The ceiling is a function of $S$ and of how often instances share a token, and $\log S$ alone explains as much of the per-drawing variance ($R^2$ \XtTwoRSgptFiveFour{} against \XtTwoRCgptFiveFour{} for GPT-5.4, \XtTwoRSsonnetFourSix{} against \XtTwoRCsonnetFourSix{} for Claude Sonnet 4.6), so we treat $S$ as the operative variable and the ceiling as a reference line. On floor-plan mosaics, pooled over the three areas, per-item F1 again rises with $S$ for every model (Spearman \XfHEightgptFiveSix{} for GPT-5.6, \XfHEightsonnetFiveFive{} for Claude Sonnet 5.5, \XfHEightgemThreeEight{} for Gemini 3.8).

\subsection{Generations}
\label{sec:generations}

\begin{figure*}[t]
\centering
\includegraphics[width=0.95\textwidth]{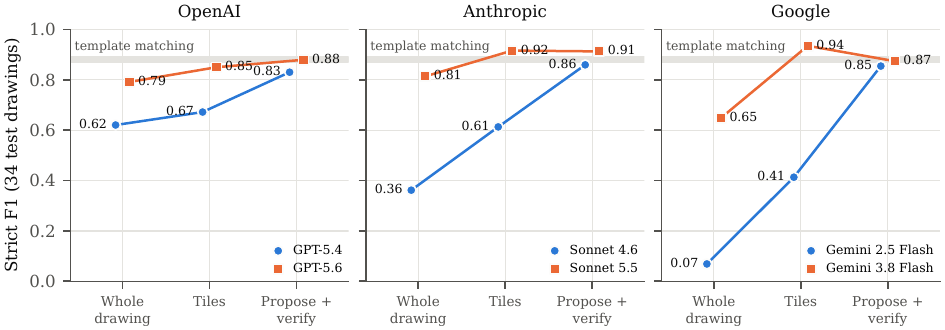}
\caption{Two generations per vendor under three configurations (F1, 34 LegendSpot test drawings; whole image read in each model's published convention). Propose and verify: template matching proposes candidates and the model accepts or rejects each one, shown as a crop at the reference's scale. Grey band: template matching with two threshold choices. F1 under one-to-one matching (axis: strict F1); point estimates.}
\label{fig:gen}
\end{figure*}

Larger budgets raise whole-image F1 where geometry says they should (\cref{fig:gen}): by \CgenclaudeATwosheets{} \CgenclaudeATwosheetsCI{} on full sheets from Claude Sonnet 4.6 to 5.5 (1{,}568 to 4{,}784 tokens; Sonnet 5.5 also thinks adaptively), and by \CgengeminiATwosheets{} \CgengeminiATwosheetsCI{} from Gemini 2.5 to 3.8 (258 to about 1{,}090 tokens). At unchanged geometry GPT-5.6 improves on GPT-5.4 by \CgengptFiveSixAOne{} \CgengptFiveSixAOneCI, while GPT-5.5 does not (\CgengptFiveFiveAOne{} \CgengptFiveFiveAOneCI): recognition also improves across generations, though not monotonically.

\subsection{Enlargement}
\label{sec:enlarge}

\cref{prop:geometry} predicts that a fixed-grid model gains from a smaller field of view but not from more pixels, while a resolution-scaled model can gain from enlargement as long as the pixels resolve the target. Choosing Gemini 2.5's and Gemma 4's field of view by \cref{prop:geometry}(c) instead of the 44 px recipe adds \CDMinusCgeminiTwoFive{} \CDMinusCgeminiTwoFiveCI{} and \CDMinusCgemmaFour{} \CDMinusCgemmaFourCI{} F1 on LegendSpot, and enlarging controlled images twofold leaves fixed-grid recall unchanged (Gemini 3.8 \XupDgemThreeEight{} \XupDCIgemThreeEight, Gemma 4 \XupDgemmaFour{} \XupDCIgemmaFour). For resolution-scaled models enlargement adds tokens but no Shannon information (\cref{prop:dpi}), so it helps only when the pixels resolve the target but the tokens do not, and whether it helps depends on the recognizer: at 8 px, twofold enlargement raised GPT-5.6's recall from \XupEightgptFiveSix{} and left Claude Sonnet 5.5's unchanged (\XupEightsonnetFiveFive). GPT-5.4's recall changed by \XtcViolD{} \XtcViolCI{} on the enlarged images, which its vendor shrinks again; this is one of the four violations (\cref{sec:violations}).

\subsection{The newest generation}
\label{sec:newest}

To compare the newest models we measured, for each, the $S$ at which it first reaches recall and precision of 0.9 ($S_{90}$), at the vendor's default reasoning setting (no reasoning parameter sent) and with an output cap of 16{,}384 tokens. Each controlled image holds 8 targets of one type and distractors of three other types, laid out in 12 base scenes; each scene is rendered at a grid of $S$ values by scaling the whole scene, without blank padding for the resolution-scaled models and centred on a fixed $1024^2$ canvas for the fixed-grid models. $S_{90}$ is the smallest tested $S$ at which the pooled recall and precision both reach 0.9 and stay at or above 0.9 at every larger tested $S$ (a single failure at the largest $S$ alone does not count); 90\% intervals resample the base scenes (2{,}000 draws). \Cref{tab:newest} gives the results; Qwen3-VL, run through a second provider, does not reach 0.9 at the largest $S$ tested (2.8). On the same 1{,}090-token grid, the $S$ that Gemini 4 needs is less than half of Gemini 3.8's.

\begin{table}[t]
\centering
\small
\setlength{\tabcolsep}{4pt}
\begin{adjustbox}{max width=\linewidth}\begin{tabular}{lcc}
\toprule
Model & $S_{90}$ & 90\% interval \\
\midrule
GPT-5.6 & 0.64 & [0.53, 0.79] \\
GPT-6.1 & 0.57 & [0.44, 0.60] \\
Claude Opus 5.5 & 0.63 & [0.51, 0.73] \\
Gemini 3.8 & \CgenSninetyGemThreeEight & [0.58, 1.63] \\
Gemini 4 & \CgenSninetyGemFour & [0.21, 0.38] \\
\bottomrule
\end{tabular}\end{adjustbox}
\caption{The $S$ at which recall and precision reach 0.9 on controlled images, at each vendor's default reasoning. The 90\% intervals resample the 12 base scenes (2{,}000 draws).}
\label{tab:newest}
\end{table}

\paragraph{Rendering.} The rendering pastes only template pixels darker than 200 after Lanczos downscaling, so the lightest strokes of a symbol are dropped at small scales; the load sweep of \cref{sec:thumb} keeps them as grey.

\section{Evidence on $L$}
\label{sec:L-evidence}

\subsection{More targets and more area at fixed $S$}
\label{sec:L-sweeps}

At fixed $S$, recall falls as one call must report more targets (\cref{fig:syn}b, \cref{tab:syn}). With 96 targets in one call it falls to \XcapNlastqwenThreevl{} for Qwen3-VL, \XcapNlastsonnetFourSix{} for Claude Sonnet 4.6, \XcapNlastgemmaFour{} for Gemma 4 and \XcapNlastgemThreeEight{} for Gemini 3.8. The slope against the number of targets is negative for every model whose $S$ in this sweep is above one, and negligible for GPT-5.6 and Claude Sonnet 5.5 on these well-separated targets (\XcapSlopegptFiveSix{} and \XcapSlopesonnetFiveFive{} per doubling; recall \XcapNlastgptFiveSix{} and \XcapNlastsonnetFiveFive{} with 96 targets). Within a vendor, what one call can handle rose across generations: from \XcapNlastsonnetFourSix{} to \XcapNlastsonnetFiveFive{} with 96 targets from Claude Sonnet 4.6 to 5.5, and from \XcapNlastgptFiveFour{} to \XcapNlastgptFiveSix{} from GPT-5.4 to GPT-5.6. Adding area at constant $S$ also costs recall: GPT-5.6, whose budget covers an 8.4-megapixel canvas without resizing, falls from \XareaFirstgptFiveSix{} to \XareaLastgptFiveSix{} with 16 targets (\cref{fig:syn}c).

\paragraph{Failures in answers that end normally.} In the target-count sweep, \CcapNormal{} of \CcapCalls{} calls ended normally. The exception is a Gemini 3.8 call with 96 targets whose thinking used 7{,}861 of the 8{,}192 output tokens allowed, although its thinking budget was set to the minimum; the longest visible answer (Gemini 2.5) used \CcapMaxOut{} tokens. A call that must report more targets fails in three ways that the answers show directly: it stops listing early, it places points less precisely, and it reports look-alike symbols.

\paragraph{Bits.} On the four 12 px images of the controlled test, the best model's whole-image answers carried on average at least about \XcapBestBits{} bits about the target positions by \cref{thm:capacity}, against at least \XcapBestFull{} bits needed for full recall without false points (recall \XcapBestRecall).

\subsection{Packing tiles into one call}
\label{sec:packing}

On the 14 LegendSpot full sheets we packed the tiles of the 44 px recipe $k$ at a time into one image, separated by grey lines in the tile margins, so that one call must search $k$ times the content (\cref{tab:packing}). Each single tile is a crop of the packed image (\cref{lem:crop}), so the theory predicts no gain from packing. For GPT-5.6 each symbol keeps its tokens, and recall falls clearly at $k = 8$, although a call packing eight tiles must report a median of only \CpackMedian{} targets (mean \CpackMean, at most \CpackMax). Claude Sonnet 5.5's packed images at $k = 4$ and 8 exceed its 4{,}784-token cap and are resized, so its drop also reflects a lower $S$; it is read with the per-answer frame rule of \cref{sec:coords} (with the documented frame alone, its F1 at $k = 8$ would be \XmosRegFsonnetFiveFivekEight).

\begin{table}[t]
\centering
\small
\setlength{\tabcolsep}{3pt}
\begin{adjustbox}{max width=\linewidth}\begin{tabular}{lcccc}
\toprule
 & $k = 1$ & 2 & 4 & 8 \\
\midrule
\multicolumn{5}{l}{\emph{GPT-5.6}} \\
Recall & \XmosRgptFiveSixkOne & \XmosRgptFiveSixkTwo & \XmosRgptFiveSixkFour & \XmosRgptFiveSixkEight \\
Precision & \XmosPgptFiveSixkOne & \XmosPgptFiveSixkTwo & \XmosPgptFiveSixkFour & \XmosPgptFiveSixkEight \\
F1 & \XmosFgptFiveSixkOne & \XmosFgptFiveSixkTwo & \XmosFgptFiveSixkFour & \XmosFgptFiveSixkEight \\
$\Delta$ recall & & \XmosdRgptFiveSixkTwo & \XmosdRgptFiveSixkFour & \textbf{\XmosdRgptFiveSixkEight} \\
\midrule
\multicolumn{5}{l}{\emph{Claude Sonnet 5.5}} \\
Recall & \XmosRsonnetFiveFivekOne & \XmosRsonnetFiveFivekTwo & \XmosRsonnetFiveFivekFour & \XmosRsonnetFiveFivekEight \\
Precision & \XmosPsonnetFiveFivekOne & \XmosPsonnetFiveFivekTwo & \XmosPsonnetFiveFivekFour & \XmosPsonnetFiveFivekEight \\
F1 & \XmosFsonnetFiveFivekOne & \XmosFsonnetFiveFivekTwo & \XmosFsonnetFiveFivekFour & \XmosFsonnetFiveFivekEight \\
$\Delta$ recall & & \textbf{\XmosdRsonnetFiveFivekTwo} & \textbf{\XmosdRsonnetFiveFivekFour} & \textbf{\XmosdRsonnetFiveFivekEight} \\
\bottomrule
\end{tabular}\end{adjustbox}
\caption{Packing $k$ tiles into one call on 14 full sheets. $\Delta$ recall: packed minus single tiles; bold: 95\% interval below zero (GPT-5.6, $k = 8$: \XmosdRCIgptFiveSixkEight; Claude Sonnet 5.5, $k = 2, 4, 8$: \XmosdRCIsonnetFiveFivekTwo, \XmosdRCIsonnetFiveFivekFour, \XmosdRCIsonnetFiveFivekEight).}
\label{tab:packing}
\end{table}

\subsection{Same targets, more content: the first block}
\label{sec:firstblock}

The floor-plan mosaics hold the same first block and the same reference at three areas, so scoring only the first block measures the effect of the content around it, which on the larger mosaics comes with a lower $S$ (\cref{tab:fpcsheets}). Whole-image F1 on the first block falls as the other blocks are added in the same call, for every model: by \XfBKDAOnegptFiveSix{} \XfBKDCIAOnegptFiveSix{} for GPT-5.6, \XfBKDAOnesonnetFiveFive{} \XfBKDCIAOnesonnetFiveFive{} for Claude Sonnet 5.5 and \XfBKDAOnegemThreeEight{} \XfBKDCIAOnegemThreeEight{} for Gemini 3.8, from one block to nine. Adding blocks also lowers $S$ (median \XfBKSmedGptFiveSix{} for GPT-5.6 on the largest mosaics), so we repeated the comparison on the mosaics where the largest still gives $S \ge 1$: the drop remains (GPT-5.6 \XfBKSoneGptFiveSix{} \XfBKSoneCIGptFiveSix{} on \XfBKSoneNGptFiveSix{} mosaics; Claude Sonnet 5.5 \XfBKSoneSonnetFiveFive{} \XfBKSoneCISonnetFiveFive{} on \XfBKSoneNSonnetFiveFive). GPT-5.6's tiles, whose content per call does not change, stay flat (GPT-5.6 \XfBKDDgptFiveSix{} \XfBKDCIDgptFiveSix), and so does GPT-5.6's zoom agent (\XfBKDAGgptFiveSix{} \XfBKDCIAGgptFiveSix). Scored over all blocks, every method declines with area, template matching included (0.54 to 0.18 at a fixed threshold), because the reference comes from the first block and 49\% of the instances in other blocks differ from it by more than 20\% in box size, against 14\% in the first block; scoring the first block removes this confound.

\begin{table*}[t]
\centering
\small
\setlength{\tabcolsep}{4pt}
\begin{adjustbox}{max width=\linewidth}\begin{tabular}{lccccccccc}
\toprule
 & \multicolumn{3}{c}{Whole image} & \multicolumn{3}{c}{Tiles} & \multicolumn{3}{c}{Zoom agent} \\
\cmidrule(lr){2-4}\cmidrule(lr){5-7}\cmidrule(lr){8-10}
Model & $k{=}1$ & 2 & 3 & 1 & 2 & 3 & 1 & 2 & 3 \\
\midrule
\multicolumn{10}{l}{\emph{First block only}} \\
GPT-5.4 & 0.61 & 0.51 & 0.46 & 0.36 & 0.38 & 0.37 &  &  &  \\
Claude Sonnet 4.6 & 0.46 & 0.20 & 0.11 & 0.29 & 0.33 & 0.34 &  &  &  \\
Gemini 2.5 Flash & 0.18 & 0.11 & 0.02 & 0.30 & 0.21 & 0.22 &  &  &  \\
Gemma 4 31B & 0.38 & 0.09 & 0.02 & 0.57 & 0.52 & 0.49 &  &  &  \\
Qwen3-VL & 0.45 & 0.27 & 0.08 & 0.34 & 0.24 & 0.28 &  &  &  \\
\midrule
GPT-5.6 & 0.73 & 0.68 & 0.52 & 0.60 & 0.57 & 0.59 & 0.81 & 0.83 & 0.80 \\
Claude Sonnet 5.5 & 0.78 & 0.75 & 0.62 & 0.76 & 0.74 & 0.69 & 0.84 & 0.79 & 0.74 \\
Gemini 3.8 (default) & 0.81 & 0.66 & 0.47 & 0.81 & 0.79 & 0.78 &  &  &  \\
Gemini 3.8 (low) & 0.76 & 0.41 & 0.17 &  &  &  &  &  &  \\
Gemini 3.8 (ultra-high) & 0.83 & 0.64 & 0.46 &  &  &  &  &  &  \\
\midrule
Template matching & \multicolumn{9}{c}{0.54 / 0.54 / 0.54} \\
 
\midrule
\multicolumn{10}{l}{\emph{All blocks}} \\
GPT-5.4 & 0.61 & 0.40 & 0.25 & 0.36 & 0.32 & 0.26 &  &  &  \\
Claude Sonnet 4.6 & 0.46 & 0.19 & 0.09 & 0.29 & 0.24 & 0.24 &  &  &  \\
Gemini 2.5 Flash & 0.18 & 0.07 & 0.02 & 0.30 & 0.18 & 0.19 &  &  &  \\
Gemma 4 31B & 0.38 & 0.06 & 0.01 & 0.57 & 0.42 & 0.32 &  &  &  \\
Qwen3-VL & 0.45 & 0.21 & 0.06 & 0.34 & 0.24 & 0.25 &  &  &  \\
\midrule
GPT-5.6 & 0.73 & 0.54 & 0.39 & 0.60 & 0.44 & 0.39 & 0.81 & 0.73 & 0.59 \\
Claude Sonnet 5.5 & 0.78 & 0.67 & 0.50 & 0.76 & 0.61 & 0.52 & 0.84 & 0.67 & 0.55 \\
Gemini 3.8 (default) & 0.81 & 0.50 & 0.32 & 0.81 & 0.68 & 0.63 &  &  &  \\
Gemini 3.8 (low) & 0.76 & 0.34 & 0.12 &  &  &  &  &  &  \\
Gemini 3.8 (ultra-high) & 0.83 & 0.52 & 0.34 &  &  &  &  &  &  \\
\midrule
Template matching & \multicolumn{9}{c}{0.54 / 0.31 / 0.18} \\
 
\bottomrule
\end{tabular}\end{adjustbox}
\caption{Floor-plan mosaics of $k \times k$ blocks ($1{,}000^2$ to $3{,}000^2$ px), F1, 60 items per area. Columns: $k = 1, 2, 3$. First block only: F1 on the first block; all blocks: F1 on the whole mosaic. Tiles: the 44 px recipe and its fixed-grid counterpart. Point estimates; empty cells: not run.}
\label{tab:fpcsheets}
\end{table*}

\subsection{Error types and tolerance}
\label{sec:errors}

\begin{table}[t]
\centering
\footnotesize
\setlength{\tabcolsep}{2.5pt}
\resizebox{\columnwidth}{!}{\begin{tabular}{llcccc}
\toprule
Model & Condition & miss & dup. & displ. & spur. \\
\midrule
GPT-5.4 & tiles & 0.23 & 0.02 & 0.13 & 0.40 \\
GPT-5.4 & propose + verify & 0.22 & 0.02 & 0.00 & 0.13 \\
Claude Sonnet 4.6 & tiles & 0.30 & 0.04 & 0.15 & 0.52 \\
Claude Sonnet 4.6 & propose + verify & 0.13 & 0.02 & 0.00 & 0.16 \\
Gemini 2.5 Flash & tiles & 0.58 & 0.02 & 0.25 & 0.43 \\
Gemini 2.5 Flash & propose + verify & 0.12 & 0.03 & 0.00 & 0.22 \\
Gemma 4 & tiles & 0.06 & 0.01 & 0.00 & 0.88 \\
Gemma 4 & propose + verify & 0.11 & 0.03 & 0.00 & 0.25 \\
Qwen3-VL & tiles & 0.38 & 0.02 & 0.05 & 0.25 \\
Qwen3-VL & propose + verify & 0.12 & 0.02 & 0.00 & 0.04 \\
GPT-5.6 & tiles & 0.15 & 0.01 & 0.00 & 0.15 \\
GPT-5.6 & propose + verify & 0.10 & 0.03 & 0.00 & 0.35 \\
Claude Sonnet 5.5 & tiles & 0.03 & 0.00 & 0.00 & 0.11 \\
Claude Sonnet 5.5 & propose + verify & 0.07 & 0.01 & 0.00 & 0.08 \\
Gemini 3.8 Flash & tiles & 0.09 & 0.00 & 0.00 & 0.06 \\
Gemini 3.8 Flash & propose + verify & 0.21 & 0.01 & 0.00 & 0.02 \\
GPT-5.6 (original) & whole, native & 0.35 & 0.00 & 0.04 & 0.24 \\
 
\bottomrule
\end{tabular}}
\caption{Errors per annotated instance on the LegendSpot test drawings (one run), pooled over drawings: missed instances, duplicate reports of a found instance, reports within $2\tau$ of a missed instance, and spurious reports elsewhere.}
\label{tab:tax}
\end{table}

\begin{table}[t]
\centering
\footnotesize
\setlength{\tabcolsep}{3pt}
\begin{adjustbox}{max width=\linewidth}\begin{tabular}{lccc}
\toprule
Model, condition & $\tau$ & ink box & $2\tau$ \\
\midrule
GPT-5.4, whole drawing & 0.62 & 0.61 & 0.67 \\
Claude Sonnet 4.6, whole drawing & 0.36 & 0.33 & 0.50 \\
GPT-5.4, tiles & 0.67 & 0.67 & 0.80 \\
Gemma 4, tiles & 0.75 & 0.75 & 0.76 \\
Qwen3-VL, propose + verify & 0.91 & 0.91 & 0.91 \\
Claude Sonnet 5.5, tiles & 0.92 & 0.92 & 0.92 \\
Gemini 3.8 Flash, tiles & 0.94 & 0.93 & 0.94 \\
Claude Sonnet 5.5, whole drawing & 0.81 & 0.81 & 0.85 \\
tuning-set threshold, template matching & 0.87 & 0.87 & 0.87 \\
 
\bottomrule
\end{tabular}\end{adjustbox}
\caption{F1 on LegendSpot under alternative matching tolerances: the primary $\tau = m/2$, a tolerance from the template's ink box, and $2\tau$.}
\label{tab:sens}
\end{table}

\Cref{tab:tax} splits the errors. Propose and verify removes nearly all displaced reports and, for most models, most spurious ones; GPT-5.6 is the exception, with more spurious reports as a verifier (0.35) than with tiles (0.15). GPT-5.6's whole-image calls at native resolution miss 0.35 of the instances and add 0.24 spurious reports per instance, so their counts are closer to the truth than their positions. Under the alternative tolerances of \cref{tab:sens} the comparisons keep their direction, with one exception: at $2\tau$ GPT-5.4's tiles (0.80) overtake Gemma 4's (0.76), consistent with less precise placement by GPT-5.4. Claude Sonnet 4.6's whole-image score also rises (0.36 to 0.50), as expected for a model that localizes coarsely at about half a token per target.

\section{Tests of the assumptions}
\label{sec:assume}

\subsection{The controlled test}
\label{sec:assume-direct}

Controlled images let each relation of \cref{sec:defs} be built exactly. Each of twelve $2048^2$ images holds 48 targets and 48 distractors at a target scale of 12, 20 or 32 px. For each image, each model received five kinds of call: the whole image; sixteen $512^2$ crops at the same effective scale (padded with blank canvas for fixed-grid interfaces, so that the vendor applies the same resizing); the same crops at twice the scale; the same crops at twice the scale but cut from a copy downsampled fourfold (a garbling); and a $3 \times 3$ tiling whose cores overlap by $2\tau$, which satisfies the covering hypothesis of \cref{thm:safe} at the whole image's scale, the hardest case the theorem allows. Calls to resolution-scaled models were sized so that the vendor does not resize them. Over \XasModels{} models and five relations (\cref{tab:assume}), none of the \XasN{} orderings is violated: \XasSup{} hold significantly and \XasCons{} are consistent, the smallest mean difference being \XasMin. The safe tiling gains \XasThmMin{} to \XasThmMax{} in hit recall at equal $S$, significantly for \XasThmSup{} of \XasModels{} models; the tiles are shown at the whole image's scale, so this gain comes from less content per call alone. The two newest models, Claude Opus 5.5 and GPT-6.1, already find most targets in the whole image (whole-image recall \XasWholeRecOpusFiveFive{} and \XasWholeRecGptSixOne, against \XasWholeRecGptFiveSix{} for GPT-5.6), and \XasNewSup{} of their \XasNewN{} orderings hold significantly. No stratum by target scale has an interval below zero.

\begin{table*}[t]
\centering
\footnotesize
\setlength{\tabcolsep}{4pt}
\resizebox{\textwidth}{!}{\begin{tabular}{lccccc}
\toprule
 & (M) crop $-$ whole & (R) $2\times$ crop $-$ crop & (I) native $-$ garbled & (R, M) $2\times$ crop $-$ whole & \cref{thm:safe}: safe tiles $-$ whole \\
\midrule
GPT-5.4 & +0.12{\scriptsize\,[$-$0.06, +0.32]} & \textbf{+0.16}{\scriptsize\,[+0.01, +0.29]} & \textbf{+0.43}{\scriptsize\,[+0.28, +0.56]} & \textbf{+0.28}{\scriptsize\,[+0.13, +0.45]} & \textbf{+0.25}{\scriptsize\,[+0.10, +0.40]} \\
GPT-5.6 & \textbf{+0.39}{\scriptsize\,[+0.16, +0.65]} & +0.03{\scriptsize\,[0.00, +0.07]} & \textbf{+0.28}{\scriptsize\,[+0.16, +0.41]} & \textbf{+0.42}{\scriptsize\,[+0.17, +0.69]} & \textbf{+0.37}{\scriptsize\,[+0.15, +0.60]} \\
Claude Sonnet 5.5 & +0.17{\scriptsize\,[$-$0.01, +0.42]} & $-$0.03{\scriptsize\,[$-$0.07, +0.01]} & \textbf{+0.49}{\scriptsize\,[+0.26, +0.72]} & +0.14{\scriptsize\,[$-$0.05, +0.38]} & +0.09{\scriptsize\,[$-$0.08, +0.31]} \\
Qwen3-VL & +0.003{\scriptsize\,[$-$0.15, +0.14]} & \textbf{+0.46}{\scriptsize\,[+0.29, +0.62]} & \textbf{+0.58}{\scriptsize\,[+0.36, +0.79]} & \textbf{+0.46}{\scriptsize\,[+0.25, +0.67]} & \textbf{+0.23}{\scriptsize\,[+0.09, +0.40]} \\
Gemma 4 & \textbf{+0.10}{\scriptsize\,[+0.02, +0.20]} & \textbf{+0.32}{\scriptsize\,[+0.19, +0.47]} & \textbf{+0.19}{\scriptsize\,[+0.10, +0.29]} & \textbf{+0.42}{\scriptsize\,[+0.24, +0.61]} & \textbf{+0.07}{\scriptsize\,[+0.02, +0.13]} \\
Gemini 3.8 & \textbf{+0.28}{\scriptsize\,[+0.10, +0.47]} & \textbf{+0.20}{\scriptsize\,[+0.07, +0.34]} & \textbf{+0.40}{\scriptsize\,[+0.20, +0.62]} & \textbf{+0.48}{\scriptsize\,[+0.33, +0.64]} & +0.16{\scriptsize\,[$-$0.03, +0.34]} \\
\midrule
Claude Opus 5.5$^\dagger$ & +0.26{\scriptsize\,[0.00, +0.53]} & $-$0.003{\scriptsize\,[$-$0.02, +0.01]} & \textbf{+0.39}{\scriptsize\,[+0.19, +0.60]} & +0.26{\scriptsize\,[0.00, +0.52]} & +0.22{\scriptsize\,[0.00, +0.44]} \\
GPT-6.1 Sol$^\dagger$ & +0.03{\scriptsize\,[$-$0.12, +0.24]} & +0.05{\scriptsize\,[0.00, +0.11]} & \textbf{+0.38}{\scriptsize\,[+0.20, +0.58]} & +0.08{\scriptsize\,[$-$0.03, +0.26]} & +0.06{\scriptsize\,[$-$0.09, +0.26]} \\
 
\bottomrule
\end{tabular}}
\caption{The controlled test of \cref{ass:mono}, \cref{ass:info} and \cref{thm:safe}. Each cell is the mean difference in hit recall over targets interior to both calls, with a 95\% bootstrap interval over the twelve images; bold: interval above zero. Every pair of calls satisfies the definitions of \cref{sec:defs} exactly. $^\dagger$: the two newest models.}
\label{tab:assume}
\end{table*}

\paragraph{Unmatched points.} The safe tiling, keeping every point in the overlapping cores, also raised F1 for all \XfoFOneN{} models, by \XfoFOneMin{} to \XfoFOneMax, significantly for \XfoFOneSup. Points that match no target do not behave monotonically: crops at equal scale returned more such points than the whole image for \XfoCropFalse{} of \XfoModels{} models. What decides F1 is the rate of \cref{prop:f1}: \XfoFewer{} models added fewer unmatched points than hits, and Gemma 4 added \XfoGemmaRate{} per additional hit, within its affordable rate $\rho_{\mathrm{w}} = \XfoGemmaRho$, because its whole-image recall was \XfoGemmaRecall\%.

\subsection{Every ordering implied in the other experiments}
\label{sec:assume-orderings}

\begin{figure*}[p]
\centering
\includegraphics[width=\textwidth]{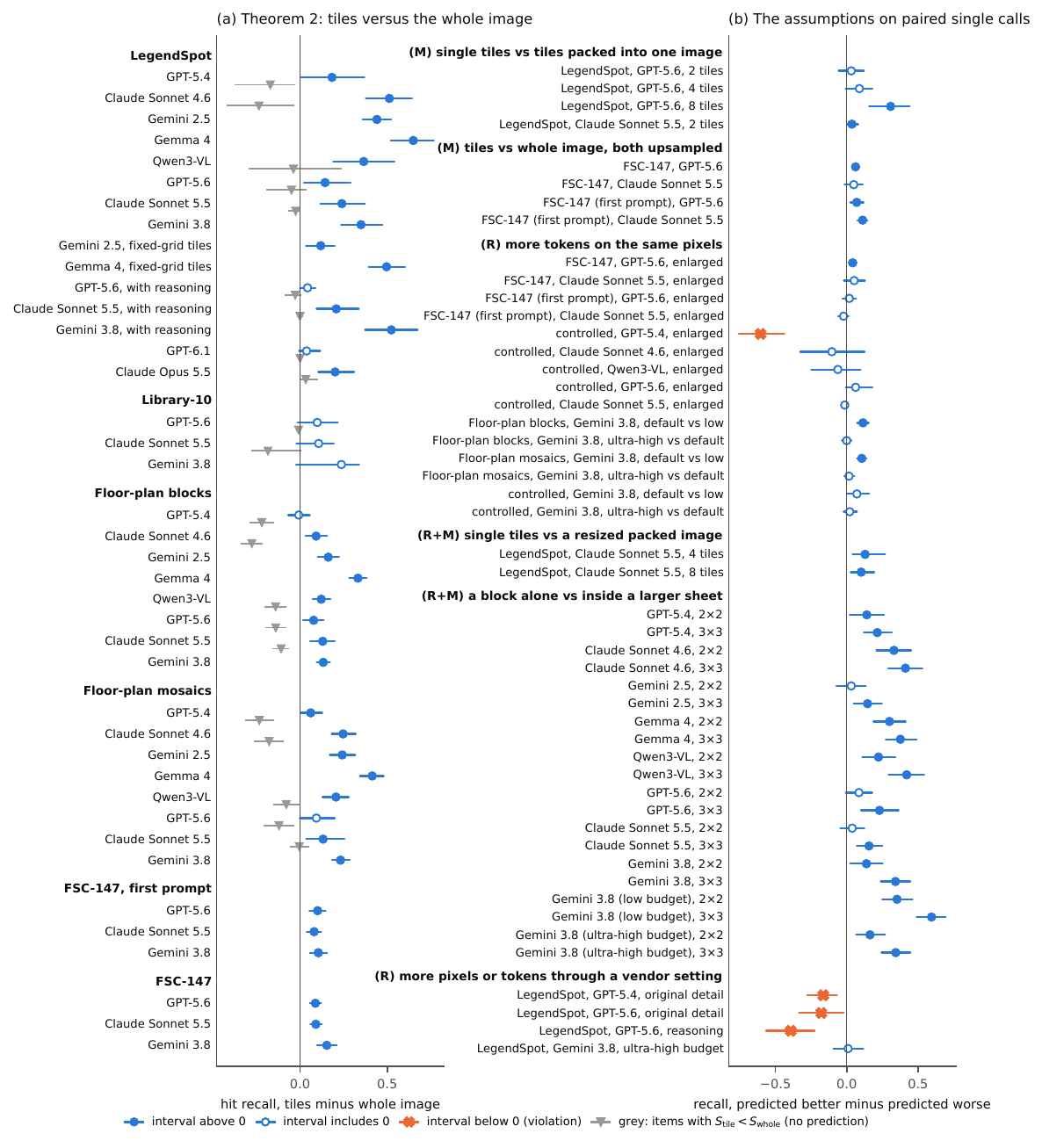}
\caption{Every pair of single calls that \cref{ass:mono} orders in the data of our other experiments, grouped by the relation it tests. Each row is a paired difference, predicted better minus predicted worse, with a 95\% bootstrap interval, clustered by source image where items share one. (a) \cref{thm:safe}: hit recall over targets interior to the whole-image call and to a tile, on items where every tile has at least the whole image's $S$; grey, for comparison, the items where the 44 px recipe zooms out, where the theory makes no prediction (\cref{sec:zoomout}). (b) Crops at equal scale (M); more tokens on the same pixels (R); crops at a larger scale, such as a block alone against the same block inside a larger mosaic sent at a lower scale (R and M); and the same image with more pixels or tokens through a vendor setting (R, with new information where native pixels are sent). Library-10: the drawings with library references.}
\label{fig:theory}
\end{figure*}

We searched the data of all other experiments for every pair of single calls that the assumptions order, on the same targets and with the same prompt, and report all \XtcN{} that we found (\cref{fig:theory}). A pair counts as a violation if its 95\% interval lies below zero. \XtcSup{} pairs lie significantly in the predicted direction, \XtcCons{} are consistent with it, and \XtcViol{} violate it.

\paragraph{Splitting (\cref{fig:theory}a).} We tested \cref{thm:safe} as stated: on hit recall, over the targets interior to the whole-image call and to a tile, on the items where every tile has at least the whole image's $S$. Our tiles' cores abut without the overlap of the covering hypothesis, so this tests the theorem's first claim; the interior targets are \XtcInteriorPct\% of all targets on those items. In all \XtcThmN{} model and data settings, which include tiles and whole images with reasoning, fixed-grid tiles, the newest models and Library-10, the difference is at least \XtcThmMin, and in \XtcThmSup{} the interval lies above zero.

\paragraph{The assumptions on paired calls (\cref{fig:theory}b).} The natural-image crop pairs are scored on all targets rather than on those interior to both calls, which counts targets cut by tile borders against the tiles and so makes these tests conservative. Packing more tiles into one image never raised recall. More tokens on the same pixels (Gemini 3.8's three budgets, and enlarged images for the resolution-scaled models) agree with (R) on standard input. No floor-plan block was detected better inside a larger mosaic, sent at a lower scale, than alone, for any of the ten model settings. The zoom-agent comparisons (\cref{sec:agents}) are analogues of \cref{prop:zoom} at the level of a policy rather than of single calls, and are not counted.

\subsection{The four violations}
\label{sec:violations}

\begin{figure}[t]
\centering
\includegraphics[width=\columnwidth]{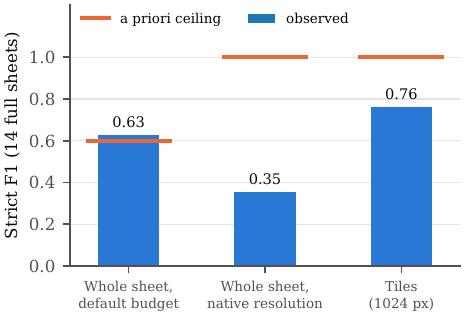}
\caption{GPT-5.6 without reasoning on the 14 full sheets: F1 of the whole sheet at a long side of 1{,}600 px, the same sheet at native resolution through the high-resolution mode, and tiles; bars, observed; orange lines, the ceiling computed from the geometry. More pixels of the same content lower F1.}
\label{fig:capacity}
\end{figure}

All \XtcViol{} violations involve OpenAI models given input outside our standard input (\cref{sec:models}). GPT-5.4's recall changed by \XtcViolD{} \XtcViolCI{} on controlled images that we enlarged to 2{,}048 px, which its vendor then shrinks, although the measured geometry gives them more tokens per target than the originals. Sending a LegendSpot drawing at native resolution through \texttt{detail=original} instead of at 1{,}600 px changed recall by \XtcVendGptFiveFour{} \XtcVendCIGptFiveFour{} for GPT-5.4, \XtcVendGptFiveSix{} \XtcVendCIGptFiveSix{} for GPT-5.6, and \XtcVendGptFiveSixR{} \XtcVendCIGptFiveSixR{} for GPT-5.6 with reasoning, whose F1 on full sheets fell to \XnFNATRgptFiveSixSheets{} from \XnFAOneRgptFiveSixSheets{} (\cref{fig:capacity} shows the same contrast without reasoning: \CcapnativeMinusautosheets{} \CcapnativeMinusautosheetsCI). These comparisons use the \XtcVendNGptFiveFour{} drawings on which the mode adds pixels; on small views the two calls send the same image. None of these losses is a coordinate error: no rescaling or shift of the answers recovers them, even fitted to the labels, and in the high-resolution mode GPT-5.6 often reports about the right number of symbols at wrong positions. Gemini 3.8's ultra-high budget on the same drawings changed recall by \XtcVendGemThreeEight{} \XtcVendCIGemThreeEight. Each violating pair shows the same content with more pixels, so the losses are failures of monotonicity on these input paths, not effects of content; a caller who uses such a path should first test (R) on a few pairs of calls, as \cref{sec:safe} recommends.

\section{Splits that zoom out}
\label{sec:zoomout}

The 44 px recipe is defined in pixels: it rescales the image so that the reference spans 44 px and then tiles. For a resolution-scaled model with 32 px tokens its tiles give $S$ of about 1.4 whatever the image (44/32 = 1.375 across the reference's shorter side), so it zooms out whenever the whole-image call already gives a larger $S$, which happens on small images with large targets. Which items these are depends on the image's size and on the model, so it is hard to see without computing $S$ (\cref{sec:computeS}). \Cref{tab:zoomout} lists every model and data cell in which the recipe zooms out on some items, the case that \cref{thm:exact}(b) and \cref{prop:tight}(a) leave unprotected. The point estimate is negative in \XtcOutNegPoint{} of the \XtcOutN{} cells and the interval lies below zero in \XtcOutNeg; the two non-negative cells are the newest models on LegendSpot, which find nearly every target in both calls. The data show that this protection is needed. \Cref{fig:map}(b) shows the same contrast for all models and data together, binned by $S$ and targets: every cell there is negative.

\begin{table}[t]
\centering
\footnotesize
\setlength{\tabcolsep}{3pt}
\resizebox{\columnwidth}{!}{\begin{tabular}{llccc}
\toprule
Data & Model & Items & Recall change & 95\% interval \\
\midrule
LegendSpot & GPT-5.4 & 9 & \textbf{$-$0.17} & [$-$0.37, $-$0.03] \\
LegendSpot & Claude Sonnet 4.6 & 7 & \textbf{$-$0.24} & [$-$0.42, $-$0.04] \\
LegendSpot & Qwen3-VL & 9 & $-$0.04 & [$-$0.29, +0.24] \\
LegendSpot & GPT-5.6 & 9 & $-$0.05 & [$-$0.19, +0.03] \\
LegendSpot & Claude Sonnet 5.5 & 9 & \textbf{$-$0.03} & [$-$0.06, $-$0.002] \\
LegendSpot & GPT-5.6, with reasoning & 9 & $-$0.03 & [$-$0.08, 0.00] \\
LegendSpot & Claude Sonnet 5.5, with reasoning & 9 & $-$0.001 & [$-$0.01, +0.01] \\
LegendSpot & GPT-6.1 & 9 & 0.00 & [0.00, 0.00] \\
LegendSpot & Claude Opus 5.5 & 9 & +0.03 & [0.00, +0.10] \\
Library-10 & GPT-5.6 & 6 & $-$0.01 & [$-$0.01, 0.00] \\
Library-10 & Claude Sonnet 5.5 & 6 & $-$0.18 & [$-$0.28, 0.00] \\
Floor-plan blocks & GPT-5.4 & 171 & \textbf{$-$0.22} & [$-$0.29, $-$0.16] \\
Floor-plan blocks & Claude Sonnet 4.6 & 171 & \textbf{$-$0.28} & [$-$0.33, $-$0.22] \\
Floor-plan blocks & Qwen3-VL & 164 & \textbf{$-$0.14} & [$-$0.20, $-$0.08] \\
Floor-plan blocks & GPT-5.6 & 171 & \textbf{$-$0.14} & [$-$0.20, $-$0.08] \\
Floor-plan blocks & Claude Sonnet 5.5 & 171 & \textbf{$-$0.11} & [$-$0.16, $-$0.07] \\
Floor-plan mosaics & GPT-5.4 & 119 & \textbf{$-$0.23} & [$-$0.31, $-$0.15] \\
Floor-plan mosaics & Claude Sonnet 4.6 & 82 & \textbf{$-$0.18} & [$-$0.26, $-$0.10] \\
Floor-plan mosaics & Qwen3-VL & 115 & \textbf{$-$0.08} & [$-$0.15, $-$0.01] \\
Floor-plan mosaics & GPT-5.6 & 119 & \textbf{$-$0.12} & [$-$0.21, $-$0.04] \\
Floor-plan mosaics & Claude Sonnet 5.5 & 119 & $-$0.01 & [$-$0.05, +0.05] \\
 
\bottomrule
\end{tabular}}
\caption{Hit recall, tiles of the 44 px recipe minus the whole image, over targets interior to both, on the items where some tile has a smaller $S$ than the whole image. Bold: interval below zero. Items: drawing and class items (mosaics: items at all three areas).}
\label{tab:zoomout}
\end{table}

\paragraph{Where the recipe zooms out and where it does not.} On floor-plan blocks, whose $1{,}000^2$ rasters give a large $S$ in the whole-image call, the recipe lowers $S$ by more than 5\% on \XfSrFracDowngptFiveSix\% of the items for the resolution-scaled models. On the same data, GPT-5.6 gains \XtcIngptFiveSix{} \XtcCIIngptFiveSix{} on the items where no tile zooms out and changes by \XtcOutgptFiveSix{} \XtcCIOutgptFiveSix{} on the others; Claude Sonnet 5.5 gains \XtcInsonnetFiveFive{} \XtcCIInsonnetFiveFive{} and changes by \XtcOutsonnetFiveFive{} \XtcCIOutsonnetFiveFive. The rule never zooms out, because it chooses $S = \max(S^\ast, \Sw)$.

\subsection{Common recipes under the condition}
\label{sec:recipes}

The condition of \cref{thm:exact} can be checked for any splitting recipe from geometry alone, before any call. \Cref{tab:recipes} applies it to two families on our images: slicing at native resolution with a fixed overlap ratio, as SAHI does~\cite{akyon2022sahi}, and splitting into a regular grid of cells without overlap, the first steps of quadrant splitting and quadrant search~\cite{wang2025dc2,wu2024vstar,shen2024zoomeye}. Each view is sent without further resizing by the caller (SAHI) or as standard input (cells), and its $S$ follows from the interface (\cref{sec:computeS}).

Neither family zooms out on any image: a native-resolution slice is never shown at a smaller scale than the whole image, and a cell of the image is never shown at a smaller scale than the image itself. The condition fails in the other two ways. Native-resolution slices keep each target at its native pixel size, so with 32 px tokens a target narrower than about 48 px stays below $S^\ast = 1.5$ however the image is split: this holds for 38\% of the LegendSpot drawings. Splits without overlap leave every target that a cut crosses in no view's interior, 9 to 19\% of the targets for a $2 \times 2$ split and 29 to 49\% for a $4 \times 4$ split, and \cref{thm:exact}(b) shows that nothing then protects them; a fixed overlap ratio leaves up to 15\% of the targets of the floor-plan mosaics cut with 512 px slices, where it is smaller than the targets plus their margins. On fixed-grid models, small native-resolution slices also cost many times the rule's tokens, since every slice pays the full grid. The rule avoids all three by choosing the scale from $S$ and the overlap from the target.

\begin{table*}[t]
\centering
\small
\setlength{\tabcolsep}{4pt}
\begin{adjustbox}{max width=\linewidth}\begin{tabular}{llccccc}
\toprule
Data & Recipe & Zoom-out: & Reaches $S^\ast$? & Protection criterion: & Tokens / rule & Tokens / rule \\
 & & images & images below 1.5 & targets unprotected & (GPT-6.1) & (Gemini) \\
\midrule
LegendSpot & SAHI, 512 px, 20\% & 0\% & 38\% & 6.6\% & 0.57 & 5.62 \\
 & SAHI, 1{,}024 px, 20\% & 0\% & 38\% & 0.4\% & 0.63 & 1.67 \\
 & $2 \times 2$, no overlap & 0\% & 52\% & 13.1\% & 0.25 & 0.67 \\
 & $4 \times 4$, no overlap & 0\% & 40\% & 32.6\% & 0.41 & 2.67 \\
\midrule
Floor-plan mosaics & SAHI, 512 px, 20\% & 0\% & 10\% & 15.3\% & 1.01 & 16.00 \\
 & SAHI, 1{,}024 px, 20\% & 0\% & 10\% & 4.9\% & 1.01 & 4.00 \\
 & $2 \times 2$, no overlap & 0\% & 10\% & 9.4\% & 0.55 & 1.00 \\
 & $4 \times 4$, no overlap & 0\% & 10\% & 28.7\% & 0.57 & 4.00 \\
\midrule
Floor-plan blocks & SAHI, 512 px, 20\% & 0\% & 30\% & 5.7\% & 0.28 & 9.00 \\
 & SAHI, 1{,}024 px, 20\% & 0\% & 30\% & 0.0\% & 0.13 & 1.00 \\
 & $2 \times 2$, no overlap & 0\% & 30\% & 18.9\% & 0.13 & 4.00 \\
 & $4 \times 4$, no overlap & 0\% & 30\% & 48.5\% & 0.13 & 16.00 \\
 
\bottomrule
\end{tabular}\end{adjustbox}
\caption{Common splitting recipes checked against the condition of \cref{thm:exact} from geometry alone (no model call). SAHI: square slices at native resolution with an overlap of 20\% of the slice side. Cells: a regular grid without overlap, each cell sent as standard input. Zoom-out: maximum over the four interfaces of \cref{tab:thumb-budget} (GPT-6.1, Claude Opus 5.5, Qwen3-VL, Gemini); below $S^\ast$: GPT-6.1's interface; protection criterion: share of the targets interior to the whole image that lie in no qualified view's interior (\cref{thm:exact}), the same for every interface; tokens: median over images of the recipe's image tokens over the rule's.}
\label{tab:recipes}
\end{table*}

\section{Robustness of \cref{fig:map}}
\label{sec:map}

\Cref{fig:map} pairs every whole-image call with every split of the same item by the same model in our data: all 40 LegendSpot drawings and all floor-plan items (including those used to write the prompt and set thresholds) with the tiles of the 44 px recipe and its fixed-grid counterpart, and the natural images with same-scale $2 \times 2$ tiles against the upsampled whole image. The models are the eight sweep models of \cref{tab:experiments}, with Gemini 3.8 at its three budgets. A pair is a split that never zooms out if every tile has at least the whole-image call's $S$, and otherwise one that zooms out; the grouping uses geometry only. Each pair is binned by the whole-image call's $S$ and number of targets, and a cell with fewer than 10 pairs is left blank. Intervals resample images, and source drawings for the mosaics, so that pairs of one image across models and areas move together. \Cref{tab:mapci} gives every cell with its interval; \cref{fig:map-models} and \cref{fig:map-data} split the map by model and by data set, and \cref{fig:map-cxcy} recomputes it reading only points written with the requested \texttt{cx} and \texttt{cy} keys (\cref{sec:coords}). For $S$ between 1.5 and 3, the gain of \cref{fig:map}(a) pooled over targets is +0.13 for Gemini 3.8, +0.14 for Gemma 4, +0.13 for Qwen3-VL, +0.09 for Claude Sonnet 4.6 and +0.08 for Claude Sonnet 5.5, against $-0.02$ for GPT-5.4, +0.03 for GPT-5.6 and +0.03 for Gemini 2.5. In every version the conclusions of \cref{sec:safe} hold: cells with $S < 1.5$ gain most, and the cells of splits that zoom out are negative. Of the 91 cells of \cref{fig:map-models}, none is significantly below zero; the lowest are Claude Sonnet 5.5's at $S \ge 3$ ($-0.14$ and $-0.12$).

\begin{table}[t]
\centering
\footnotesize
\setlength{\tabcolsep}{2.5pt}
\begin{adjustbox}{max width=\linewidth}\begin{tabular}{lcccrc}
\toprule
Split & Targets & $S$ & Pairs & Mean & 95\% interval \\
\midrule
never zoom out & $\le$3 & $<$0.5 & 50 & +0.26 & [+0.15, +0.38] \\
never zoom out & $\le$3 & [0.5, 1) & 299 & +0.20 & [+0.15, +0.25] \\
never zoom out & $\le$3 & [1, 1.5) & 351 & +0.12 & [+0.08, +0.17] \\
never zoom out & $\le$3 & [1.5, 3) & 523 & +0.07 & [+0.04, +0.11] \\
never zoom out & $\le$3 & $\ge$3 & 109 & $-$0.01 & [$-$0.07, +0.05] \\
never zoom out & 4 to 15 & $<$0.5 & 74 & +0.37 & [+0.29, +0.45] \\
never zoom out & 4 to 15 & [0.5, 1) & 364 & +0.31 & [+0.27, +0.35] \\
never zoom out & 4 to 15 & [1, 1.5) & 438 & +0.10 & [+0.07, +0.14] \\
never zoom out & 4 to 15 & [1.5, 3) & 467 & +0.09 & [+0.06, +0.11] \\
never zoom out & 4 to 15 & $\ge$3 & 123 & +0.04 & [$-$0.00, +0.09] \\
never zoom out & 16 to 31 & $<$0.5 & 65 & +0.46 & [+0.35, +0.56] \\
never zoom out & 16 to 31 & [0.5, 1) & 111 & +0.22 & [+0.17, +0.28] \\
never zoom out & 16 to 31 & [1, 1.5) & 92 & +0.16 & [+0.09, +0.25] \\
never zoom out & 16 to 31 & [1.5, 3) & 49 & +0.08 & [$-$0.01, +0.17] \\
never zoom out & 16 to 31 & $\ge$3 & 62 & $-$0.04 & [$-$0.09, +0.00] \\
never zoom out & 32 to 95 & $<$0.5 & 17 & +0.45 & [+0.26, +0.66] \\
never zoom out & 32 to 95 & [0.5, 1) & 40 & +0.43 & [+0.25, +0.59] \\
never zoom out & 32 to 95 & [1, 1.5) & 26 & +0.23 & [+0.09, +0.35] \\
never zoom out & 32 to 95 & [1.5, 3) & 46 & +0.17 & [+0.07, +0.27] \\
never zoom out & 32 to 95 & $\ge$3 & 90 & +0.10 & [+0.03, +0.17] \\
never zoom out & $\ge$96 & [0.5, 1) & 33 & +0.14 & [+0.06, +0.23] \\
never zoom out & $\ge$96 & [1, 1.5) & 42 & +0.19 & [+0.10, +0.29] \\
never zoom out & $\ge$96 & [1.5, 3) & 50 & +0.16 & [+0.08, +0.26] \\
never zoom out & $\ge$96 & $\ge$3 & 26 & $-$0.02 & [$-$0.14, +0.10] \\
zoom out & $\le$3 & [1.5, 3) & 288 & $-$0.14 & [$-$0.18, $-$0.10] \\
zoom out & $\le$3 & $\ge$3 & 387 & $-$0.22 & [$-$0.27, $-$0.16] \\
zoom out & 4 to 15 & [1, 1.5) & 27 & $-$0.01 & [$-$0.06, +0.04] \\
zoom out & 4 to 15 & [1.5, 3) & 417 & $-$0.08 & [$-$0.12, $-$0.05] \\
zoom out & 4 to 15 & $\ge$3 & 314 & $-$0.21 & [$-$0.26, $-$0.17] \\
zoom out & 16 to 31 & [1.5, 3) & 49 & $-$0.05 & [$-$0.11, +0.02] \\
 
\bottomrule
\end{tabular}\end{adjustbox}
\caption{Every cell of \cref{fig:map} with its 95\% bootstrap interval (images resampled; mosaics by source drawing).}
\label{tab:mapci}
\end{table}

\begin{figure*}[t]
\centering
\includegraphics[width=\textwidth]{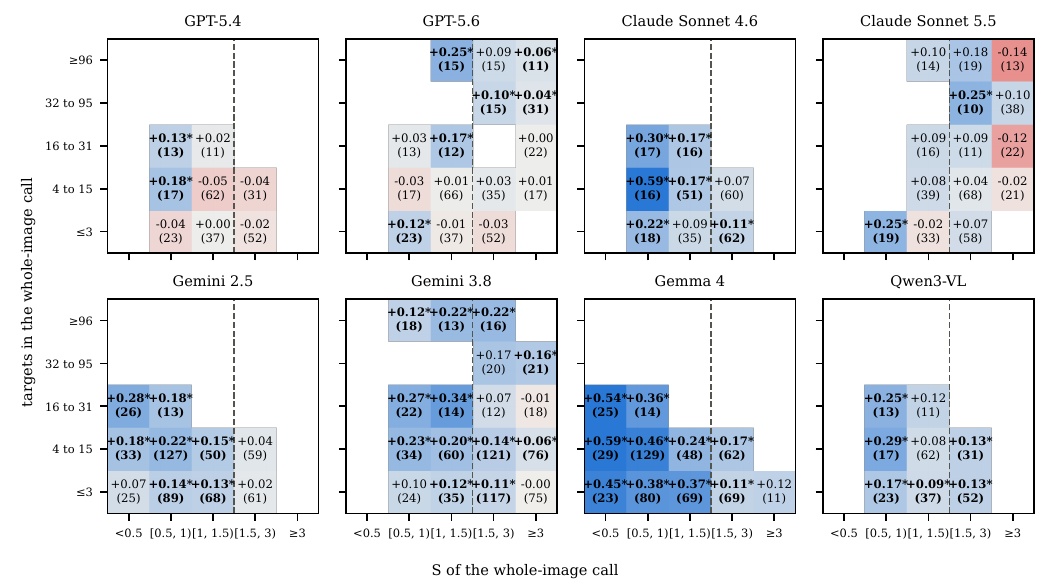}
\caption{\Cref{fig:map}(a), splits that never zoom out, for each of the eight models with tiles in our data. Colour scale as in \cref{fig:map}; cells with fewer than 10 pairs are blank; ${}^\ast$ and bold: 95\% interval excludes 0.}
\label{fig:map-models}
\end{figure*}

\begin{figure*}[t]
\centering
\includegraphics[width=\textwidth]{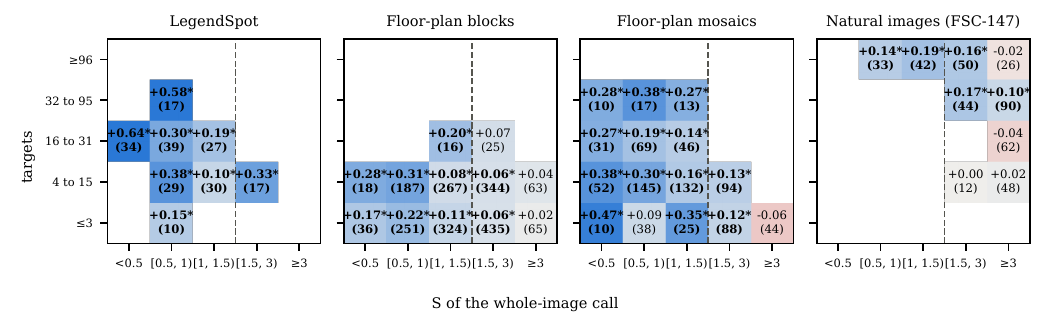}
\caption{\Cref{fig:map}(a) for each data set; cells and markers as in \cref{fig:map-models}.}
\label{fig:map-data}
\end{figure*}

\begin{figure*}[t]
\centering
\includegraphics[width=0.9\textwidth]{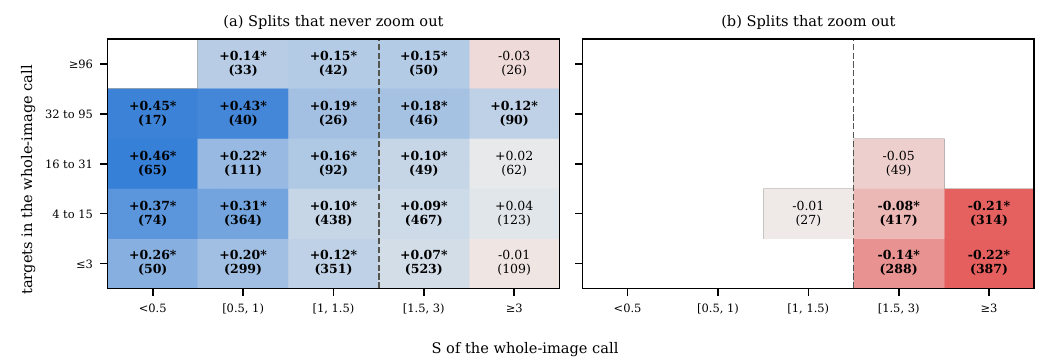}
\caption{\Cref{fig:map} reading only points written with \texttt{cx} and \texttt{cy} keys, which changes some Gemini 3.8 answers.}
\label{fig:map-cxcy}
\end{figure*}

\section{The rule and other methods: full results}
\label{sec:rule-full}

\subsection{The rule}

\Cref{tab:rule-full} gives \cref{tab:main} with intervals. Recall and precision use the one-to-one matching of \cref{sec:stats}; points in the overlaps are not merged, so a target reported twice adds an unmatched point. Precision is averaged over the images where both calls returned at least one point; the count of such images is listed, and Qwen3-VL returned no point on many images in its whole-image call. The cost ratio to $\Cmin$ is largest where the image is only one or two views wide, as \cref{cor:safe} predicts; \CtOneViewGpt{} of GPT-6.1's 60 floor-plan mosaics need a single view.

\begin{table*}[t]
\centering
\footnotesize
\setlength{\tabcolsep}{3pt}
\resizebox{\textwidth}{!}{\begin{tabular}{llcccccccc}
\toprule
Data & Model & Images & Recall, rule $-$ whole & Precision, rule $-$ whole & F1, rule $-$ whole & Images with points & Recall, rule $-$ 44 px & Tokens, rule / 44 px & Cost $/\Cmin$ \\
\midrule
LegendSpot & GPT-6.1 & 34 & +0.03 [$-$0.004, +0.09] & +0.04 [$-$0.004, +0.11] & +0.04 [$-$0.002, +0.10] & 34 & +0.003 [$-$0.003, +0.01] & 1.06 & 3.54 \\
LegendSpot & Claude Opus 5.5 & 34 & +0.10 [$-$0.004, +0.20] & \textbf{+0.10 [+0.02, +0.19]} & +0.09 [$-$0.01, +0.19] & 34 & $-$0.05 [$-$0.12, +0.01] & 0.61 & 2.57 \\
LegendSpot & Gemini 3.8 & 34 & \textbf{+0.28 [+0.15, +0.41]} & \textbf{+0.16 [+0.04, +0.28]} & \textbf{+0.22 [+0.10, +0.33]} & 33 & $-$0.04 [$-$0.09, +0.01] & 0.20 & 1.60 \\
LegendSpot & Gemini 4 & 34 & \textbf{+0.13 [+0.05, +0.23]} & +0.03 [$-$0.03, +0.12] & \textbf{+0.09 [+0.02, +0.19]} & 34 & +0.002 [$-$0.01, +0.02] & 0.20 & 1.60 \\
LegendSpot & Qwen3-VL & 34 & \textbf{+0.19 [+0.06, +0.33]} & \textbf{+0.18 [+0.03, +0.35]} & \textbf{+0.18 [+0.07, +0.31]} & 17 & $-$0.07 [$-$0.18, +0.05] & 0.69 & 1.99 \\
\midrule
Floor plans & GPT-6.1 & 60 & $-$0.02 [$-$0.07, +0.03] & +0.04 [$-$0.01, +0.09] & +0.01 [$-$0.04, +0.05] & 60 & n/a & n/a & 3.34 \\
Floor plans & Claude Opus 5.5 & 60 & +0.05 [$-$0.004, +0.11] & $-$0.02 [$-$0.08, +0.04] & +0.02 [$-$0.03, +0.07] & 56 & n/a & n/a & 2.01 \\
Floor plans & Gemini 3.8 & 60 & \textbf{+0.18 [+0.11, +0.25]} & \textbf{+0.21 [+0.10, +0.31]} & \textbf{+0.18 [+0.10, +0.26]} & 55 & \textbf{$-$0.13 [$-$0.19, $-$0.07]} & 0.11 & 1.57 \\
Floor plans & Gemini 4 & 60 & +0.07 [$-$0.004, +0.14] & \textbf{+0.12 [+0.04, +0.22]} & \textbf{+0.07 [+0.005, +0.14]} & 58 & $-$0.01 [$-$0.04, +0.03] & 0.11 & 1.57 \\
Floor plans & Qwen3-VL & 60 & \textbf{+0.07 [+0.02, +0.12]} & \textbf{+0.18 [+0.02, +0.35]} & \textbf{+0.09 [+0.04, +0.15]} & 13 & \textbf{$-$0.11 [$-$0.16, $-$0.06]} & 0.55 & 1.95 \\
 
\bottomrule
\end{tabular}}
\caption{The rule ($S^\ast = 1.5$, steps 1 and 2 of \cref{alg:rule}) against the whole-image call and the 44 px recipe, with 95\% paired bootstrap intervals over images; bold: interval excludes 0. n/a: recipe not run. Cost: median over images of the rule's tokens over $\Cmin$.}
\label{tab:rule-full}
\end{table*}

\subsection{The cost of the optional step 3}
\label{sec:step3}

Step 3 of \cref{alg:rule} splits a view further at the same scale into $q \times q$ sub-views, again overlapping by one target. Its cost follows from the geometry alone (\cref{prop:geometry}c). Let $r_q$ be the image tokens of the sub-views over those of the view. For a resolution-scaled model, splitting a full square view of $\sqrt N \times \sqrt N$ patches at $S = 1.5$ adds only the overlaps and the rounding to whole patches, so $r_q$ is close to 1; a fixed-grid model pays $g^2$ per view, so $r_q = q^2$, while each smaller view also gains a factor $q$ in $S$. Step 3 has two forms (\cref{sec:step3-proof}). In the \emph{appending} form, which a subdivision triggered by a view's content or answers must use, the view is sent first and the sub-views are sent as an additional batch whose points are added to the view's, at $1 + r_q$ times its tokens in all. In the \emph{replacing} form, for a subdivision fixed in advance, the sub-views are sent instead of the view, at $r_q$ times its tokens. \Cref{tab:step3} gives both. Each extra call also repeats the prompt and the reference image, a few hundred tokens per call for our prompts.

\begin{table}[t]
\centering
\small
\setlength{\tabcolsep}{3pt}
\resizebox{\columnwidth}{!}{\begin{tabular}{lccccccc}
\toprule
 & & \multicolumn{2}{c}{$q = 2$} & \multicolumn{2}{c}{$q = 3$} & \multicolumn{2}{c}{$q = 4$} \\
Model & $N$ per view & $r_q$ & $1 + r_q$ & $r_q$ & $1 + r_q$ & $r_q$ & $1 + r_q$ \\
\midrule
GPT-6.1 & 8{,}192 & 1.04 & 2.04 & 1.07 & 2.07 & 1.14 & 2.14 \\
Claude Opus 5.5 & 4{,}784 & 1.09 & 2.09 & 1.09 & 2.09 & 1.21 & 2.21 \\
Qwen3-VL & 2{,}500 & 1.08 & 2.08 & 1.17 & 2.17 & 1.25 & 2.25 \\
Gemini 3.8, Gemini 4 (fixed grid) & about 1{,}090 & 4.00 & 5.00 & 9.00 & 10.00 & 16.00 & 17.00 \\
 
\bottomrule
\end{tabular}}
\caption{Image tokens of step 3, splitting one full view into $q \times q$ sub-views at the same $S = 1.5$ with one target of overlap. $r_q$: the batch of sub-views over the parent view (replacing form); $1 + r_q$: the view and its sub-views together (appending form). Prompts and reference images excluded.}
\label{tab:step3}
\end{table}

\subsection{Other data and methods}
\label{sec:rule-more}

\paragraph{Library-10.} On the ten drawings whose references come from a symbol library we applied the configurations computed from \cref{prop:geometry} without change (\cref{tab:redp10}; intervals clustered by drawing). The drawings are small (about 1{,}500 by 1{,}700 px), so geometry predicts that a whole-image call already resolves the targets for GPT-5.6 (median $S = \XrdSgptFiveSix$) and Claude Sonnet 5.5 (\XrdSsonnetFiveFive) but less so for Gemini 3.8 (\XrdSgemThreeEight, against 2.5 in its tiles). Tiles change F1 by \XrdDDgptFiveSixAOnegptFiveSix{} \XrdDClDgptFiveSixAOnegptFiveSix{} for GPT-5.6, \XrdDDsonnetFiveFiveAOnesonnetFiveFive{} \XrdDClDsonnetFiveFiveAOnesonnetFiveFive{} for Claude Sonnet 5.5 and \XrdDDgemThreeEightAOnegemThreeEight{} \XrdDClDgemThreeEightAOnegemThreeEight{} for Gemini 3.8: the largest change is Gemini 3.8's, whose whole-image $S$ is lowest. Template matching with the threshold from LegendSpot collapses to \XrdFTM{} (\XrdFTMoracle{} even with the best threshold chosen on Library-10 itself), because library symbols differ in scale and stroke from the drawn ones, and every VLM configuration scores above it.

\begin{table}[t]
\centering
\small
\setlength{\tabcolsep}{3pt}
\begin{adjustbox}{max width=\linewidth}\begin{tabular}{lccc}
\toprule
 & GPT-5.6 & Sonnet 5.5 & Gemini 3.8 \\
\midrule
Whole image & \XrdFAOnegptFiveSix & \XrdFAOnesonnetFiveFive & \XrdFAOnegemThreeEight \\
Tiles & \XrdFDgptFiveSix & \XrdFDsonnetFiveFive & \textbf{\XrdFDgemThreeEight} \\
Propose + verify & \XrdFHgptFiveSix & \XrdFHsonnetFiveFive & \XrdFHgemThreeEight \\
Zoom agent & \textbf{\XrdFAGgptFiveSix} & \XrdFAGsonnetFiveFive & n/a \\
\midrule
TM, LegendSpot threshold & \multicolumn{3}{c}{\XrdFTM} \\
TM, oracle threshold & \multicolumn{3}{c}{\XrdFTMoracle} \\
CountGD / GeCo2 & \multicolumn{3}{c}{\XlbRdcountgd{} / \XlbRdgecoTwo} \\
\bottomrule
\end{tabular}\end{adjustbox}
\caption{Library-10: 23 drawing and class items with 225 instances, references from a shared symbol library. F1; nothing was tuned on these drawings except the oracle threshold, chosen on Library-10 as an upper bound for template matching (TM). Point estimates.}
\label{tab:redp10}
\end{table}

\paragraph{Floor-plan blocks.} With the LegendSpot configurations and prompts unchanged (\cref{tab:fpc300}), tiles raised F1 for the two models whose whole-image $S$ is about one (Gemini 2.5 \XfDmAgemini{} \XfDmACIgemini, Gemma 4 \XfDmAgemmaFour{} \XfDmACIgemmaFour) and for Gemini 3.8 (\XfDmAgemThreeEight{} \XfDmACIgemThreeEight), whose tiles raise $S$, and lowered it for GPT-5.4 (\XfDmAgpt{} \XfDmACIgpt), Claude Sonnet 4.6 (\XfDmAclaude{} \XfDmACIclaude), GPT-5.6 (\XfDmAgptFiveSix{} \XfDmACIgptFiveSix) and Claude Sonnet 5.5 (\XfDmAsonnetFiveFive{} \XfDmACIsonnetFiveFive), for which the 44 px recipe zooms out on many of these $1{,}000^2$ blocks (\cref{sec:zoomout}). The zoom agents, which choose their own enlargement, score highest for both models whose agents ran on these blocks.

\begin{table}[t]
\centering
\small
\setlength{\tabcolsep}{4pt}
\begin{adjustbox}{max width=\linewidth}\begin{tabular}{lccc}
\toprule
Model & Whole & Tiles & Agent \\
\midrule
GPT-5.4 & \XfBenchAgpt & \XfBenchDgpt & \\
Claude Sonnet 4.6 & \XfBenchAclaude & \XfBenchDclaude & \\
Gemini 2.5 Flash & \XfBenchAgemini & \XfBenchDgemini & \\
Gemma 4 31B & \XfBenchAgemmaFour & \XfBenchDgemmaFour & \\
Qwen3-VL & \XfBenchAqwenThreevl & \XfBenchDqwenThreevl & \\
GPT-5.6 & \XfBenchAgptFiveSix & \XfBenchDgptFiveSix & \XfBenchAGgptFiveSix \\
Claude Sonnet 5.5 & \XfBenchAsonnetFiveFive & \XfBenchDsonnetFiveFive & \XfBenchAGsonnetFiveFive \\
Gemini 3.8 Flash & \XfBenchAgemThreeEight & \XfBenchDgemThreeEight & \\
\midrule
Template matching & \multicolumn{3}{c}{\XfBenchAtm} \\
CountGD / GeCo2 & \multicolumn{3}{c}{\XlbTwofpccountgdimgOne{} / \XlbTwofpcgecoTwoimgOne} \\
\bottomrule
\end{tabular}\end{adjustbox}
\caption{Floor-plan blocks (339 test items, 21 classes), F1 with the block-edge band ignored. Thresholds of template matching and the counters are set on the 30 remaining items. Point estimates.}
\label{tab:fpc300}
\end{table}

\paragraph{Natural images.} \Cref{tab:fsc} reports FSC-147 under both prompts. Upsampling (more tokens, no new information) helps the resolution-scaled models on small objects and adds little on large ones, where F1 is already about 0.9; same-scale $2 \times 2$ tiles, which keep $S$ and divide the content per call by about four for these models, add a further \XfscODgptFiveSixsmallTU{} \XfscODCIgptFiveSixsmallTU{} for GPT-5.6 and \XfscODsonnetFiveFivesmallTU{} \XfscODCIsonnetFiveFivesmallTU{} for Claude Sonnet 5.5 on small, crowded images. Some answers on these images are truncated at the output cap; without them the tile gains are \CfscNoTruncGptFiveSix{} (\CfscNoTruncNGptFiveSix{} images) and \CfscNoTruncSonnetFiveFive{} (\CfscNoTruncNSonnetFiveFive). Gemini 3.8, whose fixed grid gives an upsampled image no more tokens, gains from tiles on small and medium objects; on its grid tiles both raise $S$ and lower the content per call. Claude Sonnet 5.5 omits the closing bracket of the list in \XfscParsesonnetFiveFive\% of its answers, which our reader recovers.

\begin{table*}[t]
\centering
\small
\begin{adjustbox}{max width=\linewidth}\begin{tabular}{llcccccc}
\toprule
Model & Stratum & $S$ (W) & F1 W & F1 U & F1 T & U $-$ W & T $-$ U \\
\midrule
\multicolumn{8}{l}{\emph{neutral prompt}} \\
GPT-5.6 & small & 0.40 & 0.51 & 0.59 & 0.73 & +0.09 [+0.002, +0.17] & +0.14 [+0.06, +0.22] \\
 & medium & 1.07 & 0.83 & 0.89 & 0.91 & +0.06 [+0.03, +0.10] & +0.02 [+0.01, +0.04] \\
 & large & 2.41 & 0.91 & 0.92 & 0.93 & +0.02 [+0.001, +0.04] & +0.01 [$-$0.01, +0.03] \\
Claude Sonnet 5.5 & small & 0.45 & 0.46 & 0.53 & 0.73 & +0.08 [+0.03, +0.12] & +0.20 [+0.12, +0.27] \\
 & medium & 1.22 & 0.87 & 0.86 & 0.90 & $-$0.02 [$-$0.05, +0.02] & +0.04 [+0.01, +0.07] \\
 & large & 2.75 & 0.89 & 0.91 & 0.92 & +0.02 [$-$0.003, +0.05] & +0.02 [$-$0.01, +0.05] \\
Gemini 3.8 & small & 0.98 & 0.33 & 0.41 & 0.57 & +0.08 [+0.002, +0.17] & +0.16 [+0.09, +0.24] \\
 & medium & 2.46 & 0.64 & 0.62 & 0.79 & $-$0.02 [$-$0.11, +0.07] & +0.17 [+0.08, +0.27] \\
 & large & 5.71 & 0.83 & 0.86 & 0.91 & +0.03 [$-$0.04, +0.09] & +0.06 [$-$0.001, +0.12] \\
\multicolumn{8}{l}{\emph{drawing prompt}} \\
GPT-5.6 & small & 0.40 & 0.25 & 0.19 & 0.30 & $-$0.06 [$-$0.12, $-$0.01] & +0.11 [+0.05, +0.18] \\
 & medium & 1.07 & 0.30 & 0.35 & 0.43 & +0.05 [$-$0.05, +0.16] & +0.08 [$-$0.01, +0.18] \\
 & large & 2.41 & 0.40 & 0.44 & 0.55 & +0.04 [$-$0.03, +0.11] & +0.11 [+0.04, +0.18] \\
Claude Sonnet 5.5 & small & 0.45 & 0.36 & 0.45 & 0.59 & +0.09 [+0.03, +0.14] & +0.14 [+0.09, +0.20] \\
 & medium & 1.22 & 0.56 & 0.51 & 0.63 & $-$0.04 [$-$0.09, +0.002] & +0.12 [+0.06, +0.17] \\
 & large & 2.75 & 0.60 & 0.53 & 0.63 & $-$0.07 [$-$0.13, $-$0.01] & +0.10 [+0.03, +0.16] \\
Gemini 3.8 & small & 0.98 & 0.22 & 0.21 & 0.31 & $-$0.01 [$-$0.05, +0.01] & +0.10 [+0.04, +0.17] \\
 & medium & 2.46 & 0.21 & 0.20 & 0.31 & $-$0.01 [$-$0.08, +0.04] & +0.12 [+0.05, +0.19] \\
 & large & 5.71 & 0.20 & 0.18 & 0.36 & $-$0.02 [$-$0.10, +0.04] & +0.18 [+0.11, +0.26] \\
 
\bottomrule
\end{tabular}\end{adjustbox}
\caption{FSC-147, 150 images in three strata of exemplar size. W: whole image; U: upsampled to a long side of 1{,}536 px; T: $2 \times 2$ tiles at the scale of U. $S$: median tokens per object side in W. Intervals: 95\% bootstrap over images. Neutral prompt: the prompt of \cref{sec:prompts}; drawing prompt: the first, drawing-specific wording.}
\label{tab:fsc}
\end{table*}

\paragraph{Template matching and trained counters.} Template matching reaches \FTMdevAll{} F1 on the 34 LegendSpot drawings. Propose and verify, in which the VLM only judges candidates found by template matching at many tokens per candidate, places every model between \FHgptFiveFourAll{} and \FHqwenThreevlAll, none significantly different from template matching (differences \CHMinusTMdevgptFiveFour{} to \CHMinusTMdevqwenThreevl); by \cref{prop:hybrid} the remaining loss lies in the verifier, since the candidates contain a true instance for a ceiling of \HybBound. CountGD and GeCo2, class-agnostic exemplar counters, reach \XctrFsccountgd{} and \XctrFscgecoTwo{} F1 on our FSC-147 images; on LegendSpot, with the legend crop as exemplar on the same tiles as the VLMs, they reach \XlbFcountgd{} and \XlbFgecoTwo{} (OWLv2 \XlbFowlvTwo), and at most \XlbOraclegecoTwo{} even with the best threshold chosen per drawing; on Library-10 they reach \XlbRdcountgd{} and \XlbRdgecoTwo. Instances in drawings are thin line figures among look-alike symbols, unlike the textured objects these counters were trained on.

\paragraph{Calls and tokens.} \Cref{tab:cost} lists the calls and input tokens per LegendSpot drawing for the pipelines of the earlier experiments.

\begin{table}[t]
\centering
\footnotesize
\resizebox{\columnwidth}{!}{\begin{tabular}{llcc}
\toprule
Model & Condition & calls & input tokens \\
\midrule
GPT-5.4 & whole drawing & 1.0 & 2.1k \\
GPT-5.4 & tiles & 13.4 & 16.8k \\
Claude Sonnet 4.6 & tiles & 13.4 & 18.5k \\
Gemma 4 & tiles & 55.6 & 52.8k \\
Gemini 2.5 Flash & tiles & 145.8 & 134.7k \\
GPT-5.4 & propose + verify & 57.9 & 54.1k \\
Claude Sonnet 5.5 & whole drawing & 1.0 & 2.4k \\
GPT-5.6 (original) & whole, native & 1.0 & 11.1k \\
Claude Sonnet 5.5 & tiles & 13.4 & 20.6k \\
GPT-5.6 & tiles & 13.4 & 16.8k \\
Gemini 3.8 Flash & tiles & 37.5 & 96.9k \\
Claude Sonnet 5.5 & propose + verify & 57.9 & 66.0k \\
 
\bottomrule
\end{tabular}}
\caption{Calls and input tokens per LegendSpot test drawing (one run). Template matching runs locally in seconds.}
\label{tab:cost}
\end{table}

\section{Zoom agents}
\label{sec:agents}

\begin{figure*}[t]
\centering
\includegraphics[width=\textwidth]{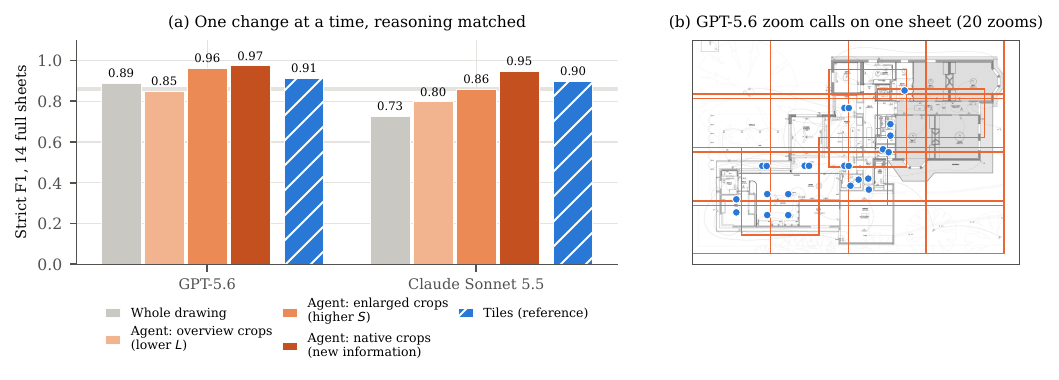}
\caption{(a) F1 on the 14 full sheets with the same reasoning setting throughout (GPT-5.6: medium effort; Claude Sonnet 5.5: adaptive thinking). From left to right each bar changes one step: one whole-image call; an agent whose tool returns the requested region cut from the overview at the overview's scale (lower $L$ only); the same crops enlarged to the size of a native crop (also higher $S$, no new information); native crops (also new information); and, for reference, tiles. Agent bars average two runs, the overview-crop bars one. (b) The 20 regions GPT-5.6 chose on one sheet (orange), over the overview it received; blue dots are the annotated instances.}
\label{fig:agent}
\end{figure*}

\paragraph{Setup.} The agent receives the legend crop, an overview of the image (long side 1{,}600 px) and a \texttt{zoom} tool that returns any requested rectangle cut from the native image (long side at most 1{,}024 px, at most twice native), for up to 20 calls, with reasoning where the model has it (\cref{tab:models}). Each zoom raises $S$ by the enlargement factor and limits $L$ to the rectangle. Every image the agent keeps in its context is counted in its tokens, as in \cref{prop:frontier}.

\paragraph{Who splits, on the same model.} \Cref{tab:whosplits} puts side by side, for each model with an agent and with the same reasoning setting throughout, the whole-image call, the caller's split (tiles of the 44 px recipe, which zoom out on only one of these sheets) and the agent that splits the image itself. A model that zooms well can split more accurately than the caller's fixed tiles (GPT-5.6, Claude Sonnet 5.5); a model whose coordinates are unreliable across zoom levels fails when it splits itself (Gemini 3.8, Qwen3-VL). Only a caller's split that meets the condition, such as the rule, carries the protection of \cref{thm:exact} and a cost known before the call.

\begin{table}[t]
\centering
\small
\setlength{\tabcolsep}{3pt}
\resizebox{\columnwidth}{!}{\begin{tabular}{lccccc}
\toprule
 & \multicolumn{3}{c}{F1, 14 full sheets} & \multicolumn{2}{c}{Input tokens per drawing} \\
\cmidrule(lr){2-4}\cmidrule(lr){5-6}
Model & Whole & Caller's tiles & Agent & Tiles & Agent \\
\midrule
GPT-5.6 & \XnFAOneRgptFiveSixSheets & \XnFDRgptFiveSixSheets & \XnFAGgptFiveSixSheets & \XtileTokIngptFiveSix & \XagTokIngptFiveSix \\
Claude Sonnet 5.5 & \XnFAOneRsonnetFiveFiveSheets & \XnFDRsonnetFiveFiveSheets & \XnFAGsonnetFiveFiveSheets & \XtileTokInsonnetFiveFive & \XagTokInsonnetFiveFive \\
Gemini 3.8 & \XnFAOneRgemThreeEightSheets & \XnFDRgemThreeEightSheets & \XnFAGgemThreeEightSheets & \XtileTokIngemThreeEight & \XagTokIngemThreeEight \\
\bottomrule
\end{tabular}}
\caption{Who splits, on the same model and with the same reasoning setting (GPT-5.6: medium effort; Claude Sonnet 5.5: adaptive thinking; Gemini 3.8: dynamic thinking). Caller's tiles: the 44 px recipe and its fixed-grid counterpart, which zoom out on one of the 14 sheets for GPT-5.6 and Claude Sonnet 5.5 and on none for Gemini 3.8. Tokens: mean input tokens per LegendSpot test drawing. Qwen3-VL, which has no reasoning mode, reaches \XqD{} F1 with tiles and \XqAgF{} as an agent over all 34 drawings.}
\label{tab:whosplits}
\end{table}

\paragraph{Which step does the work.} \Cref{fig:agent}a changes one step at a time; the three steps are ordered by \cref{prop:zoom} at the level of single calls. Lowering $L$ alone gave no significant gain for either model (GPT-5.6 \XnHThreegptFiveSixOVCAOneR{} \XnHThreegptFiveSixOVCAOneRCI; Claude Sonnet 5.5 \XnHThreesonnetFiveFiveOVCAOneR{} \XnHThreesonnetFiveFiveOVCAOneRCI). Enlarging the same crops, which raises $S$ without new information, gave GPT-5.6 \XnHThreegptFiveSixOVZOVC{} \XnHThreegptFiveSixOVZOVCCI, and native pixels added \XnHThreegptFiveSixAGOVZ{} \XnHThreegptFiveSixAGOVZCI: for GPT-5.6 the active ingredient is $S$. For Claude Sonnet 5.5 no single step is decisive (enlargement \XnHThreesonnetFiveFiveOVZOVC{} \XnHThreesonnetFiveFiveOVZOVCCI), and only native pixels give a gain whose interval excludes zero (\XnHThreesonnetFiveFiveAGOVZ{} \XnHThreesonnetFiveFiveAGOVZCI): for it the overview lacks usable pixels, not only tokens~\cite{xu2020usable}. Over all 34 drawings the agents against their tiles differ by \XnAGDRallgptFiveSix{} \XnAGDRallgptFiveSixCI{} (GPT-5.6) and \XnAGDRallsonnetFiveFive{} \XnAGDRallsonnetFiveFiveCI{} (Claude Sonnet 5.5); on Library-10 GPT-5.6's agent again scores highest (\cref{tab:redp10}).

\paragraph{How the agents search.} GPT-5.6 used \XagZoomsgptFiveSix{} zooms per drawing on average; on several sheets it first laid a regular grid of equal crops over the whole image and then inspected candidate regions (\cref{fig:agent}b), and on full sheets it viewed \XagCovSheetsgptFiveSix{} of the area at $S \ge 1$. Claude Sonnet 5.5 used \XagZoomssonnetFiveFive{} zooms on average, none on 9 of the 20 small drawings, and concentrated them on hard sheets. The overview itself, at a median $S$ of \XovSmedSheets{} on full sheets, already covers much of each sheet. Per drawing, GPT-5.6's agent used \XagTokIngptFiveSix{} input and \XagTokOutgptFiveSix{} output tokens, against \XtileTokIngptFiveSix{} and \XtileTokOutgptFiveSix{} for its tiles; Claude's agent used \XagTokInsonnetFiveFive{} and \XagTokOutsonnetFiveFive, about half of its tiles (\XtileTokInsonnetFiveFive{} and \XtileTokOutsonnetFiveFive). On floor-plan mosaics, where the first block's content is fixed while the mosaic grows, GPT-5.6's agent stays flat in first-block F1, as tiles do and whole-image calls do not (\cref{sec:firstblock}).

\paragraph{Two agents fail at the coordinates.} The Gemini 3.8 and Qwen3-VL agents count about as well as their tiles but place points wrongly. Gemini 3.8 used the most zooms (\XagZoomsgemThreeEight, \XagTokIngemThreeEight{} input tokens per drawing) and scored \XagFgemThreeEightAll{} F1 (against its tiles \XagMinusDgemThreeEight{} \XagMinusDCIgemThreeEight), yet gave the exact count on \XagExactCountgemThreeEight\% of drawings. Qwen3-VL scored \XqAgF, against \XqD{} for its tiles (\XqAgMinusD{} \XqAgMinusDCI), with a median relative count error of \XqCountErrAg, no worse than with tiles (\XqCountErrD). A check on the stored trajectories, with no new call, locates the failure. The count is within 10\% of the truth on \CagCntGem\% of the drawings for the Gemini 3.8 agent (\CtileCntGem\% for its tiles) and on \CagCntQwen\% for the Qwen3-VL agent (\CtileCntQwen\%). The points are not: \CagLocGem\% of the Gemini 3.8 agent's points lie within $\tau$ of an instance and \CagLocGemTwo\% within $2\tau$ (median distance to the nearest instance \CagDistGem$\,\tau$), against \CtileLocGem\% for its tiles; for Qwen3-VL, \CagLocQwen\% and \CagLocQwenTwo\% (median \CagDistQwen$\,\tau$), against \CtileLocQwen\%. Reading the agents' answers in the other coordinate convention makes them worse, so the convention is not the cause. Agents and tiles both see crops of native pixels; what differs is that the agent itself converts positions seen in a zoomed crop back to the overview, a conversion that no tiled or whole-image call requires (\cref{sec:coords}).

\section{Code, data and reproducibility}
\label{sec:repro}

The repository at \url{https://github.com/shijunzhe/vlm-small-objects} contains:
the LegendSpot images, legend crops and centre annotations;
the scripts that build the floor-plan blocks and mosaics from the public FloorPlanCAD test split and select the FSC-147 images;
the generators of all controlled images, with their seeds;
the code that sends every call, with the prompts of \cref{sec:prompts} and the per-model geometry of \cref{tab:calib};
the readers of \cref{sec:coords} and the scoring, statistics and figure scripts;
the model answers behind every table and figure;
and the planning tool for safe splitting, which takes an image size, a target box and an interface and returns $\Sw$, the views of \cref{alg:rule}, their tokens and $\Cmin$.
One command rescores every experiment from the stored answers and regenerates all numbers of this paper, which enter the text as generated macros, and all tables and figures, without any model call. The Library-10 drawings are used for evaluation only and are not included.

\paragraph{Societal impact.} The rule and the theory make VLM calls on large images more accurate and their cost predictable, which helps wherever many small objects must be found, from documents and screens to drawings and maps. Two kinds of harm deserve mention. First, counts and positions from VLM calls can enter consequential estimates, such as quantities taken off a drawing for a bid, and a missed or invented object then carries a cost; the guarantee is relative, stating that a protected split detects in expectation at least what the whole-image call detects, and it does not certify that every object is found. We therefore report precision with recall, give the paired-call check of \cref{sec:safe} for new input paths, and recommend human review of counts that feed decisions. Second, better small-object search in high-resolution imagery can serve surveillance as well as inspection; the techniques here add no new sensing capability, since they only decide how an image already held is sent to a model, and their use should follow the rules that govern the imagery itself.

\section{Reference data for the rule of thumb}
\label{sec:thumb}

\Cref{tab:thumb} gives approximate values for one image content. This section lists the data behind them, keeping three things apart: the vendor's limits, the per-view budget the rule uses, and what we measured.

\paragraph{Per-view budgets.} \Cref{tab:thumb-budget} gives the budgets. At $S = 1.5$ a square view of $\sqrt N \times \sqrt N$ patches holds about $\sqrt N/1.5$ targets per side. For a resolution-scaled model the target must also span about 1.5 tokens in the sent image, about 48 px for GPT and Qwen (32 px tokens) and 42 px for Claude (28 px tokens); a caller can enlarge the image to reach this, which raises $S$ but cannot restore detail that the native pixels lack. A near-square view then has a side of at most about $m\sqrt N/S$ native pixels with $S = \max(1.5, \Sw)$, plus the rounding to whole patches and the overlap. For fixed-grid models only the geometry matters: a view of longer side $\ell$ gives $S = mg/\ell$.

\begin{table*}[t]
\centering
\footnotesize
\begin{tabular}{p{2.7cm}p{4.6cm}p{3.4cm}p{1.8cm}p{1.4cm}}
\toprule
Model & Vendor documentation (October 2026) & Our billing & Budget used & Per side \\
\midrule
GPT-6.1, GPT-5.6 & 32 px patches; the default setting keeps the original size; requests above 30{,}000 patches are refused & $2896^2$: 8{,}281 patches (GPT-5.6), $3072^2$: 9{,}216 (GPT-6.1), not resized & 8{,}192 patches & about 60 \\
GPT-5.4 & 2{,}500 patches, long side $\le$2{,}048 px; original: 10{,}000 & 2{,}500 patches (3{,}019 input tokens) at $2048^2$ & 2{,}500 & about 33 \\
Claude Opus 5.5, Sonnet 5.5 & 28 px; long side $\le$2{,}576 px, $\le$4{,}784 tokens & $1932^2$: 4{,}761 & 4{,}784 & about 46 \\
Claude Sonnet 4.6 & long side $\le$1{,}568 px, $\le$1{,}568 tokens & 1{,}542 at $1536^2$ and above & 1{,}568 & about 26 \\
Gemini 3.8, Gemini 4 & Gemini 3.x: 1{,}120 per image (default), 2{,}240 (ultra-high); Gemini 4 not listed separately & about 1{,}090 at every size & about 1{,}090 & about 22 \\
Gemini 2.5, Gemma 4 & 258 and 280 per image & 258 and 280 & 258, 280 & about 11 \\
Qwen3-VL & 32 px per token; cap set by the provider & about 2{,}500 at 1{,}600 px & 2{,}500 & about 33 \\
\bottomrule
\end{tabular}
\caption{Per-view budgets. Per side: targets along the side of a full square view at $S = 1.5$, $\sqrt N/1.5$.}
\label{tab:thumb-budget}
\end{table*}

\paragraph{Load at $S = 1.5$.} To see how many targets one view can carry, we rendered controlled images at $S = 1.5$ for every model, each with $n$ targets of one legend symbol and $n$ distractors of three look-alike symbols on a lattice of pitch 1.6 target sides, about half the lattice cells occupied, with grey linework. The layout of image $i$ at load $n$ is fixed in target units and shared by all models; light strokes are kept as grey (each pixel takes the darker of the canvas and the reduced symbol). Views are capped by the per-view budgets of \cref{tab:thumb-budget}: 2{,}896 px for GPT ($91 \times 91 = 8{,}281$ patches, which the vendor does not resize), 1{,}932 px for Claude, 1{,}600 px for Qwen3-VL and the fixed 1{,}024 px canvas for Gemini. Each load has 12 images, the calls use the reasoning settings of \cref{tab:models} with an output cap of 16{,}384 tokens, and answers are read in each model's convention with the lenient recovery of \cref{sec:coords} (\cref{tab:thumb-load}; 384 calls, all answers in the released package). We give both the pooled recall, the quantity the guarantees protect, and the number of images on which recall and precision both reach 0.9. The largest load tested is the number of targets the layout fits in the budget: 324 for GPT, 196 for Claude and 42 for Gemini.

\begin{table*}[t]
\centering
\small
\setlength{\tabcolsep}{4pt}
\begin{adjustbox}{max width=\linewidth}\begin{tabular}{lccccccl}
\toprule
Model & 8 & 16 & 32 & 64 & 128 & 256 & Largest load tested \\
\midrule
GPT-6.1 & 1.00; 11 & 1.00; 11 & 1.00; 11 & 1.00; 11 & 1.00; 9 & 1.00; 11 & 324: 1.00; 11 \\
Claude Opus 5.5 & 0.97; 7 & 1.00; 6 & 1.00; 10 & 1.00; 10 & 1.00; 11 &  & 196: 0.83; 9 (2 empty) \\
Gemini 4 & 1.00; 11 & 1.00; 11 & 1.00; 11 &  &  &  & 42: 0.99; 11 \\
GPT-5.6 & 0.96; 6 & 0.99; 11 & 1.00; 9 & 0.97; 10 & 0.91; 8 & 0.34; 0 & 324: 0.29; 0 (6 empty) \\
Gemini 3.8 & 0.96; 7 & 0.93; 8 & 0.77; 7 &  &  &  & 42: 0.67; 1 \\
Qwen3-VL & 0.89; 4 & 0.88; 4 & 0.77; 3 & 0.39; 0 &  &  &  \\
 
\bottomrule
\end{tabular}\end{adjustbox}
\caption{Controlled images at $S = 1.5$, by number of targets per view: pooled recall; images (of 12) with recall and precision both at least 0.9; empty: answers with no point.}
\label{tab:thumb-load}
\end{table*}

GPT-6.1 and Gemini 4 keep pooled recall at \lsv{gpt61}{324}{R} and at least \lsv{gem4}{42}{R} at every load tested, and Claude Opus 5.5 at least \lsv{opus55}{8}{R} up to 128 targets. At 196 targets Claude Opus 5.5 finds every target on 10 of the 12 images and returns an empty list on the other \lsv{opus55}{196}{empty}, so its pooled recall is \lsv{opus55}{196}{R}: at this load it fails all or nothing. Per image, Claude Opus 5.5 at small loads and GPT-6.1 at some loads miss the joint criterion by reporting extra points. GPT-5.6 misses a few targets with 8, holds from 16 to 128 targets (\lsv{gpt56}{128}{R} at 128), and falls at 256 and 324 targets (\lsv{gpt56}{256}{R} and \lsv{gpt56}{324}{R}), where \lsv{gpt56}{256}{empty} and \lsv{gpt56}{324}{empty} of its 12 answers are empty. Gemini 3.8 drops below 0.9 at 32 targets by missing targets (\lsv{gem38}{32}{R}); at 42 it lists 27 to 29 points on 11 of the 12 images. Qwen3-VL, run here through a third provider (\cref{tab:models}), recalls \lsv{qwen3vl}{8}{R} and \lsv{qwen3vl}{16}{R} of 8 and 16 targets, but its answers often have malformed lists and extra points, so few images meet the joint criterion; it falls at 32 (\lsv{qwen3vl}{32}{R}) and 64 targets (\lsv{qwen3vl}{64}{R}). With the vendor's default reasoning, in a load stage run with the protocol of \cref{sec:newest}, GPT-5.6 recalls 0.92 of about 256 targets at $S \approx 1.27$.

\paragraph{Time.} Only the runs of \cref{sec:newest} and that load stage recorded wall-clock time (12 concurrent requests, vendor default reasoning, queueing included). The median time of one call grows with the number of targets it must report, because it is set mainly by the length of the answer: for 8 targets from about 3\,s (Claude Opus 5.5) to 18\,s (Qwen3-VL), and for 64 to 128 targets from about 13\,s to 57\,s. Views of a split are independent, so with enough concurrency a split takes about as long as its slowest view.

\section{Further comparisons on LegendSpot}
\label{sec:more}

\begin{table}[t]
\centering
\small
\setlength{\tabcolsep}{3pt}
\resizebox{\columnwidth}{!}{\begin{tabular}{lccccc}
\toprule
 & Whole & Vendor & \multicolumn{2}{c}{Tiles} & Prop.\,+ \\
\cmidrule(lr){4-5}
Model & image & hi-res & 44 px & counterpart & verify \\
\midrule
GPT-5.4 & \FATwogptFiveFourAll & \FBgptFiveFourAll & \multicolumn{2}{c}{\FDgptFiveFourAll} & \FHgptFiveFourAll \\
Claude Sonnet 4.6 & \FATwosonnetFourSixAll & n/a & \multicolumn{2}{c}{\FDsonnetFourSixAll} & \FHsonnetFourSixAll \\
Gemini 2.5 Flash & \FATwogeminiTwoFiveAll & \FBgeminiTwoFiveAll & \FCgeminiTwoFiveAll & \FDgeminiTwoFiveAll & \FHgeminiTwoFiveAll \\
Gemma 4 31B & \FATwogemmaFourAll & n/a & \FCgemmaFourAll & \FDgemmaFourAll & \FHgemmaFourAll \\
Qwen3-VL & \FATwoqwenThreevlAll & n/a & \multicolumn{2}{c}{\FDqwenThreevlAll} & \FHqwenThreevlAll \\
\midrule
Template matching & \multicolumn{5}{c}{\FTMdevAll{} (threshold from the prompt-writing sheets), \FTMlooAll{} (leave-one-out)} \\
\bottomrule
\end{tabular}}
\caption{Older-generation models on the 34 LegendSpot test drawings, F1. Whole image read in each model's published convention. Tiles: the 44 px recipe for every model, and its fixed-grid counterpart, whose field of view is set by \cref{prop:geometry}(c) (for resolution-scaled models the two coincide). Point estimates.}
\label{tab:round1}
\end{table}

\paragraph{Tiles, the whole image and vendor high-resolution modes.} With tiles, every older-generation model except GPT-5.4 beats its whole-image call (\cref{tab:round1}): by \CDMinusATwosonnetFourSix{} \CDMinusATwosonnetFourSixCI{} F1 for Claude Sonnet 4.6, \CDMinusATwogeminiTwoFive{} \CDMinusATwogeminiTwoFiveCI{} for Gemini 2.5, \CDMinusATwogemmaFour{} \CDMinusATwogemmaFourCI{} for Gemma 4 and \CDMinusATwoqwenThreevl{} \CDMinusATwoqwenThreevlCI{} for Qwen3-VL; GPT-5.4 changes by \CDMinusATwogptFiveFour{} \CDMinusATwogptFiveFourCI. Gemma 4, fourth with the whole image, is first with tiles. Vendors' high-resolution modes do not replace splitting: GPT-5.4 with \texttt{detail=original} sees each symbol at about 1.3 tokens and scores \CBMinusDgptFiveFour{} \CBMinusDgptFiveFourCI{} F1 against its tiles, and Gemini 2.5's high setting \CBMinusDgeminiTwoFive{} \CBMinusDgeminiTwoFiveCI{} against its fixed-grid counterpart. On the 14 full sheets, tiles beat the whole image for Claude Sonnet 5.5 (\CDMinusATwosheetssonnetFiveFive{} \CDMinusATwosheetssonnetFiveFiveCI) and Gemini 3.8 (\CDMinusATwosheetsgemThreeEight{} \CDMinusATwosheetsgemThreeEightCI); GPT-5.6 changes by \CDMinusATwosheetsgptFiveSix{} \CDMinusATwosheetsgptFiveSixCI.

\paragraph{Reasoning.} Reasoning helps where recognition, not resolution, limits. Over the 34 drawings, GPT-5.6's whole-image F1 rises from \XplainFgptFiveSixrAll{} to \XrsnFgptFiveSixrAll{} with high-effort reasoning, and on full sheets its tiles gain \XnDRDgptFiveSix{} \XnDRDgptFiveSixCI{} at medium effort. Claude Sonnet 5.5 runs with adaptive thinking in every call, so its reasoning arm adds only an explicit high effort and a larger output cap (tiles: \XnDRDsonnetFiveFive{} \XnDRDsonnetFiveFiveCI). Gemini 3.8's thinking changes full-sheet F1 with tiles by \XnDRDgemThreeEight{} \XnDRDgemThreeEightCI, and with the whole image lowers it from \XplainFgemThreeEighttAll{} to \XrsnFgemThreeEighttAll, partly because its thinking switches between pixel and 0 to 1000 coordinates (\cref{sec:coords}). On full sheets, the whole image with reasoning still scores below tiles without reasoning for Claude Sonnet 5.5 (\XrsnMinusDsonnetFiveFivetSheets{} \XrsnMinusDCIsonnetFiveFivetSheets, whole minus tiles) and Gemini 3.8 (\XrsnMinusDgemThreeEighttSheets{} \XrsnMinusDCIgemThreeEighttSheets); for GPT-5.6 with high-effort reasoning the difference is \XrsnMinusDgptFiveSixrSheets{} \XrsnMinusDCIgptFiveSixrSheets.

\paragraph{Earlier names.} Versions 1 and 2 of this paper called LegendSpot REDP-X40, Library-10 REDP-10, the floor-plan blocks FPC-300 and the floor-plan mosaics FPC-sheets; the code keeps these internal names, and its README maps them. Version 3 reads every answer with one reader, which keeps points written with \texttt{x} and \texttt{y} keys and points written as lists (\cref{sec:coords}); a few numbers therefore differ slightly from version 2.

\end{document}